\documentclass[twoside,11pt]{article}

\usepackage{blindtext}

\usepackage{jmlr2e}

\usepackage{bm}
\usepackage{amsmath}
\usepackage{mathtools}
\usepackage{xcolor}
\usepackage{algorithm,algpseudocode}
\usepackage{enumitem}
\usepackage{booktabs}
\usepackage{subfig}
\usepackage[utf8]{inputenc}
\usepackage[T1]{fontenc}

\newcommand{\ie}{i.e.}

\usepackage{lastpage}

\ShortHeadings{On Model Compression for Neural Networks}{Li \textit{et al.}}
\firstpageno{1}

\begin{document}

\title{On Model Compression for Neural Networks: Framework, Algorithm, and Convergence Guarantee}

\author{\name Chenyang Li$^{1}$ \email cl237@njit.edu \\
        \name Jihoon Chung$^{2}$ \email jhchung7@pusan.ac.kr  \\
        \name Mengnan Du$^{1}$ \email mengnan.du@njit.edu \\
        \name Haimin Wang$^{1}$ \email haimin.wang@njit.edu  \\
        \name Xianlian Zhou$^{1}$ \email alexzhou@njit.edu \\
        \name Bo Shen$^{1,*}$ \email bo.shen@njit.edu  \\
   $^{1}$New Jersey Institute of Technology, USA \\
   $^{2}$Pusan National University, Korea\\
   $^{*}$Corresponding Author}

 \editor{}

\maketitle


\begin{abstract}
Model compression is a crucial part of deploying neural networks (NNs), especially when the memory and storage of computing devices are limited in many applications. This paper focuses on two model compression techniques: low-rank approximation and weight pruning in neural networks, which are very popular nowadays. However, training NN with low-rank approximation and weight pruning always suffers significant accuracy loss and convergence issues. In this paper, a holistic framework is proposed for model compression from a novel perspective of nonconvex optimization by designing an appropriate objective function. Then, we introduce NN-BCD, a block coordinate descent (BCD) algorithm to solve the nonconvex optimization. One advantage of our algorithm is that an efficient iteration scheme can be derived with closed-form, which is gradient-free. Therefore, our algorithm will not suffer from vanishing/exploding gradient problems. Furthermore, with the Kurdyka-\L{}ojasiewicz (K\L{}) property of our objective function, we show that our algorithm globally converges to a critical point at the rate of $\mathcal{O}(1/k)$, where $k$ denotes the number of iterations. Lastly, extensive experiments with tensor train decomposition and weight pruning demonstrate the efficiency and superior performance of the proposed framework. Our code implementation is available at \url{https://github.com/ChenyangLi-97/NN-BCD}.
\end{abstract}


\section{Introduction} \label{sec: introduction}
The advent of neural networks (NNs) has brought about a revolution in various applications, including anomaly detection~\citep{su2024large}, healthcare~\citep{esteva2019guide,zhu2021twitter}, solar physics~\citep{jiang2023generating,xu2024super}, and more.  Especially, Large Language Models (LLMs) consistently exhibit remarkable performance across various tasks \citep{zhao2023survey,chang2023survey}. Nevertheless, their exceptional capabilities come with significant challenges stemming from their extensive size and computational requirements. For instance, the GPT-175B model~\citep{brown2020language}, with an impressive 175 billion parameters, demands a minimum of 320GB (using multiples of 1024) of storage in half-precision (FP16) format. The sizes of many state-of-the-art NN models are too large for most embedded and Internet-of-Things (IoT) systems \cite{chen2019deep,hu2023artificial}, thereby causing high storage and computational demands and severely hindering the practical deployment of NNs.

To tackle this problem, researchers have proposed numerous model compression techniques for NNs~\citep{carreira2017model,carreira2021model,li2023model}, which can be summarized into the following categories. (1) Low-rank approximation~\citep{li2018constrained,idelbayev2020low}: this technique involves approximating the weight matrices/tensors of a deep learning model with low-rank matrices/tensors. (2) Weight pruning~\citep{han2015learning, luo2017thinet,carreira2018learning}: this technique explores the redundancy in the model parameters and tries to remove the redundant and non-critical ones.  (3) Quantization~\citep{carreira2017model2,xu2018deep}: this involves reducing the number of bits required to represent the weights and activations in a neural network. For example, weights and activations may be represented using 8-bit integers instead of 32-bit floating-point numbers.  (4) Knowledge distillation~\citep{gou2021knowledge}: this learns a distilled model and trains a more compact neural network to reproduce the output of a larger network. 


In this paper, we will focus on low-rank approximation and weight pruning among all model compression techniques since they can be categorized into one framework, where they share a common goal. Specifically, both techniques aim to reduce the complexity of models by identifying and retaining only the most essential components or parameters, where low-rank approximation and weight pruning encourage low-rank and sparse representations, respectively. Among low-rank approximation, tensorizing neural networks is an extremely attractive NN model compression technique based on tensor decomposition.  By utilizing advanced tensor decomposition techniques~\citep{kolda2009tensor,shen2022smooth} like tensor train decomposition (TTD)~\citep{oseledets2011tensor}, it is possible to achieve more than a 1,000× reduction in parameters for the input-to-hidden layers of neural networks~\citep{yang2017tensor,pan2019compressing}. Weight pruning has been proven to be very effective in reducing the resource requirements of neural networks. This assertion is supported by numerous studies and experiments conducted in the field~\citep{han2015learning,molchanov2016pruning,molchanov2019importance,hoefler2021sparsity}. 


\textbf{Research Gaps.} Although low-rank decomposition and weight pruning demonstrate powerful performance in model compression, neural network training for both cases is quite challenging~\citep{kim2015compression, zhang2018systematic, chijiwa2021pruning, yin2021towards,JMLR:v23:21-0545}.  In general, there are two strategies of neural network training for model compression: (1) Train from scratch, and (2) Decompose/prune a pre-trained model and then retrain. In the first case, the required NN (for example, tensor train-based or the sparse neural network) is directly trained from scratch. Since the structure of the NNs is already pre-set to low-rank or sparse format before the training, the corresponding model capacity is typically limited as compared to the uncompressed structure. Therefore, the training process is very sensitive to initialization and is more challenging to achieve high accuracy. In the second approach, though the pre-trained uncompressed model provides a good initialization, the uncompressed model needs to be decomposed into a low-rank format or pruned into sparse format. This causes inevitable and non-negligible approximation error, which leads to performance degradation. No matter which training strategy is adopted, the training of NN heavily relies on gradient-based methods~\citep{rumelhart1986learning}. Thus, they are typically more prone to the vanishing/exploding gradient problems~\citep{hanin2018neural} and hence are difficult to be trained well.


This paper aims to mitigate the aforementioned gaps and establishes a holistic framework to train neural networks for model compression. Our framework can be used for both low-rank approximation and weight pruning. It reformulates the neural network training as a nonconvex optimization problem. To solve this problem, a neural network block coordinate descent (NN-BCD) algorithm is proposed. The block coordinate descent method has been recently adapted to neural network training and achieved impressive success recently~\citep{taylor2016training,zhang2017convergent,lau2018proximal}. The advantages of block coordinate descent are twofold. First, it is gradient-free, and thus can deal with non-differentiable objectives and potentially can avoid the vanishing/exploding gradient problems~\citep{hanin2018neural}. Second, it can be easily implemented in a distributed and parallel manner and is compatible with distributed or federated scenarios. These advantages are naturally inherited in our setting as well. In summary, the \textbf{main contributions} of this paper are as follows:
\begin{itemize}
    \item We establish a holistic framework for training neural networks for model compression with a novel nonconvex optimization formulation, which is compatible with two important model compression techniques: low-rank approximation and weight pruning.
    \item We propose an efficient algorithm, NN-BCD, to solve the nonconvex optimization. NN-BCD is gradient-free and can be implemented efficiently.
    \item Theoretically, we analyze the convergence of the iterative sequence generated by the NN-BCD algorithm, which is proved to be globally convergent to a critical point at a rate of $\mathcal{O}(1/k)$. 
    \item We conduct extensive experiments to demonstrate the superior performance of NN-BCD empirically, including tensor train decomposition-based NN and weight pruning. 
\end{itemize}
The remainder of this paper is organized as follows. A brief review of notation and related research work is provided in Section~\ref{sec: background}. The proposed framework and its algorithm are introduced in Section~\ref{sec: proposed method}. Extensive experiments in Section~\ref{sec: experiments} are provided for testing and validation of the proposed method. Finally, the conclusions are discussed in Section~\ref{sec: conclusion}.

\section{Notation and Research Background}\label{sec: background}
In Section~\ref{subsec: tensor basis}, the notation and basics of multi-linear/tensor algebra used in this paper are reviewed.  Due to the efficiency of tensor train decomposition (TTD)~\citep{oseledets2011tensor}, the tensor train fully-connected layer~\citep{novikov2015tensorizing} is reviewed in Section~\ref{subsec: tensor layer}. Several related pruning formulations are introduced in Section~\ref{subsec: weight pruning}.
\subsection{Notations} \label{subsec: tensor basis}
Throughout this paper, scalars are denoted by lowercase letters, for example, $x$; vectors are denoted by lowercase boldface letters, for example, $\bm{x}$; matrices are denoted by uppercase boldface, for example, $\bm{X}$. The order of a tensor is the number of its mode. A real-valued tensor of order-$d$ is denoted by $\bm{\mathcal{X}} \in \mathbb{R}^{n_{1} \times n_{2} \times \cdots \times n_{d}}$ and $\bm{\mathcal{X}}(i_{1}, \cdots, i_{d})$ represents its entry. The inner product of two tensors with the same shape $\bm{\mathcal{X}}$ and $\bm{\mathcal{Y}}$ is the sum of the products of all their entries, \ie,  $ \left\langle \bm{\mathcal{X}},\bm{\mathcal{Y}} \right\rangle =\sum_{i_1 }{\cdots \sum_{i_d }{\bm{\mathcal{X}} \left(i_1,\dots,i_d\right) \cdot \bm{\mathcal{Y}} \left(i_1,\dots,i_d\right)}}$. Moreover, the Frobenius norm of a tensor $\bm{\mathcal{X}}$ is defined as ${\left\|\bm{\mathcal{X}}\right\|}_F=\sqrt{\left\langle \bm{\mathcal{X}},\bm{\mathcal{X}}\right\rangle }$. 

\subsection{Tensor Train Fully-Connected Layer}\label{subsec: tensor layer}
We briefly introduce the tensor train decomposition~\citep{oseledets2011tensor}, which is the state-of-the-art decomposition method in tensorized neural networks. Nevertheless, our framework can be extended to other types of tensor decomposition as well. We call a tensor $\bm{\mathcal{A}} \in \mathbb{R}^{n_{1} \times n_{2} \times \cdots \times n_{d}}$ satisfies the tensor train decomposition if it can be written as the sum of a sequence of order-3 tensors as follows
\begin{equation}\label{eq: ttd}
\begin{aligned}
       \bm{\mathcal{A}}(i_{1}, i_{2}, \cdots, i_{d}) & =\sum_{\alpha_{0}, \alpha_{1} \cdots \alpha_{d}}^{r_{0}, r_{1}, \cdots r_{d}} \bm{\mathcal{G}}_{1}(\alpha_{0}, i_{1}, \alpha_{1}) \bm{\mathcal{G}}_{2}(\alpha_{1}, i_{2}, \alpha_{2}) \cdots 
\bm{\mathcal{G}}_{d}(\alpha_{d-1}, i_{d}, \alpha_{d}) \\
&:=\bm{\mathcal{G}}_{1}(:, i_{1}, :) \bm{\mathcal{G}}_{2}(:, i_{2}, :) \cdots 
\bm{\mathcal{G}}_{d}(:, i_{d}, :) ,
\end{aligned}
\end{equation} where $\bm{\mathcal{G}}_{k} \in \mathbb{R}^{r_{k-1} \times n_{k} \times r_{k}}, k=1,2, \cdots, d$, are called TT-cores and $\bm{r}=[r_{0}, r_{1}, \cdots, r_{d}], r_{0}=r_{d}=1$ are called TT-ranks. Based on TTD, a space of $\sum_{k=1}^d n_k r_{k-1} r_k$ is needed if $\bm{\mathcal{A}}$ is stored in its TT-format while storing all the entries of  $\bm{\mathcal{A}}$ directly requires a space of $\Pi_{k=1}^d n_k$. Thus, the TT-format is very efficient in terms of memory if the ranks are small.  

Based on TTD, the fully-connected layer in a neural network can be represented by a tensor satisfying a certain TT-format, \ie, a TT fully-connected (TT-FC) layer. Consider a single fully-connected layer with weight matrix $\bm{W} \in \mathbb{R}^{M \times N}$, input $\bm{x} \in \mathbb{R}^{N}$, and output $\bm{y} \in \mathbb{R}^{M} $. The output $\bm{y}$ is obtained by $\bm{y}=\bm{W} \bm{x}$. {To transform this standard layer to TT-FC layer,  the weight matrix $\bm{W}$ is tensorized to an order-$d$ tensor $\bm{\mathcal{W}} \in \mathbb{R}^{(m_{1} \times n_{1}) \times \cdots \times(m_{d} \times n_{d})}$ by reshaping and reordering, where $M=\prod_{k=1}^{d} m_{k}$, $N=\prod_{k=1}^{d} n_{k}$, and $m_k,n_k$ are tensor structural parameters. Then $\bm{\mathcal{W}}$ can be represented in TT-format
\begin{equation}\label{eq: tensor layer decomposition}
    \bm{\mathcal{W}}((i_{1}, j_{1}), \cdots,(i_{d}, j_{d}))=\bm{\mathcal{G}}_{1}[:, (i_{1}, j_{1}),:] \cdots \bm{\mathcal{G}}_{d}[:, (i_{d}, j_{d}),:] 
    := \textnormal{TTD}(\bm{r}),
\end{equation}  
where order-4 tensor $\bm{\mathcal{G}}_{k} \in \mathbb{R}^{r_{k-1} \times m_{k} \times n_{k} \times r_{k}}$ are TT-cores with TT-ranks $\bm{r}$. We remark that here $\bm{\mathcal{G}}_{k}$ has one dimension more than the standard format~\eqref{eq: ttd} since the output and input dimensions of $\bm{W}$ are divided separately.  Hence, the forward propagation on the TT-FC layer can be expressed in the tensor format as follows (the bias term is ignored here)
\begin{equation*}
    \bm{\mathcal{Y}}(i_{1}, \cdots, i_{d})=\sum_{j_{1}, \cdots, j_{d}} \bm{\mathcal{G}}_{1}[:, (i_{1}, j_{1}),:] \cdots \bm{\mathcal{G}}_{d}[:, (i_{d}, j_{d}),:] \bm{\mathcal{X}}(j_{1}, \cdots, j_{d}),
\end{equation*}
where $\bm{\mathcal{X}} \in \mathbb{R}^{m_{1} \times \cdots \times m_{d}}$ and $\bm{\mathcal{Y}} \in \mathbb{R}^{n_{1} \times \cdots \times n_{d}}$ are the tensorized input and output corresponding to $\bm{x}$ and $\bm{y}$, respectively. More details about the TT-FC layer are introduced in~\citep{novikov2015tensorizing}.  

\subsection{Weight Pruning} \label{subsec: weight pruning}
One way to do weight pruning is magnitude pruning~\citep{guo2016dynamic,li2018optimization,zhang2018systematic}, which aggregates the previous and following connections of a particular weight. Assuming a $N$-layer neural network associated with $N$ weight matrices $\bm{W}_1,\cdots,\bm{W}_{N}$, a standard magnitude pruning strategy will remove low-magnitude weights first. The goal is to produce a sparse weight matrix for each layer that minimizes the Frobenius norm of the difference between the original weights and the sparse weights
\begin{equation}
    \min_{\bm{W}_i^{sparse}} \|\bm{W}_i-\bm{W}_i^{sparse}\|_F  \text{ subject to } \|\bm{W}_i^{sparse}\|_0 = \beta_i \quad i=1, \ldots, N, 
\end{equation}
where $\|\cdot\|_0$ is the entrywise $\ell_0$ norm and $\beta_i$ is the sparsity level. 

Instead of pruning connections based on magnitude, we can apply various kinds of penalties to weights themselves to make them progressively shrink toward zero. For example, \citet{louizos2018learning} proposes a practical method for $\ell_0$ norm regularization for neural networks:
pruning the network during training by encouraging weights to become exactly zero. Specifically, it considers a regularized empirical risk minimization over weights $\Btheta$ as below
\begin{equation} \label{eq: L0 regularization}
    \min_{\Btheta} \mathcal{R}_n(\Phi(\bm{X} ; \Btheta), \bm{Y}) + \lambda \|\Btheta\|_0,
\end{equation}
where $\mathcal{R}_n$ is the empirical loss,  $\lambda>0$ is a weighting factor for the regularization and the empirical loss. In addition to $\ell_0$ norm, LASSO-type regularization is widely used in the pruning literature~\citep{wen2016learning,he2016deep,yu2017compressing,savarese2020winning}.

\section{Methodology} \label{sec: proposed method}
As shown in Figure~\ref{fig: general structure}, our proposed framework is compatible with both fully-connected and convolutional layers. In terms of model compression, our framework can implement low-rank approximation and weight pruning for both fully-connected and convolutional layers. Specifically, our proposed formulation is introduced in Section~\ref{subsec: formulation}, followed by the proposed NN-BCD algorithm in Section~\ref{subsec: algorithm}. The convergence of the proposed NN-BCD algorithm is studied in Section~\ref{subsec: convergence}.
\begin{figure}[!htbp]
\centering
\includegraphics[width=1\textwidth]{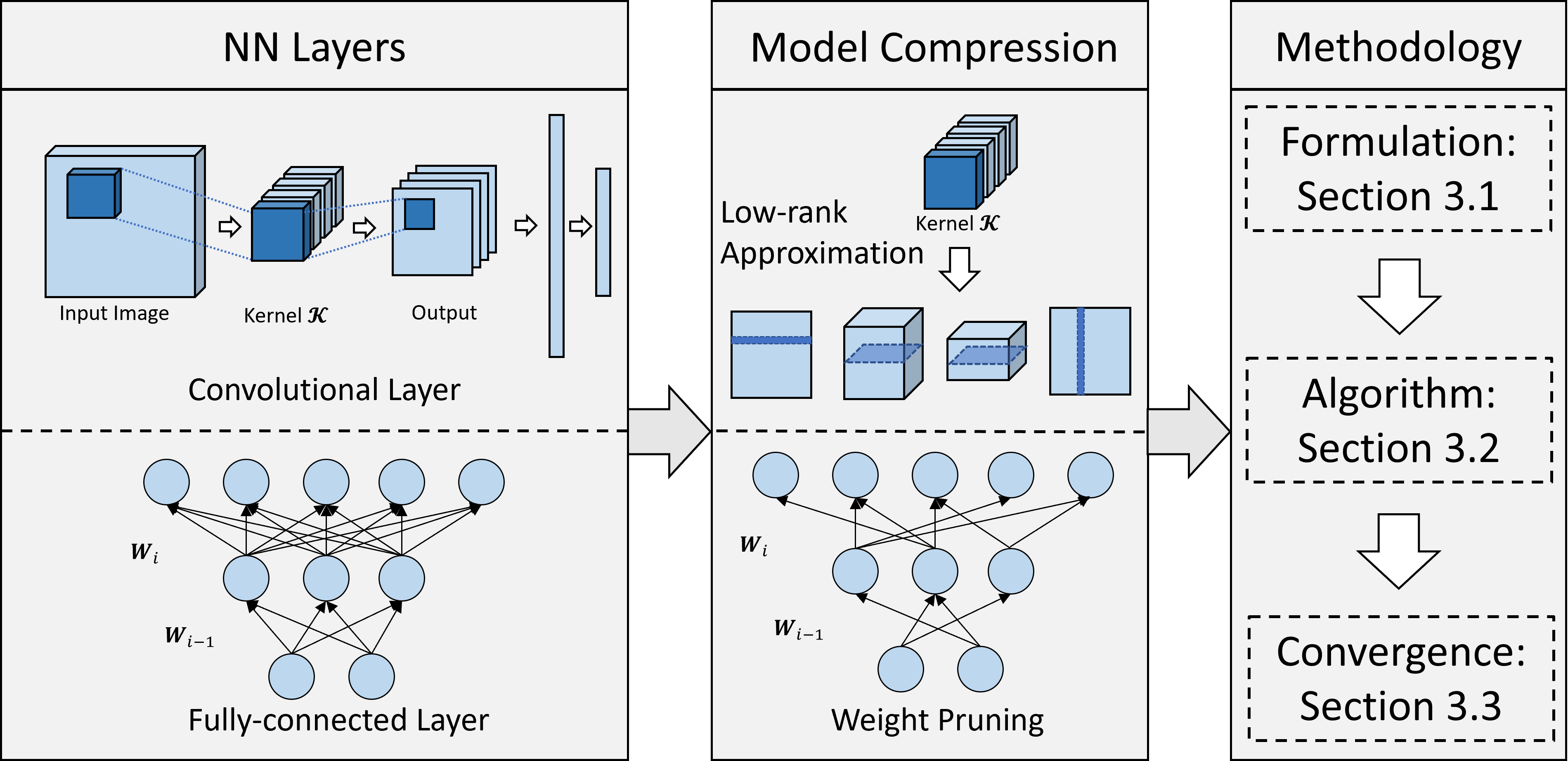} 
\caption{The proposed framework of NN training for model compression.} 
    \label{fig: general structure} 
\end{figure}
\subsection{Problem Formulation}\label{subsec: formulation} 
Consider a $N$-layer fully-connected neural network\footnote{Note: our framework also works with convolutional layers as shown in Appendix~\ref{appendix: convolutional layer}. For simplicity, we only introduce a fully-connected neural network in the main context.}, where $n_i \in \mathbb{N}$ be the number of hidden units ($1\leq i \leq N-1$). $n_0$ and $n_N$ be the number of units of input and output layers. Let $\bm{W}_i \in \mathbb{R}^{n_i \times n_{i-1}}$ be the weight matrix between the $(i-1)$-th layer and the $i$-th layer, where $i=1, \ldots, N$. The dataset is defined as $\mathcal{Z}:=\{(\bm{x}_j, \bm{y}_j)\}_{j=1}^n \subset \mathbb{R}^{n_0} \times \mathbb{R}^{n_N}$, where $\bm{y}_j$'s are the one-hot vectors of labels. To simplify notations, we denote $\Btheta\coloneqq\{\bm{W}_i\}_{i=1}^N, \bm{X}:=(\bm{x}_1, \bm{x}_2, \ldots, \bm{x}_n) $, and $\bm{Y}:=(\bm{y}_1, \bm{y}_2, \ldots, \bm{y}_n)$. 

Model compression usually aims to solve the  following empirical risk minimization problem
\begin{equation} \label{eq: empirical loss}
\min_{\Btheta} \mathcal{R}_n(\Phi(\bm{X} ; \Btheta), \bm{Y}),  \text{ subject to } \mathcal{MC}(\Btheta)=0, 
\end{equation}
where $\mathcal{R}_n(\Phi(\bm{X}; \Btheta), \bm{Y}):=\frac{1}{n}\sum_{j=1}^n \ell(\Phi(\bm{x}_j; \Btheta), \bm{y}_j)$ denotes the empirical loss with loss function $\ell(\cdot,\cdot)$ such as mean squared, logistic, hinge, and cross-entropy functions. $\Phi(\bm{x}_j; \Btheta)=\sigma_N(\bm{W}_N \sigma_{N-1}(\bm{W}_{N-1} \cdots \bm{W}_2 \sigma_1(\bm{W}_1 \bm{x}_j)))$ represents the neural network, where the activation function $\sigma_i(\cdot)$ can be ReLU, leaky ReLU, sigmoid, linear, polynomial, or softplus. $\mathcal{MC}(\cdot)$ is a generalized model compression operator, which can consider both low-rank approximation and weight pruning, respectively. For example,  if $\mathcal{MC}(\Btheta)=0$ is used for weight pruning, it is equivalent to that $\|\bm{W}_i\|_0-\beta_i=0,i=1,\ldots, N$. If it is used  for TTD-based neural networks, $\mathcal{MC}(\Btheta)=0$ is equivalent to the tensorized weight $\bm{\mathcal{W}}_i=\textnormal{TTD}(\bm{r}_i),i=1,\ldots,N$ as defined in~\eqref{eq: tensor layer decomposition}.

\begin{remark}
    The bias terms are not included in our neural network $\Phi(\bm{x}_j; \Btheta)$ for the convenience of the representation. For example, it is $\sigma_1(\bm{W}_1 \bm{x}_j+\bm{b}_1)$ with the bias term (\ie, $\bm{b}_1$) instead of $\sigma_1(\bm{W}_1 \bm{x}_j)$. However, our framework can work with bias terms without additional effort. Our code implementation also includes bias terms. 
\end{remark}


The optimization in~\eqref{eq: empirical loss} is highly nonconvex due to the complicated coupling relation of variables, and thus, very challenging to solve. As a remedy, variable splitting~\citep{taylor2016training,lau2018proximal,wang2023lifted} is widely used to make the problem computationally tractable. At a high level, variable splitting simplifies a complex problem that involves nonlinearly coupled variables by introducing auxiliary variables, resulting in a problem with much looser variable coupling. By applying variable splitting in our setting, we introduce auxiliary variables $\bm{U}_i$, $\bm{V}_i$, and the compressed weight $\bm{W}_i^{MC}$ to reformulate each layer of the neural network as $ \bm{U}_{i}=\bm{W}_{i}\bm{V}_{i-1},\bm{V}_{i} =\sigma_{i}(\bm{U}_{i}), \bm{W}_i=\bm{W}_i^{MC}$. Define $\bm{\mathcal{U}}:=\{\bm{U}_{i}\}_{i=1}^{N}$,  $\bm{\mathcal{V}}:=\{\bm{V}_{i}\}_{i=1}^{N}$, and $\Btheta_{MC}\coloneqq\{\bm{W}_i^{MC}\}_{i=1}^N$, we have the following optimization problem
\begin{equation}\label{eq: problem formulation}
    \begin{aligned}
        \min_{\Btheta,\Btheta_{MC},\bm{\mathcal{U}},\bm{\mathcal{V}}} \mathcal{L}_{0}(\Btheta,\Btheta_{MC}, \bm{\mathcal{V}}) & :=\mathcal{R}_{n}(\bm{V}_{N} ; \bm{Y})+\sum_{i=1}^{N} r_{i}(\bm{W}_{i}^{MC})+\sum_{i=1}^{N} s_{i}(\bm{V}_{i}) \\
    \text{ subject to } \bm{U}_{i}=\bm{W}_{i} \bm{V}_{i-1},\bm{V}_{i} &=\sigma_{i}(\bm{U}_{i}), \bm{W}_i=\bm{W}_i^{MC}, \mathcal{MC}(\Btheta_{MC})=0, 
    \end{aligned}
\end{equation}
where $\mathcal{R}_{n}(\bm{V}_{N} ; \bm{Y})=\frac{1}{n} \sum_{j=1}^{n} \ell((\bm{V}_{N})_{: j}, \bm{y}_{j})$ is the empirical risk, $(\bm{V}_{N})_{: j}$ denotes the $j$-th column of $\bm{V}_{N}$, $\bm{V}_{0}=\bm{X}$. Note that we incorporate regularization functions $r_{i}(\cdot)$ and $s_{i}(\cdot)$ to reveal the priors of the compressed weight $\bm{W}_{i}^{MC}$ and the variable $\bm{V}_{i}$. For example, for the regularization of $\bm{W}_i^{MC}$, $r_{i}$ can be used for most of the regularization-based pruning~\citep{wang2020neural} including the squared $\ell_{2}$ norm, the $\ell_{1}$ norm, the grouped lasso, the $\ell_{0}$ norm, etc. The regularization term for $\bm{V}_i$ can be the $\ell_{1}$ norm or the indicator function of some convex set with simple projection. Setting $r_{i}$ or $s_{i}$ as zero means that no regularization or constraint is placed on $\bm{W}_{i}^{MC}$ or $\bm{V}_{i}$. 

Instead of considering the constrained optimization in~\eqref{eq: problem formulation}, which can be challenging to solve, the following formulation is considered by adding some constraints as regularization terms in the objective function
\begin{equation} \label{eq: final formulation}
\begin{aligned}
     \min_{\Btheta, \Btheta_{MC}, \bm{\mathcal{U}}, \bm{\mathcal{V}} } & \mathcal{L}(\Btheta, \Btheta_{MC}, \bm{\mathcal{V}},\bm{\mathcal{U}}):=  \mathcal{L}_{0}(\Btheta,\Btheta_{MC}, \bm{\mathcal{V}}) +\frac{\rho}{2} \sum_{i=1}^{N}\|\bm{U}_{i}-\bm{W}_{i} \bm{V}_{i-1}\|_{F}^{2}  \\
     &+\frac{\gamma}{2} \sum_{i=1}^{N}\|\bm{V}_{i}-\sigma_{i}(\bm{U}_{i})\|_{F}^{2}+ \frac{\tau}{2} \sum_{i=1}^{N}\|\bm{W}_i - \bm{W}_i^{MC}\|_{F}^{2}, \\
     \text{ subject to } & \mathcal{MC}(\Btheta_{MC})=0, 
\end{aligned}
\end{equation} 
where $\rho,\gamma,\tau>0$ are hyperparameters that capture the penalty for constraints violation. $\Btheta_{MC}$ is the set of compressed weights, which is \textit{our research of interest}.  Our framework is flexible enough and compatible with many model compression techniques, such as low-rank approximation, regularization-based pruning, and $\ell_0$ norm constrained. 
\begin{remark}[Advantages of Formulation~\eqref{eq: final formulation}]  
As mentioned before, existing NN training for model compression is either (1) Training from scratch or (2) Decomposing/pruning a pre-trained model and then retraining. For the first strategy, it does not utilize any information related to the high-accuracy uncompressed model. Proper utilization of the high-accuracy uncompressed model is very critical for NN compression. For the second strategy, though the knowledge of the pre-trained model is indeed utilized, the pre-trained model generally lacks low-rank/sparse property after direct low-rank approximation/pruning error is too significant to be properly recovered even using long-time re-training. Consequently, such inherent limitations of the existing training strategies cause significant accuracy loss for the compressed NNs. To overcome these limitations, it is to maximally retain the knowledge contained in the uncompressed model, or in other words, minimize the approximation error after low-rank approximation/pruning with given target ranks/sparsity. In our formulation~\eqref{eq: final formulation}, $\mathcal{L}_{0}(\Btheta,\Btheta_{MC}, \bm{\mathcal{V}})$ is the loss function of the uncompressed model while the regularization term $\|\bm{W}_i - \bm{W}_i^{MC}\|_{F}^{2}$ can encourage the uncompressed NN to gradually exhibit low-rank/sparse property.
\end{remark}


\subsection{Neural Network BCD Algorithm} \label{subsec: algorithm}
To solve the optimization problem in \eqref{eq: final formulation}, note that the objective is a nonconvex function with multi-block variables $\Btheta,\Btheta_{MC},\bm{\mathcal{U}},\bm{\mathcal{V}}$. Thus, we can apply the block coordinate descent algorithm, which is a Gauss-Seidel type method, to address objectives with multi-block variable structure~\citep{attouch2013convergence}. It iteratively updates one block of variables at a time, while keeping the remaining blocks fixed. Based on this idea, we propose the neural network block coordinate descent (NN-BCD) algorithm to solve~\eqref{eq: final formulation}. In order to make the training process more stable and to achieve the sufficient decrease property for theoretical justification, some proximal terms are added to some sub-problems arising from the NN-BCD algorithm if the original sub-problems are not strongly convex.  

In more detail, at each iteration, the NN-BCD with backward order is used to update variables, \ie, the variables are updated from the output layer (layer $N$) to the input layer (layer $1$) iteratively. Within each layer, the variables $\{\bm{V}_{i}, \bm{U}_{i}, \bm{W}_{i},\bm{W}_{i}^{MC}\}$ are updated cyclically based on formulation~\eqref{eq: final formulation}. Since the activation function in the last layer is an identical function (\ie, $\sigma_{N} \equiv \mathrm{Id}$), the optimization for the output layer is updated differently from other layers. The details of NN-BCD algorithms are presented below and summarized in Algorithm~\ref{alg: NN-BCD}.


\begin{algorithm}[!htbp]
\caption{NN-BCD Algorithm}\label{alg: NN-BCD}
\textbf{Input:} Sample  $\bm{X} \in \mathbb{R}^{n_0 \times n}$ and $\bm{Y} \in \mathbb{R}^{n_N \times n}$, $\gamma,\rho,\tau,\alpha>0$ \\
\textbf{Initialization:} $\{\bm{V}_i^0, \bm{U}_i^0,\bm{W}_i^0,\bm{W}_i^{MC,0}\}_{i=1}^N, \bm{V}_0^k \equiv \bm{V}_0:=\bm{X}$   
\begin{algorithmic}[1]
\For{$k=1,\ldots$}
\State Update $\bm{V}_N^k$ by solving~\eqref{eq: V_N update}

\State Update $\bm{U}_N^k$ by solving~\eqref{eq: U_N update}

\State Update $\bm{W}_N^k$ by solving~\eqref{eq: W_i update}

\State Update $\bm{W}_N^{MC,k}$ by solving~\eqref{eq: W_i^MC update}
      \For{$i=N-1,\ldots,1$}     
      \State Update $\bm{V}_i^k$ by solving~\eqref{eq: V_i update}
      
      \State Update $\bm{U}_i^k$ by solving~\eqref{eq: U_i update}

      \State Update $\bm{W}_i^k$ by solving~\eqref{eq: W_i update}

      \State Update $\bm{W}_i^{MC,k}$ by solving~\eqref{eq: W_i^MC update}
      \EndFor
\EndFor
\end{algorithmic}
\textbf{Output:} $\{\bm{W}_i^{MC}\}_{i=1}^N$
\end{algorithm}

\textbf{Optimization over $\bm{V_i}$:} At iteration $k$, $\bm{V}_N$ can be updated through the following optimization problem 
\begin{equation}\label{eq: V_N update}
    \bm{V}_N^k=\underset{\bm{V}_N}{\operatorname{argmin}} \left\{s_N(\bm{V}_N)+\mathcal{R}_n(\bm{V}_N ; \bm{Y})+\frac{\gamma}{2} \| \bm{V}_N-\bm{U}_N^{k-1}\|_F^2+\frac{\alpha}{2}\|\bm{V}_N-\bm{V}_N^{k-1} \|_F^2 \right\},
\end{equation}
{where $\tilde{s}_{N}(\bm{V}_{N}):=s_{N}(\bm{V}_{N})+\mathcal{R}_{n}(\bm{V}_{N} ; \bm{Y})$ is regarded as a new proximal function}  and $\frac{\alpha}{2}\|\bm{V}_N-\bm{V}_N^{k-1} \|_F^2$ is the proximal term. When $i<N$, $\bm{V}_i$ can be updated through the following optimization problem 
\begin{equation}\label{eq: V_i update}
    \bm{V}_i^k=\underset{\bm{V}_i}{\operatorname{argmin}}
     \left\{s_i(\bm{V}_i)+\frac{\gamma}{2} \| \bm{V}_i-\sigma_i(\bm{U}_i^{k-1})\|_F^2+\frac{\rho}{2} \| \bm{U}_{i+1}^k-\bm{W}_{i+1}^k\bm{V}_i\|_F^2 \right\}. 
\end{equation}
 Note that subproblems~\eqref{eq: V_i update} is a simple proximal update~\citep{attouch2013convergence,bolte2014proximal}, which is the least square minimization that admits closed-form solutions for many commonly used neural networks.  Some typical examples leading to the closed-form solutions include (a) regularization terms $s_{i}$ are zero (\ie, no regularization), or the squared $\ell_{2}$ norm; (b) the indicator function of a nonempty closed convex set with a simple projection like the nonnegative closed half-space and the closed interval $[0,1]$; (c) the loss function $\ell$ is the squared loss or hinge loss. More details are discussed in Appendix~\ref{appendix: subproblems V_N}.

\textbf{Optimization over $\bm{{U}_i}$:}
At iteration $k$, $\bm{U}_N$ can be updated through the following optimization problem 
\begin{equation} \label{eq: U_N update}
    \bm{U}_N^k=\underset{\bm{U}_N}{\operatorname{argmin}} \left\{\frac{\gamma}{2}\|\bm{V}_N^k-\bm{U}_N\|_F^2+\frac{\rho}{2} \| \bm{U}_N-\bm{W}_N^{k-1} \bm{V}_{N-1}^{k-1} \|_F^2 \right\},
\end{equation}
and $\bm{U}_i,i<N$ can be updated through 
\begin{equation}\label{eq: U_i update}
\bm{U}_i^k=\underset{\bm{U}_i}{\operatorname{argmin}} \left\{\frac{\gamma}{2}\|\bm{V}_i^k-\sigma_i(\bm{U}_i)\|_F^2+\frac{\rho}{2} \| \bm{U}_i-\bm{W}_i^{k-1} \bm{V}_{i-1}^{k-1} \|_F^2  +\frac{\alpha}{2}\|\bm{U}_i-\bm{U}_i^{k-1} \|_F^2 \right\},
\end{equation}
where $\frac{\alpha}{2}\|\bm{U}_i-\bm{U}_i^{k-1} \|_F^2$ is the proximal term. Subproblem~\eqref{eq: U_N update} is a least square optimization that admits a closed-form solution. Moreover, the subproblem~\eqref{eq: U_i update}  is a nonlinear and nonsmooth where  $\sigma_{i}$ is ReLU or leaky ReLU. In this case, the closed-form solution to solve the subproblem~\eqref{eq: U_i update} is provided in Appendix~\ref{appendix: subproblems Relu}.

\textbf{Optimization over $\bm{{W}_i}$:} At iteration $k$, $\bm{W}_i,i=1,\dots,N$ can be updated through the following optimization problem 
\begin{equation}\label{eq: W_i update}
\bm{W}_i^k=\underset{\bm{W}_i}{\operatorname{argmin}} \left\{\frac{\rho}{2}\|\bm{U}_i^k-\bm{W}_i \bm{V}_{i-1}^{k-1}\|_F^2+\frac{\tau}{2}\|\bm{W}_i - \bm{W}_i^{MC}\|_{F}^{2} \right\}, 
\end{equation}
Again, the closed-form solution to solve the above optimization problem can be obtained since it is a least square problem. 

\textbf{Optimization over $\bm{W}_i^{MC}$:} At iteration $k$, $\bm{W}_i^{MC},i=1,\dots,N$ can be updated through the following optimization problem
\begin{equation}\label{eq: W_i^MC update}
\begin{aligned}
      \bm{W}_i^{MC,k} & =\underset{\bm{W}_i^{MC}}{\operatorname{argmin}} \left\{r_i(\bm{W}_i^{MC})+\frac{\tau}{2}\|\bm{W}_i^k - \bm{W}_i^{MC}\|_{F}^{2} +\frac{\alpha}{2}\|\bm{W}_i^{MC}-\bm{W}_i^{MC,k-1}\|_F^2 \right\} \\
      \text{ subject to } & \mathcal{MC}(\bm{W}_i^{MC})=0,
\end{aligned}
\end{equation} where $\frac{\alpha}{2} \|\bm{W}_i^{MC}-\bm{W}_i^{MC,k-1}\|_F^2$ is the proximal term. We remark that this subproblem is also a least square problem and can be implemented efficiently. Examples of low-rank approximation and weight pruning are provided in Appendix~\ref{appendix: subproblems MC}, where closed-form solutions can be obtained. 

\subsection{Convergence Analysis} \label{subsec: convergence}
The following definitions are necessary in order to introduce our main theoretical results.
 \begin{definition}[\normalfont Critical point~\citep{attouch2009convergence, attouch2010proximal}]~A necessary condition for $\bm{x}$ to be a minimizer of a proper and lower semicontinuous (PLSC) function  $f$ is that
 \begin{equation} \label{eq: critical point}
     \bm{0} \in \partial f(\bm{x}).
 \end{equation}
 A point that satisfies \eqref{eq: critical point} is called limiting-critical or simply critical.
\end{definition}

\begin{definition}[\normalfont Global convergence~\citep{petrovai2017global,xu2018convergence}]  \label{def: global convergence}
Any iterative algorithm for solving an optimization problem over a set $X$, is said to be \textbf{globally convergent} if for any starting point $\bm x_0 \in X$, the sequence generated by the algorithm always has an accumulation critical point.
\end{definition}

The global convergence property of Algorithm~\ref{alg: NN-BCD} is established by proving a sufficient decrease property, a subgradient bound, and analyzing the Kurdyka-\L{}ojasiewicz (K\L) property of the objective function in~\eqref{eq: final formulation}.

In what follows, let $$\{\mathcal{P}^{k}\}_{k \in \mathbb{N}}:=\Big\{\Big(\{\bm{W}_{i}^{k}\}_{i=1}^{N},\{\bm{V}_{i}^{k}\}_{i=1}^{N},\{\bm{U}_{i}^{k}\}_{i=1}^{N},\{\bm{W}_{i}^{MC,k}\}_{i=1}^{N}\Big)\Big\}_{k \in \mathbb{N}}$$ be the iterative sequence  generated by Algorithm~\ref{alg: NN-BCD}.
The following lemma establishes the sufficient decrease property of sequence  $\{\mathcal{P}^{k}\}_{k \in \mathbb{N}} $.

\begin{lemma}[Sufficient Decrease Property] \label{lemma: sufficient decrease}
Suppose $\alpha,\gamma,\rho,\tau>0$ and $\left\{\mathcal{P}^{k}\right\}_{k \in \mathbb{N}}$ is the sequence generated by the NN-BCD algorithm~\ref{alg: NN-BCD}. Then we have 
\begin{equation}\label{eq: sufficient decrease}
 \mathcal{L}(\mathcal{P}^{k}) \leq \mathcal{L}(\mathcal{P}^{k-1})-\lambda\|\mathcal{P}^{k}-\mathcal{P}^{k-1}\|_{F}^{2}.
\end{equation}
 where $\lambda=\min \left\{ \alpha, \gamma+\rho,\tau\right\}/2$. 
\end{lemma}

Lemma~\ref{lemma: sufficient decrease} is crucial for the overall convergence of a nonconvex problem. It builds on the proximal update scheme for all non-strongly convex subproblems defined in Algorithm~\ref{alg: NN-BCD}. Its detailed proof is provided in Appendix~\ref{appendix: proof of sufficient decrease}. According to Lemma~\ref{lemma: sufficient decrease}, the sequence $\{\mathcal{L}(\mathcal{P}^{k})\}_{_{k \in \mathbb{N}}}$ generated from our algorithm is monotone decreasing. The descent quantity for each iteration is bounded by the discrepancy between the previous and current iterations. In contrast, existing literature, such as \citet{davis2020stochastic}, demonstrates the convergence of subsequences for stochastic gradient descent (SGD) for NN training. However, our NN-BCD algorithm~\ref{alg: NN-BCD} in this study ensures the convergence of the whole sequence. The main distinction between the subsequence convergence of SGD and the whole sequence convergence of our NN-BCD algorithm is primarily due to the fact that SGD can achieve the descent property instead of the sufficient descent property. 


In contrast to existing literature~\citep{attouch2013convergence, xu2013block, bolte2014proximal,   xu2017globally}, which requires multiconvexity, differentiability, or Lipschitz differentiability assumptions for neural networks, the assumptions for Lemma~\ref{lemma: sufficient decrease} are greatly relaxed. Another lemma crucial to our convergence analysis establishes a bound to the subgradients and its proof is provided in Appendix~\ref{appendix: proof of subgradient low bound}.  
\begin{lemma}[Subgradient Bound] \label{lemma: subgradient low bound}
    Under the same assumptions of Lemma~\ref{lemma: sufficient decrease}, let $\mathcal{B}$ be a uniform upper bound of $\mathcal{P}^{k}$ for all $k, L_{\mathcal{B}}$ be a uniform Lipschitz constant of activation function $\sigma_{i}$ on the bounded set $\{\mathcal{P}:\|\mathcal{P}\|_{F} \leq \mathcal{B}\}$, and
$\delta:=\max \{\gamma, \alpha+\rho \mathcal{B}, \alpha+\gamma L_{\mathcal{B}}, 2 \rho \mathcal{B}+ 2\rho \mathcal{B}^{2}, \alpha +\tau \}$. Then for any $k$, we have
\begin{equation}\label{eq: subgradient bound}
    \begin{aligned}
        \operatorname{dist}(\mathbf{0}, \partial \mathcal{L} (\mathcal{P}^{k}))
        \leq  &\delta \sum_{i=1}^{N}\Big[ \|\bm{W}_{i}^{k}-\bm{W}_{i}^{k-1}\|_{F} +\|\bm{V}_{i}^{k}-\bm{V}_{i}^{k-1}\|_{F} \\
        & +\|\bm{U}_{i}^{k}-\bm{U}_{i}^{k-1}\|_{F}+\|\bm{W}_{i}^{MC,k}-\bm{W}_{i}^{MC,k-1}\|_{F} \Big] \\
         \leq & \bar{\delta}\|\mathcal{P}^{k}-\mathcal{P}^{k-1}\|_{F},
    \end{aligned}
\end{equation}
where $\bar{\delta}= \delta \sqrt{4 N}$, $\operatorname{dist}(\mathbf{0}, \mathcal{S})=\inf _{\bm{s} \in \mathcal{S}}\|\bm{s}\|_{F}$ denotes the distance of $\mathbf{0}$  to a set $\mathcal{S}$, and
\begin{equation*}
    \partial \mathcal{L}\big(\mathcal{P}^{k})=\Big(\big\{\partial_{\bm{W}_{i}} \mathcal{L}(\mathcal{P}^{k})\big\}_{i=1}^{N}, \big\{\partial_{\bm{V}_{i}} \mathcal{L}(\mathcal{P}^{k})\big\}_{i=1}^{N},\big\{\partial_{\bm{U}_{i}} \mathcal{L}(\mathcal{P}^{k})\big\}_{i=1}^{N},\big\{\partial_{\bm{W}_i^{MC}} \mathcal{L}(\mathcal{P}^{k})\big\}_{i=1}^{N} \Big).
\end{equation*}
\end{lemma} 

Based on Lemmas~\ref{lemma: sufficient decrease} and~\ref{lemma: subgradient low bound}, the following theorem can be obtained. 
\begin{theorem}[Global Convergence of NN-BCD] \label{thm: global convergence}
Let $\{\mathcal{P}^{k}\}_{k \in \mathbb{N} } $
be the sequences generated from Algorithm~\ref{alg: NN-BCD}.   Suppose that $r_{i}$  and $\mathcal{L}$ are coercive\footnote{An extended-real-valued function $h$ is called coercive if and only if $h(\bm{x}) \to +\infty$ as $\|\bm{x}\|\to \infty$.} for any $i=1, \ldots, N$. Then for any $\alpha,\gamma,\rho,\tau>0$ and any finite initialization $\mathcal{P}^{0}$, the following statements hold
\begin{enumerate}
    \item $\{\mathcal{L}(\mathcal{P}^{k})\}_{k \in \mathbb{N}}$ converges to $\mathcal{L}^{*}$, which is the unique convergent value of the whole sequence.
    \item $\{\mathcal{P}^{k}\}_{k \in \mathbb{N}}$  converges to a critical point of $\mathcal{L}$ in~\eqref{eq: final formulation}.
    \item If the initialization $\mathcal{P}^{0}$ is sufficiently close to some global minimum $\mathcal{P}^{*}$ of $\mathcal{L}$, then $\mathcal{P}^{k}$  converges to $\mathcal{P}^{*}$.
    \item The averaged subgradient ${1}/{K}\cdot  \sum_{k=1}^{K}\|\bm{g}^{k}\|_{F}^{2}$ converges to zero with rate $\mathcal{O}(1 / K)$, where $\bm{g}^{k} \in$ $\partial \mathcal{L}(\mathcal{P}^{k})$. 
\end{enumerate}
\end{theorem}
In summary, Theorem \ref{thm: global convergence} guarantees that the NN-BCD algorithm converges to a critical point of the objective at a rate of $\mathcal{O}(1/k)$, where $k$ denotes the iteration number. Its detailed proof is provided in Appendix~\ref{appendix: theorem of global convergence}. 

Remarkably, in most existing literature, either multiconvexity or Lipschitz differentiability assumption is required to establish the convergence of nonconvex optimizations with multi-block variables~\citep{xu2013block, xu2017globally}. However, the neural networks involved in~\eqref{eq: final formulation} may not satisfy these requirements typically. For example, the ReLU activation function is non-differentiable and nonconvex. In contrast, the assumptions adopted in our analysis are quite mild. We solely rely on the Lipschitz continuity of the activation functions on a bounded set, which is met by the majority of commonly used activation functions. Theorem~\ref{thm: global convergence} demonstrates that global convergence can be achieved under the assumptions that most neural networks satisfy, which is verified by our experiments in Section~\ref{sec: experiments}. 
\section{Experiments} \label{sec: experiments}
To evaluate the performance of the proposed NN-BCD algorithm~\ref{alg: NN-BCD}, different neural network structures with different datasets are considered:
\begin{itemize}
    \item Section~\ref{subsec: case study CNN}: we studied TTD-based CNN with the MNIST dataset. Specifically, $r_i$'s and $s_i$'s in \eqref{eq: final formulation} are set to zero. For $i=1,\dots,N$, $\mathcal{MC}(\bm{\mathcal{W}}_i^{MC})=0$ means that $\bm{\mathcal{W}}_i^{MC} = \textnormal{TTD}(\bm{r}_i)$, where TT-ranks $\bm{r}_i$ can be changed to obtain different compression ratios.
    \item Section~\ref{subsec: case study HAR}: we studied TTD-based  Multilayer Perceptron (MLP) with the UCI-HAR (Human Activity Recognition)\footnote{\url{https://archive.ics.uci.edu/dataset/240/human+activity+recognition+using+smartphones}} dataset~\citep{Reyes2012HAR}. Specifically, $r_i$'s and $s_i$'s in \eqref{eq: final formulation} are set to zero. For $i=1,\dots,N$, $\mathcal{MC}(\bm{\mathcal{W}}_i^{MC})=0$ means that  $\bm{\mathcal{W}}_i^{MC} = \textnormal{TTD}(\bm{r}_i)$, where TT-ranks $\bm{r}_i$ can be changed to obtain different compression ratios. Additional experiments with a deeper MLP structure are conducted in Appendix~\ref{appendix: case study HAR 5 layers} 
    \item Section~\ref{subsec: case study flare}: we studied weight pruning with the flare classification\footnote{\url{https://web.njit.edu/~wangj/LSTMpredict/}} dataset~\citep{liu2019predicting}. Specifically, $r_i$'s and $s_i$'s in \eqref{eq: final formulation} are set to zero. For $i=1,\dots,N$, $\mathcal{MC}(\bm{W}_i^{MC})=0$ means that $\|\bm{W}_i^{MC}\|_0=\beta_i$, where sparsity level $\beta_i$ can be changed.
\end{itemize}
 The effectiveness of the compression can be measured by \textbf{Compression Ratio (CR)}, which is defined as 
\begin{equation*}
    \textnormal{Compression Ratio}=\frac{\text{the number of weights after compression}}{\text{the number of weights without compression}}.    
\end{equation*} A smaller value of CR indicates a better compression performance. 

All experiments in this section apply the squared loss function (\ie, $\|\cdot\|_2^2$). The same initializations are set for each experiment.  Specifically, all the weights $\{\bm{W}_i\}_{i=1}^N$ are initialized from a Gaussian distribution with a standard deviation of 0.01. The auxiliary variables $\{\bm{U}_i\}_{i=1}^N$, state variables $\{\bm{V}_i\}_{i=1}^N$, and the compressed weights $\{\bm{W}_i^{MC}\}_{i=1}^N$ are initialized by a single forward pass~\citep{zeng2019global,zeng2021admm}. All results in this section are the average results of ten repetitions for comparison. We trained our model through high-performance computing from NJIT. The codes of NN-BCD are implemented in Python 3.7. The GPU we used for model training is a single NVIDIA Tesla P100 16GB. 

\subsection{Experiments on Tensorized CNN} \label{subsec: case study CNN}
The MNIST dataset, which is a handwritten digits dataset, is used to evaluate the effectiveness and efficiency of our proposed method. The numbers of training and test samples are 60,000 and 10,000, respectively. The size of each input image is $28\times28$ and the output dimension is the number of classes (\ie, $10$).  In this experiment,  the CNN architecture has one convolution layer and two hidden fully-connected layers. The size of the kernel tensor $\bm{\mathcal{K}}$ is $3\times 3\times 32$ while the number of hidden units in each fully-connected layer is $2^{10}=1024$. Our CNN also uses the ReLU activation function.  


\begin{table}[!htbp]
\centering
 \caption{Results of NN-BCD algorithm with different compression ratios (CNN MNIST).}
 \label{tab: Mnist CNN CR ratio}
  \begin{tabular}{lrrrrrr}
    \toprule
    & \multicolumn{2}{c}{Training Loss~\eqref{eq: final formulation}}                     & \multicolumn{2}{c}{Training Accuracy}                       & \multicolumn{2}{c}{Test Accuracy}                           \\ \midrule
 CR     & Mean  & Std   & Mean  & Std   & Mean  & Std  \\ \midrule
0.0128	&0.0327	&0.0023	&0.9882	&0.0015	&0.9738	&0.0045\\
0.0370	&0.0185	&0.0002	&0.9995	&0.0004	&0.9825	&0.0026\\
0.0878	&0.0115	&0.0007	&1.0000	&0.0001	&0.9828	&0.0029\\
0.1784	&0.0074	&0.0008	&1.0000	&0.0000	&0.9822	&0.0024\\
0.2807	&0.0054	&0.0008	&1.0000	&0.0000	&0.9818	&0.0024\\
1.0000	&0.0038	&0.0001	&0.9966	&0.0019	&0.9709	&0.0051\\
    \bottomrule
  \end{tabular}
\end{table}  
 The same hyperparameter setting, \ie, $\gamma=5, \rho=5, \tau=0.1, \alpha=1$ in \eqref{eq: final formulation}, is used for the CNN model training with different compression ratios, where the compression ratios are determined by TT-rank $\bm{r}_i$. The training loss~\eqref{eq: final formulation}, training accuracy, and test accuracy results are shown in Table~\ref{tab: Mnist CNN CR ratio}. With a smaller compression ratio (with CR < 1), a higher training loss and lower training/test accuracy are observed since the compressed CNN with a smaller compression ratio has a larger model approximation error (\ie, the term $\|\bm{\mathcal{W}}_i - \bm{\mathcal{W}}_i^{MC}\|_{F}^{2}$ in \eqref{eq: final formulation} is larger). From Table~\ref{tab: Mnist CNN CR ratio}, our method with CR < 1 can outperform the uncompressed model (CR = 1). When CR = 0.0128, the test accuracy is 0.9738, which is slightly better than the uncompressed model with test accuracy = 0.9709. This result shows that our method not only compresses the model but also improves the test performance. \begin{figure}[!htbp] 
\centering
	\subfloat[Training loss]{\includegraphics[width=0.33\textwidth]{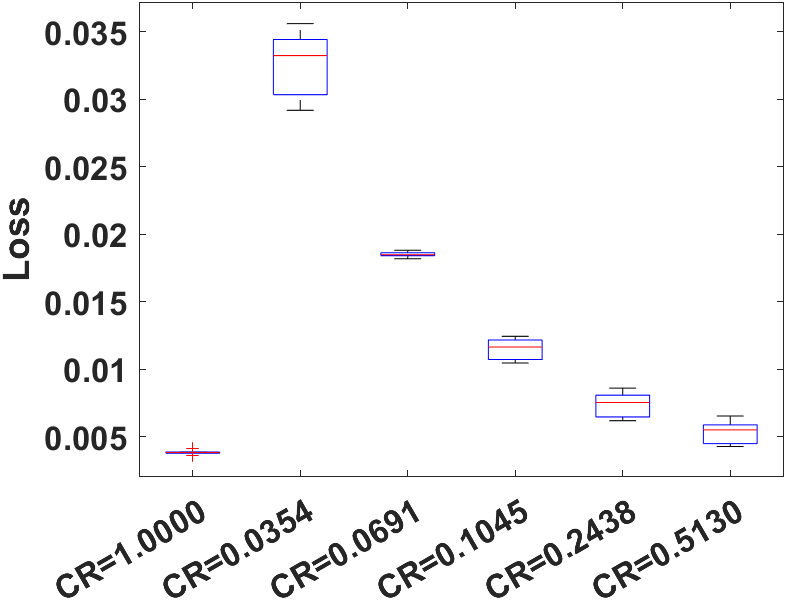} \label{subfig: MnistCNNTrainingLossBox}}
    \subfloat[Training Accuracy]{\includegraphics[width=0.33\textwidth]{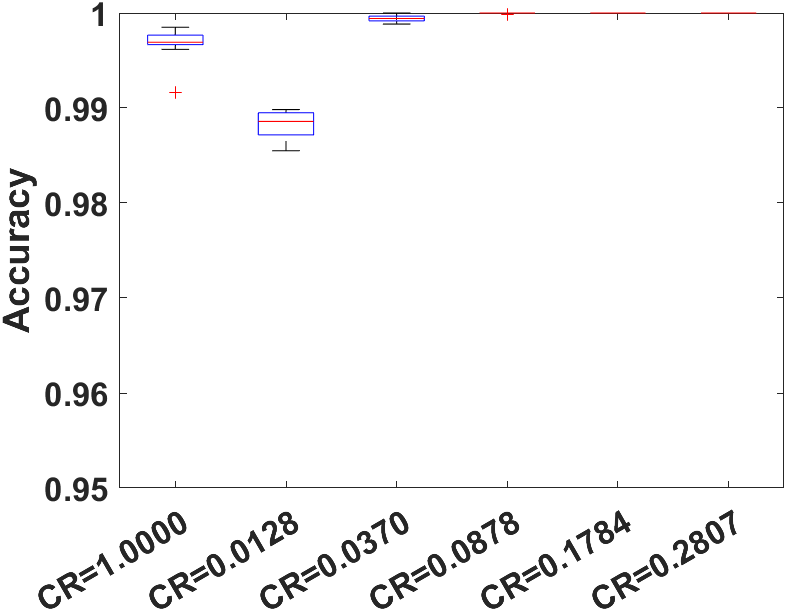} \label{subfig: MnistCNNTrainingAccBox}}
    \subfloat[Test Accuracy]{\includegraphics[width=0.33\textwidth]{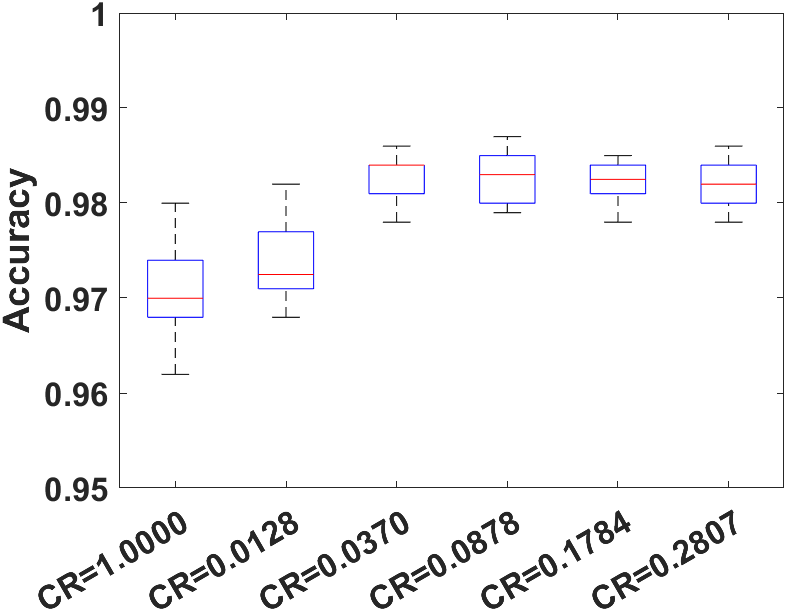} \label{subfig: MnistCNNTestAccBox}}
\caption{The boxplots among ten repetitions with different compression ratios (CNN MNIST): (a) training loss; (b) training accuracy; (c) test accuracy.} 
\label{fig: Mnist CNN box}
\end{figure}  The boxplots of the last iteration's training loss, training accuracy, and test accuracy among ten repetitions with different compression ratios are shown in Figure~\ref{fig: Mnist CNN box}. Our method with different CRs has a very small variation among ten repetitions in terms of training and test accuracy, showing that our method is stable in both training and testing. 


\begin{figure}[!htbp]
\centering
	\subfloat[Training Loss]{\includegraphics[width=0.33\textwidth]{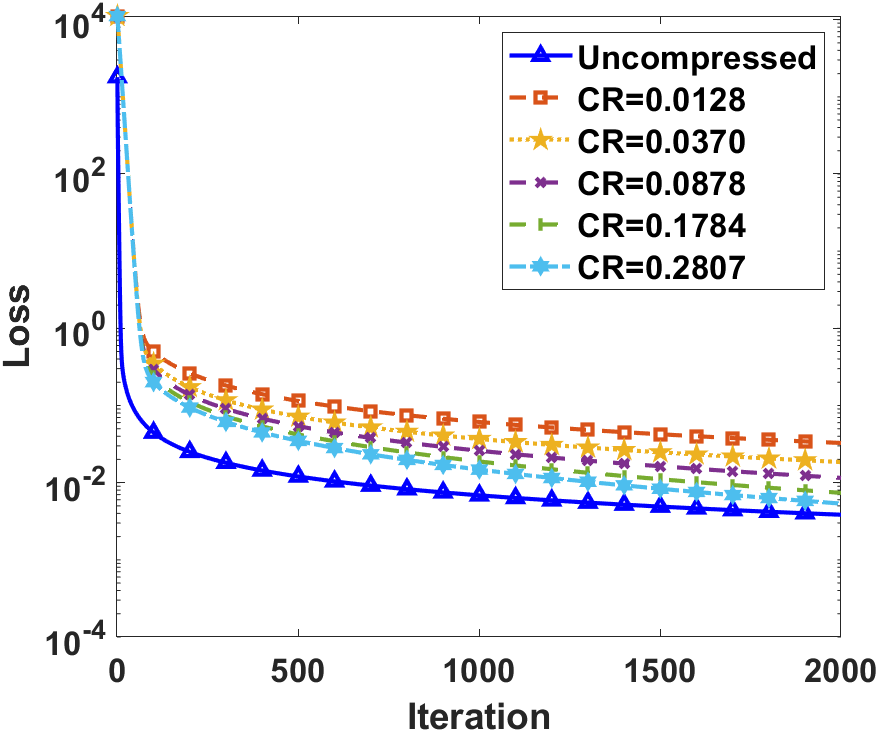} \label{subfig: MnistCNNTrainingLoss}}
     \subfloat[Training Error Rate]{\includegraphics[width=0.33\textwidth]{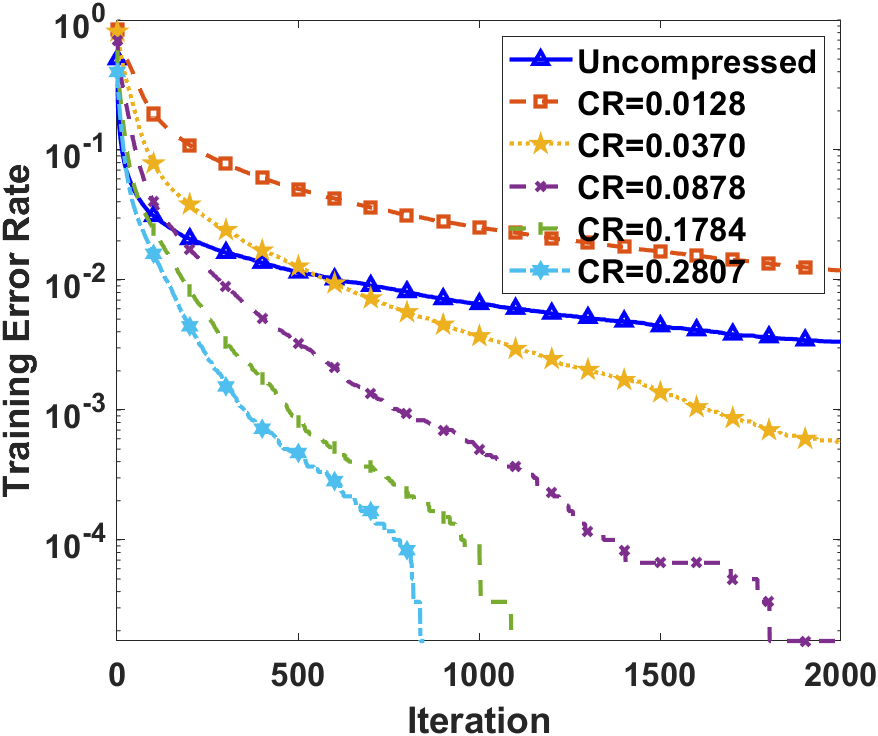} \label{subfig: MnistCNNTrainError}}
    \subfloat[Test Error rate]{\includegraphics[width=0.33\textwidth]{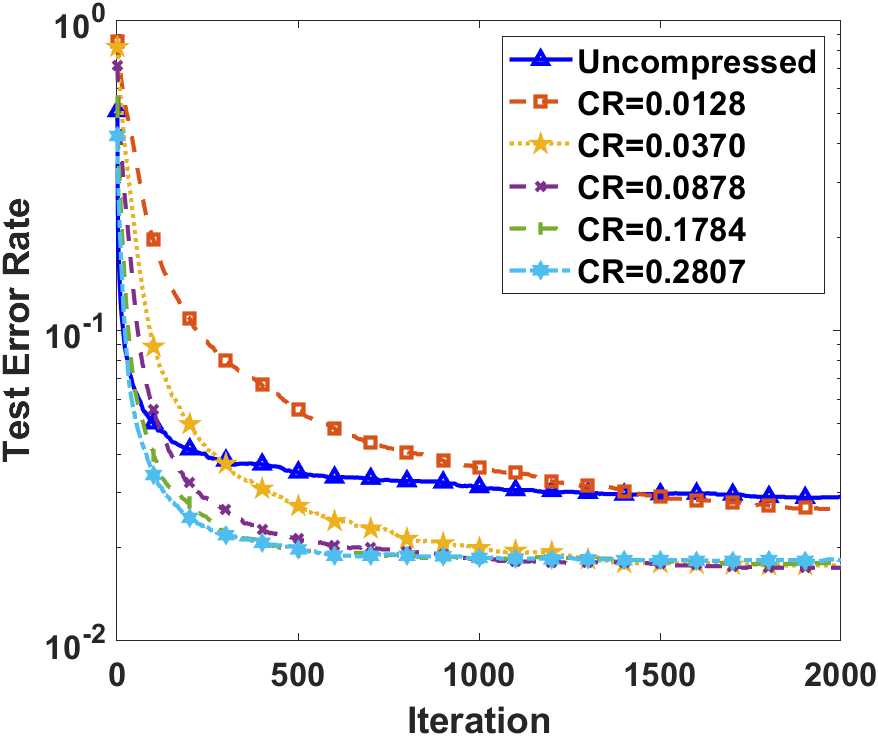} \label{subfig: MnistCNNTestError}}
\caption{The convergence analysis of NN-BCD algorithm with different compression ratios (CNN MNIST): (a) training loss; (b) training error rate; (c) test error rate. The Y-axis is in the log scale.} 
\label{fig: Mnist CNN}
\end{figure} 
The curves of the training loss, training error rate, and test error rate are plotted in Figure~\ref{fig: Mnist CNN}, where the error rate is defined as 1 $-$ accuracy.  Figure~\ref{subfig: MnistCNNTrainingLoss} shows that the training loss of our method converges with different CRs. The training loss also has a monotone decreasing trend, which verified the statements in Theorem~\ref{thm: global convergence}. Figure~\ref{subfig: MnistCNNTrainError} and Figure~\ref{subfig: MnistCNNTestError} show that the training error rate and the test error rate keep decreasing when the number of iterations increases for different CRs. When CR = 1 (the model without compression), the training accuracy keeps increasing to 0.9962 but the test accuracy stops increasing at 0.9621. We can also observe that the training error rate converges faster with a larger CR when CR < 1.

\begin{figure}[!htbp]
\centering
 \subfloat[Effect of Different Hyperparameters]{\includegraphics[width=0.5\textwidth]{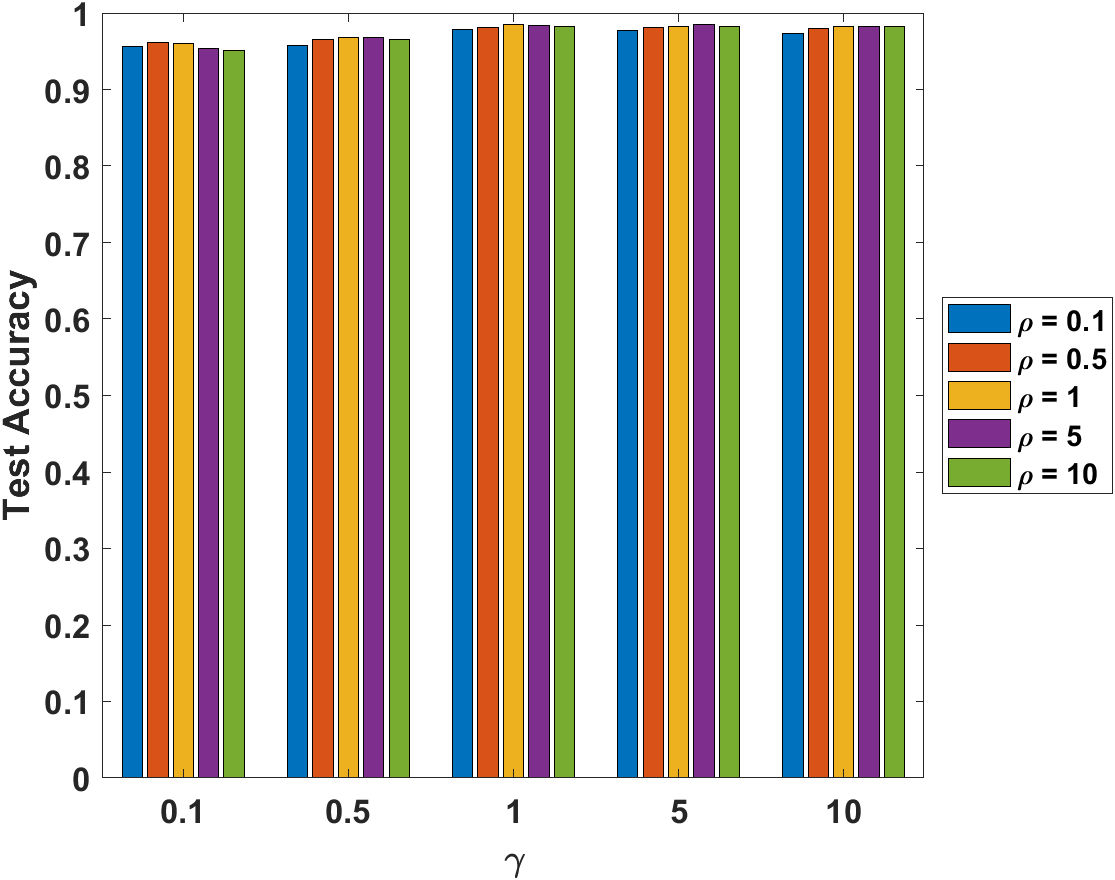} \label{subfig: MnistCNNDiffGammaRho}}
 \subfloat[Stability of Initialization]{\includegraphics[width=0.5\textwidth]{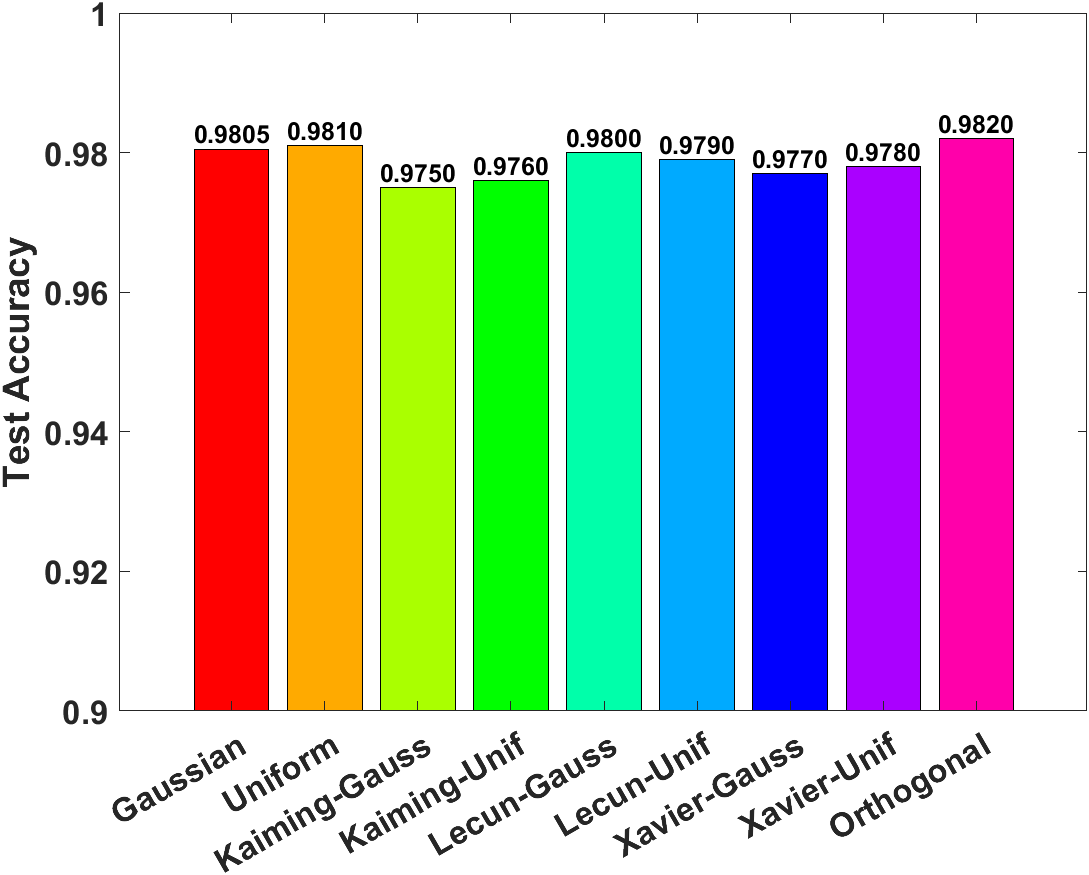} \label{subfig: MnistCNNDiffInit}}
\caption{Effect of different hyperparameters and weight initialization of NN-BCD algorithm (CNN MNIST) when CR = 0.1784: (a) Effect of hyperparameters of NN-BCD; (b) Stability of different weight initialization methods.} 
\label{fig: Mnist CNN stability}
\end{figure}
In our algorithm~\ref{alg: NN-BCD}, hyperparameters $\gamma,\rho$ are critical to the algorithm performance based on our experiments.  To explore the effect of hyperparameters $\gamma,\rho$ on the performance of our method, we test a large number set of different hyperparameter combinations with CR = 0.1784 by fixing $\tau=0.1, \alpha=1$. Figure~\ref{subfig: MnistCNNDiffGammaRho} shows that the CNN consistently performs well for different scales of $\gamma$ and $\rho$.  To further explore our proposed method's stability, we also did experiments with all different weight initialization methods, including Gaussian, Uniform, Kaiming-Gauss, Kaiming-Unif, Lecun-Gauss, Lecun-Unif, Xavier-Gauss, Xavier-Unif, Orthogonal~\citep{boulila2022weight}. Figure~\ref{subfig: MnistCNNDiffInit} shows that all different weight initialization methods perform very well, which also verifies the global convergence in Theorem~\ref{thm: global convergence}. Because the Definition~\ref{def: global convergence} of global convergence states that the algorithm should converge for all kinds of initial solutions.   

 \begin{table}[!htbp]
  \centering
 \caption{Results of Different Decomposition Methods (CNN MNIST).}
 \label{tab: Results of Different Decomposition Methods (CNN Mnist).}
  \begin{tabular}{lrrrrrr}
    \toprule
    &CR  &Training Accuracy &Test Accuracy \\ 
 \midrule
TTD+SGD~\citep{yuan2019high}	    &0.1883	&0.9961 &0.9770\\
Tucker+SGD~\citep{li2020sgd}	&0.1949	&0.9944 &0.9768\\
CP+SGD~\citep{maehara2016expected}	    &0.2191	&0.9978 &0.9800\\
Ours 	&0.1784	&1.0000 &0.9822\\
    \bottomrule
  \end{tabular}
\end{table}
To compare our proposed NN-BCD algorithm with other tensor decomposition-based methods in the literature, we applied the same CNN structure to the tensor train decomposition, Tucker decomposition, and CP decomposition coupled with the SGD algorithm. Table~\ref{tab: Results of Different Decomposition Methods (CNN Mnist).} shows that our method achieves the highest test accuracy as well as training accuracy with the lowest CR.

\subsection{Experiments on Tensorized MLP} \label{subsec: case study HAR}
In this section, we test another dataset called UCI-HAR (Human Activity Recognition) \citep{Reyes2012HAR} to test the effectiveness of our proposed method. The original dataset was made publicly available by a group of researchers from the University of Genova, Italy. The HAR dataset contains 30 volunteers between the age of 19-48 years old. Each person performed six activities (WALKING, WALKING\_UPSTAIRS, WALKING\_DOWNSTAIRS, SITTING, STANDING, LAYING), wearing a smartphone (Samsung Galaxy S II) on their waist, which can capture 3-axial linear acceleration and 3-axial angular velocity. There are a total of 7325 data samples for training and 2946 for testing. After their feature engineering process, there are 561 features for each sample. For this experiment, we are using these features to train our MLP and to predict those six human activities.

\begin{table}[!htbp]
  \centering
 \caption{Results of NN-BCD algorithm with different compression ratios (MLP-4 HAR).}
 \label{tab: HAR4Layer CR ratio}
  \begin{tabular}{lrrrrrr}
    \toprule
    & \multicolumn{2}{c}{Training Loss~\eqref{eq: final formulation}}                     & \multicolumn{2}{c}{Training Accuracy}                       & \multicolumn{2}{c}{Test Accuracy}                           \\ \midrule
 CR    & Mean   & Std   & Mean  & Std   & Mean  & Std \\ 
 \midrule
0.0502	&0.1320	&0.0035	&0.9526	&0.0088	&0.9376	&0.0096\\
0.1043	&0.0557	&0.0015	&0.9823	&0.0011	&0.9559	&0.0029\\
0.1676	&0.0280	&0.0004	&0.9897	&0.0007	&0.9616	&0.0032\\
0.4061	&0.0372	&0.0005	&0.9914	&0.0007	&0.9638	&0.0024\\
0.6310	&0.0242	&0.0002	&0.9944	&0.0006	&0.9657	&0.0029\\
1.0000	&0.0068	&0.0001	&0.9940	&0.0006	&0.9603	&0.0016\\
    \bottomrule
  \end{tabular}
\end{table}
We consider the NN structure that has four hidden layers with the ReLU activation function. The number of neurons in each layer is 561, 1024, 1024, 1024, 512, and 6 (including the input and output layers). Having four hidden layers allows for greater model complexity than NNs with fewer layers, as we have a large number of 561 features. Unlike the fixed hyperparameters in the MNIST dataset, the hyperparameters $\gamma, \rho, \tau, \alpha$  in this experiment are determined by a grid search based on the validation accuracy. \begin{figure}[!htbp]
\centering
	\subfloat[Training loss]{\includegraphics[width=0.33\textwidth]{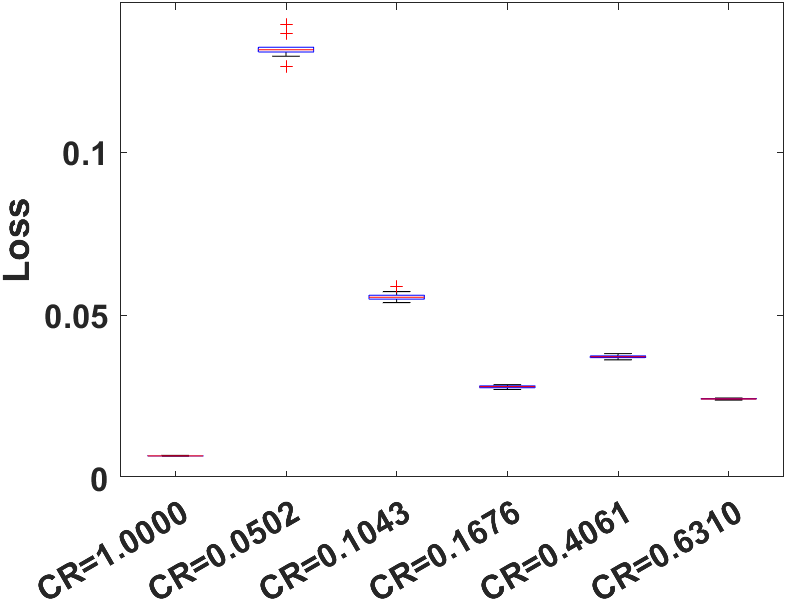} \label{subfig:HAR4Layerbox loss}}
    \subfloat[Training Accuracy]{\includegraphics[width=0.33\textwidth]{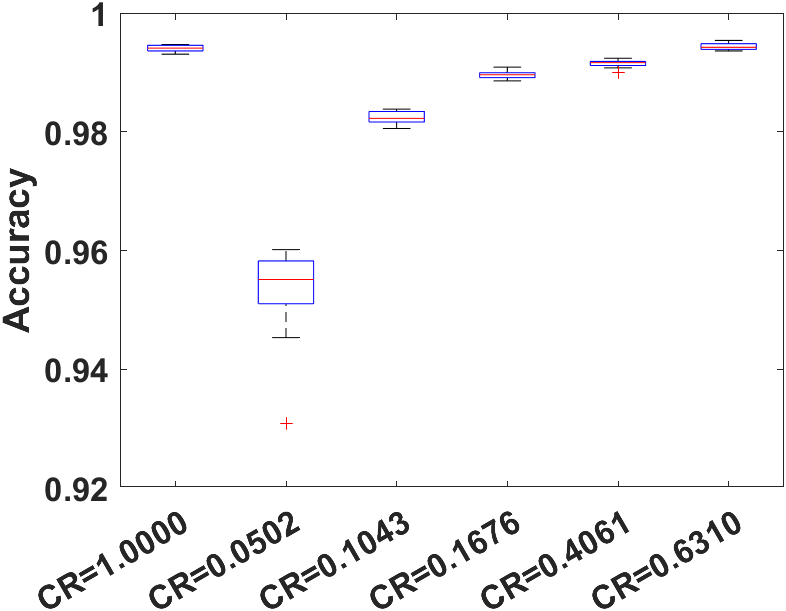} \label{subfig:HAR4Layerbox train acc}}
    \subfloat[Test Accuracy]{\includegraphics[width=0.33\textwidth]{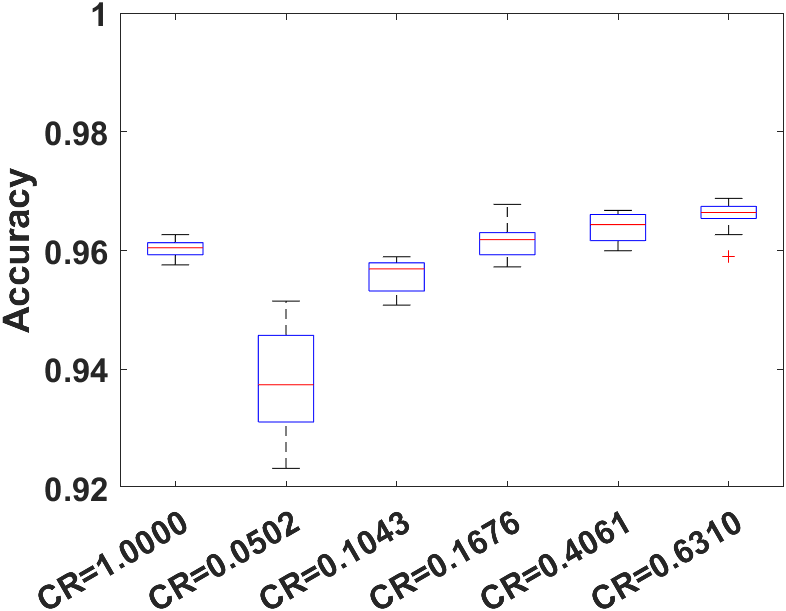} \label{subfig:HAR4Layerbox test acc}}
\caption{The boxplots among ten repetitions with different compression ratios (MLP-4 HAR): (a) training loss; (b) training accuracy; (c) test accuracy.} 
\label{fig: HAR4LayerBox}
\end{figure}  
 Table~\ref{tab: HAR4Layer CR ratio} shows the mean and standard deviation of the training loss~\eqref{eq: final formulation}, training accuracy, and test accuracy. Our model achieves test accuracy of 0.9559 and 0.9638 when CR = 0.1043 and 0.4061, respectively, which are very close to the performance of the uncompressed model (CR = 1). Our method with CR > 0.0502 achieves better classification performance compared with test accuracy = 0.9525 as reported in~\citep{sikder2019human}. Similarly to Table~\ref{tab: Mnist CNN CR ratio}, Table~\ref{tab: HAR4Layer CR ratio} also shows that with a higher compression ratio (with CR < 1), a higher training/test accuracy is observed. However, there is no trend relation between CR and training loss because we are using different hyperparameters for different CRs, which is crucial to calculating the objective function~\eqref{eq: final formulation}. Figure~\ref{fig: HAR4LayerBox} presents the boxplots for the final iteration's training loss, training accuracy, and test accuracy across ten repetitions at each compression ratio. Our method demonstrates an extremely small variation except when CR = 0.0502. 

\begin{figure}[!htbp]
\centering
	\subfloat[Training loss]{\includegraphics[width=0.33\textwidth]{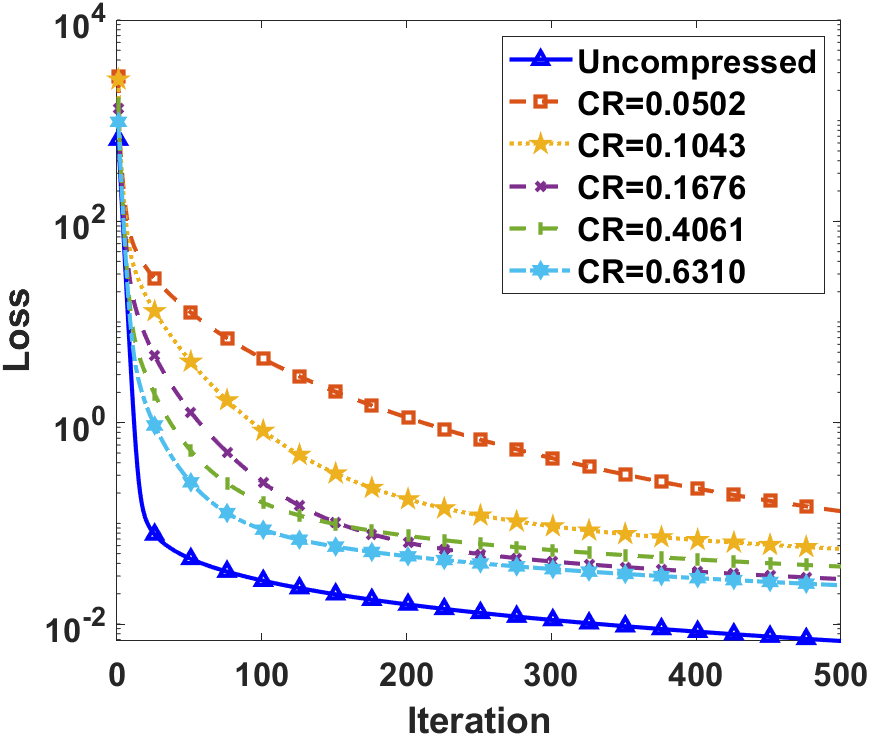} \label{subfig: HAR4Layer training loss}}
     \subfloat[Training Accuracy]{\includegraphics[width=0.33\textwidth]{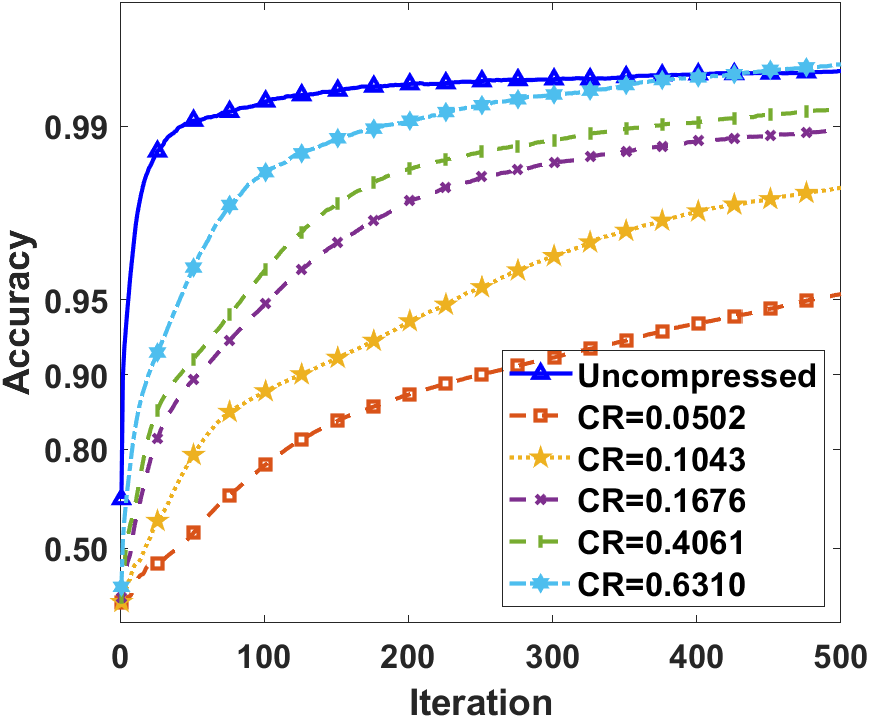} \label{subfig: HAR4Layer training acc}}
    \subfloat[Test Accuracy]{\includegraphics[width=0.33\textwidth]{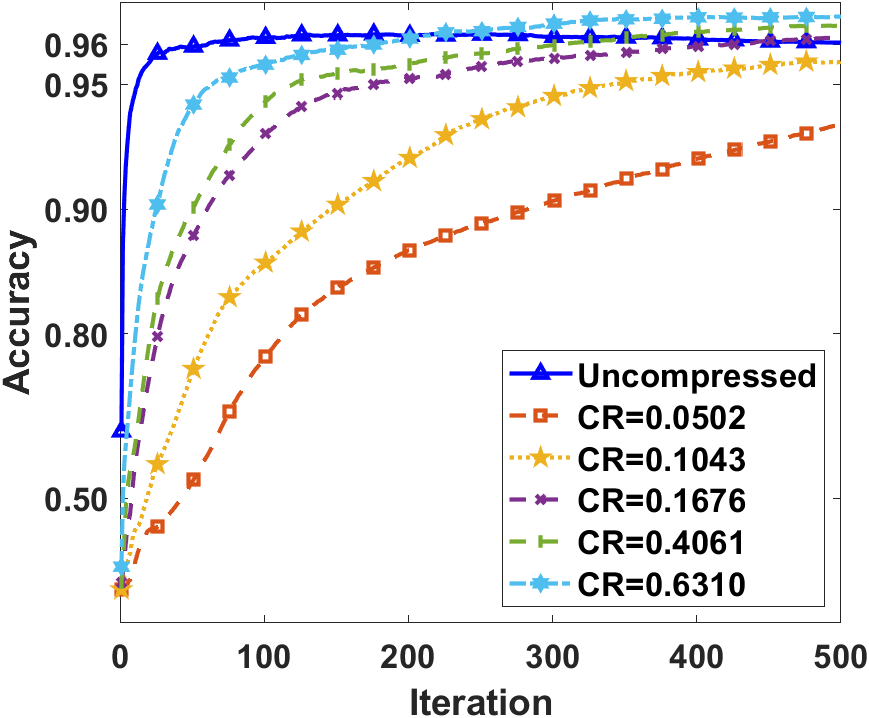} \label{subfig: HAR4Layer test acc}}
\caption{The convergence analysis of NN-BCD algorithm with different compression ratios (MLP-4 HAR): (a) training loss; (b) training accuracy; (c) test accuracy.} 
\label{fig: HAR4Layer convergence}
\end{figure}
 \begin{figure}[!htbp]
\centering
 \subfloat[Effect of Different Hyperparameters]{\includegraphics[width=0.5\textwidth]{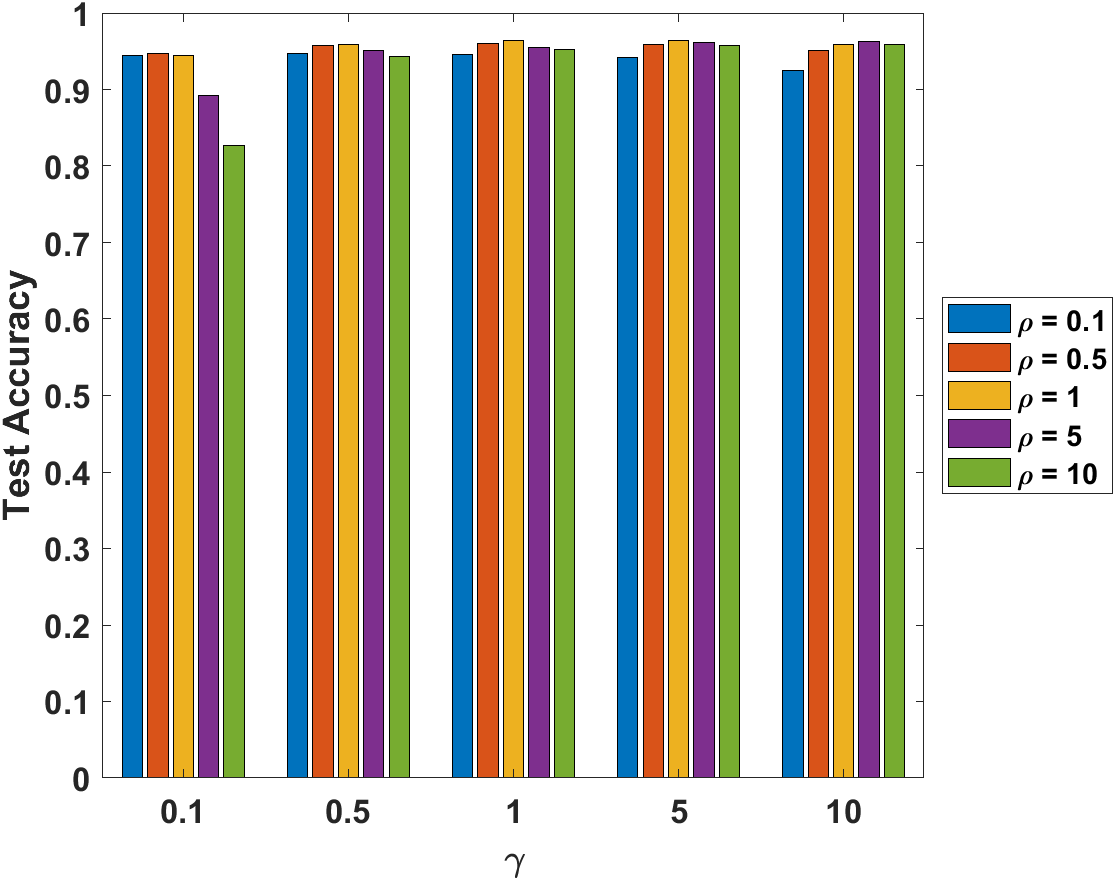} \label{subfig: HAR4LayerDiffGammaRho}}
 \subfloat[Stability of Initialization]{\includegraphics[width=0.5\textwidth]{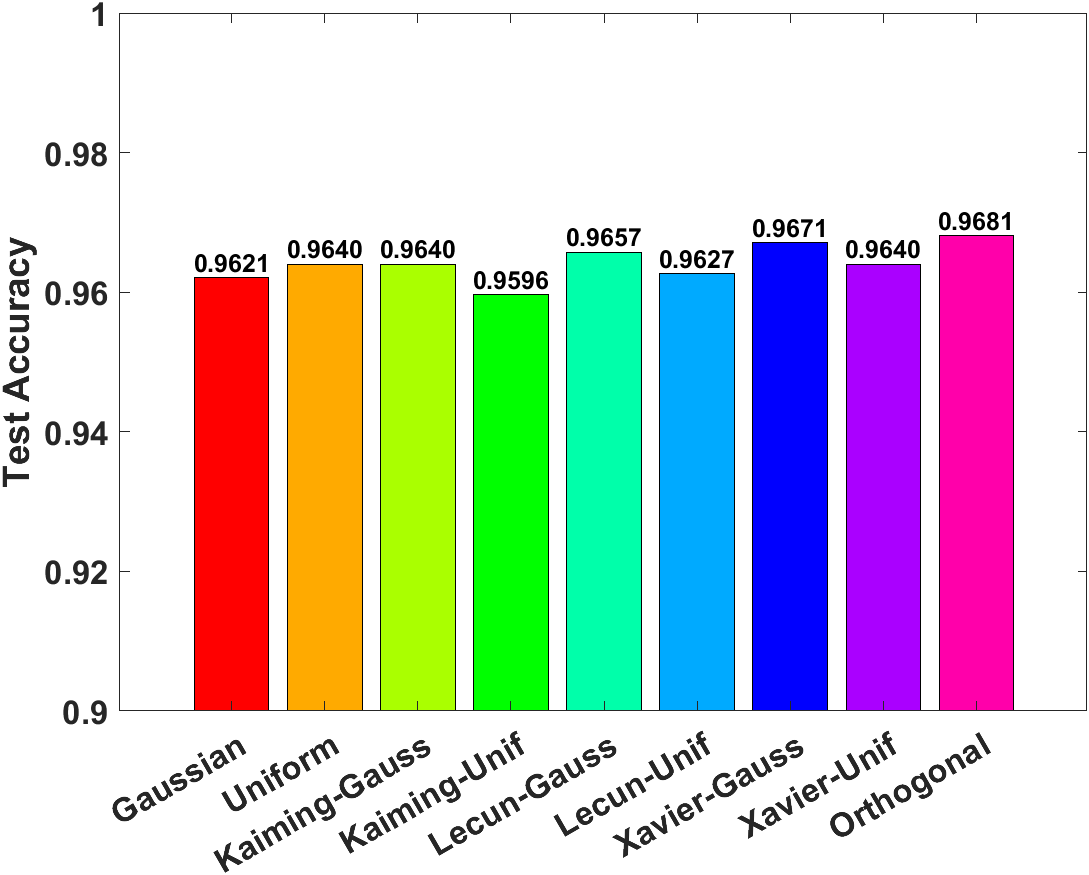} \label{subfig: HAR4LayerDiffInit}}
\caption{Effect of different hyperparameters and weight initialization of NN-BCD algorithm (MLP-4 HAR) when CR = 0.1676 : (a) Effect of hyperparameters of NN-BCD; (b) Stability of different weight initialization methods.} 
\label{fig: HAR4Layer stability}
\end{figure}
In addition, the curves of the training loss, training accuracy, and test accuracy are plotted in Figure~\ref{fig: HAR4Layer convergence}. Figure~\ref{subfig: HAR4Layer training loss} shows the monotone decreasing trend of training loss. Figure~\ref{subfig: HAR4Layer training acc} and Figure~\ref{subfig: HAR4Layer test acc} show that our proposed method outperforms the uncompressed model (CR = 1) in three out of five CRs. Besides, the uncompressed model is also suffering from the over-fitting problem as we can see that the training accuracy keeps increasing and is very close to 1 at the last iteration. However, the test accuracy starts to decrease around 130 iterations in Figure~\ref{subfig: HAR4Layer test acc}. Instead, the test accuracy of our method shows the monotone increasing trend with all CR < 1.

As for the effects of different hyperparameters and the stability of weight initialization, the same experiment setup as the CNN MNIST experiment in Section~\ref{subsec: case study CNN} is applied when CR = 0.1676. Figure~\ref{subfig: HAR4LayerDiffGammaRho} shows that our method is not sensitive to different scales of hyperparameters, which makes the hyperparameters turning process much easier and practicable.  Figure~\ref{subfig: HAR4LayerDiffInit} shows that our method performs equally well across all different weight initialization methods, which demonstrates our method's adaptability with all kinds of weight initialization methods. 

\begin{table}[!htbp]
  \centering
 \caption{Results of Different Decomposition Methods (MLP-4 HAR).}
 \label{tab: Results of Different Decomposition Methods (HAR).}
  \begin{tabular}{lrrrrrr}
    \toprule
    &CR  &Training Accuracy &Test Accuracy \\ 
 \midrule
TTD+SGD~\citep{yuan2019high}   &0.4704	&1.0000 &0.9444\\
Tucker+SGD~\citep{li2020sgd}	&0.4103	&0.9921 &0.9484\\
Our	&0.4061	&0.9914 &0.9638\\
    \bottomrule
  \end{tabular}
\end{table}
To compare our NN-BCD algorithm's performance with other tensor decomposition methods in the literature, we applied the same MLP structure to the tensor train decomposition and Tucker decomposition coupled with the SGD algorithm.  Table~\ref{tab: Results of Different Decomposition Methods (HAR).} shows that the TT+SGD method has 0.9444 test accuracy while the Tucker+SGD method achieves 0.9484 test accuracy. Our method method achieves the highest test accuracy of 0.9638 among all methods.

\subsection{Experiments on Weight Pruning} \label{subsec: case study flare}
In this experiment, the flare classification dataset~\citep{liu2019predicting} is used.  In the solar activities, solar flares are classified by their strength. There are four different classes of solar flares: B-class, C-class, M-class, and X-class, which are ranked from the smallest to the largest. Flares larger than the M-class can much more likely cause potential damage to the astronauts as well as some plants or animals. Thus, in our experiment, we are trying to do a binary classification of predicting the flare that is larger than the M-class (including M-class) or less than the M-class. The input data $\bm{X}$ is called Space-weather HMI Active Region Patches (SHARPs)~\citep{bobra2014helioseismic}. SHARP data contains many physical parameters for flare predictions. In total, there are 40 features, which are the physical parameters of flares.  Flares that occurred between 2010 May and 2018 May are used for our experiment. The training set has 111,050 samples, including 4,057 positive samples (flares larger than or equal to the M-class) and 106,993 negative samples (flares smaller than the M-class). The test set has 44,689 samples, including 1,278 positive samples and 43,411 negative samples. 

We applied the LeNet-300-100 structure~\citep{alford2018pruned}, which is a popular model in the pruning literature. Two hidden layers follow the LeNet-300-100 structure, which has 100 and 300 neurons, respectively. The input layer $d_0=40$ represents 40 different features.  We evaluate our model using balanced accuracy (BACC) as the indicator because the flare dataset is super imbalanced
$$ \text{BACC} = \frac{1}{2}\left(\frac{\text{TP}}{\text{TP} + \text{FN}} + \frac{\text{TN}}{\text{TN} + \text{FP}}\right).$$ There are many more B-class and C-class flares than M-class and X-class flares. Thus, BACC is a good choice to measure our model performance.  According to~\citep{vysogorets2023connectivity}, sparsity can be calculated as
$$\text{sparsity}=\frac{\text{the number of pruned/removed weights}}{\text{the number of weights before pruning}}.$$ A higher value of sparsity indicates that the size of the compressed model is smaller. 

\begin{table}[!htbp]
  \centering
 \caption{Results of NN-BCD algorithm with different sparsity levels (LetNet-300-100 Flare).}
 \label{tab: FlareLeNet Sparsity}
  \begin{tabular}{lrrrrrr}
    \toprule
    & \multicolumn{2}{c}{Training Loss~\eqref{eq: final formulation}}                     & \multicolumn{2}{c}{Training BACC}     & \multicolumn{2}{c}{Test BACC}                           \\ \midrule
Sparsity    & Mean   & Std   & Mean  & Std   & Mean  & Std  \\ 
 \midrule
0.4976	&0.2168	&0.0033	&0.8620	&0.0041	&0.9060	&0.0017\\
0.7464	&0.5683	&0.0089	&0.8561	&0.0032	&0.9071	&0.0017\\
0.8970	&1.8283	&2.3250	&0.8520	&0.0027	&0.8990	&0.0061\\
0.9502	&1.1613	&0.1030	&0.8529	&0.0034	&0.9076	&0.0014\\
0.9932	&26.2695&0.6828	&0.8308	&0.0088	&0.8920	&0.0056\\
0.9977	&0.0326	&0.0050	&0.8157	&0.0121	&0.8890	&0.0106\\
    \bottomrule
  \end{tabular}
\end{table} \begin{figure}[!htbp]
\centering
	\subfloat[Training loss]{\includegraphics[width=0.32\textwidth]{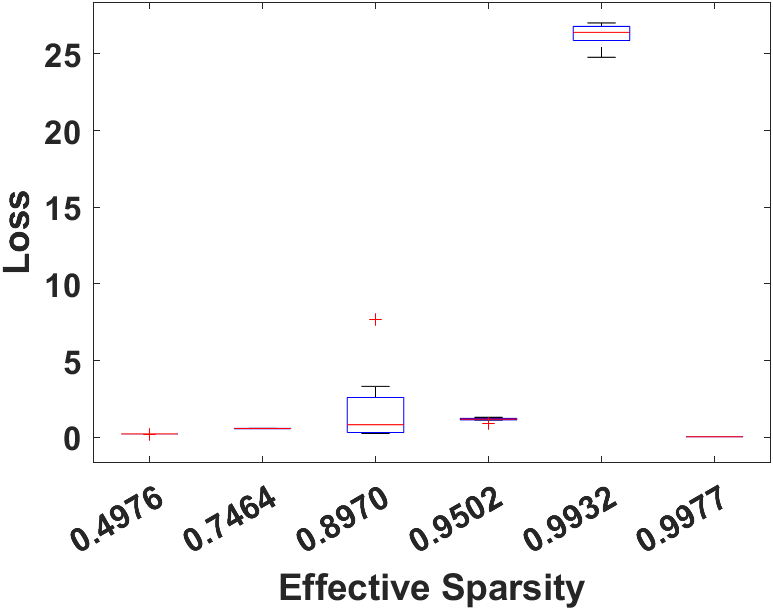} \label{subfig:Flare_sparsity_box_loss}}
    \subfloat[Training Accuracy]{\includegraphics[width=0.33\textwidth]{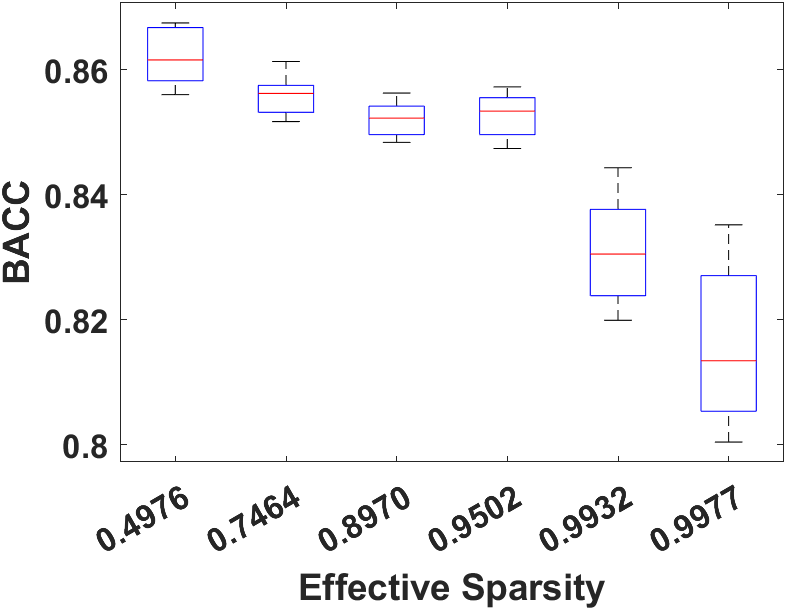} \label{subfig:Flare_sparsity_box_train_bacc}}
    \subfloat[Test Accuracy]{\includegraphics[width=0.33\textwidth]{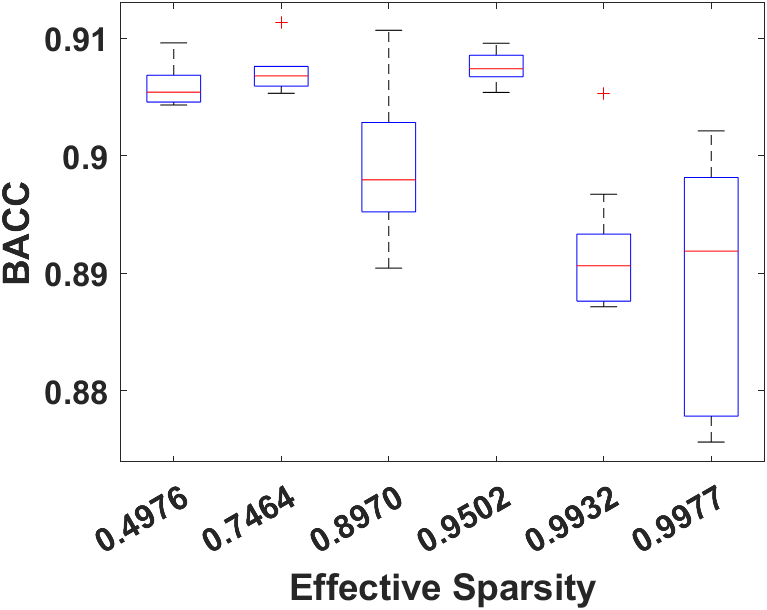} \label{subfig:Flare_sparsity_box_test_bacc}}
\caption{The boxplots among ten repetitions with different sparsity levels (LetNet-300-100 Flare): (a) training loss; (b) training BACC; (c) test BACC.} 
\label{fig: FlareLeNet Box}
\end{figure}
Table~\ref{tab: FlareLeNet Sparsity} shows the result of our NN-BCD algorithm tested on the flare class dataset. With all sparsities, we can achieve BACC around 0.9, which is the best performance reported in~\citep{liu2019predicting}. For example, when we set the sparsity to 0.9977 (\ie, the LeNet-300-100 is compressed $\times$434.7 times), it can have an average BACC of 0.8890. This result shows that our method is not only very effective in reducing the number of parameters in LeNet-300-100 but also maintains a very high BACC. Figure~\ref{fig: FlareLeNet Box} shows the boxplots among ten repetitions. Our model shows an overall steady performance among most of the sparsity levels. All the standard deviations of the training BACC and test BACC are below 0.02. Note that the mean total loss for sparsity = 0.9932 is much higher than the others because we choose different sets of hyperparameters for different sparsity levels. The standard deviation for sparsity = 0.9977 is a little higher than the rest of the sparsity levels. This is due to that the performance of pruned LeNet-300-100 is not very stable with too many weights removed during the training. However, the test BACC falls into the range between 0.875 and 0.90, which is still acceptable. 
\begin{figure}[!htbp]
\centering
	\subfloat[Training loss]{\includegraphics[width=0.32\textwidth]{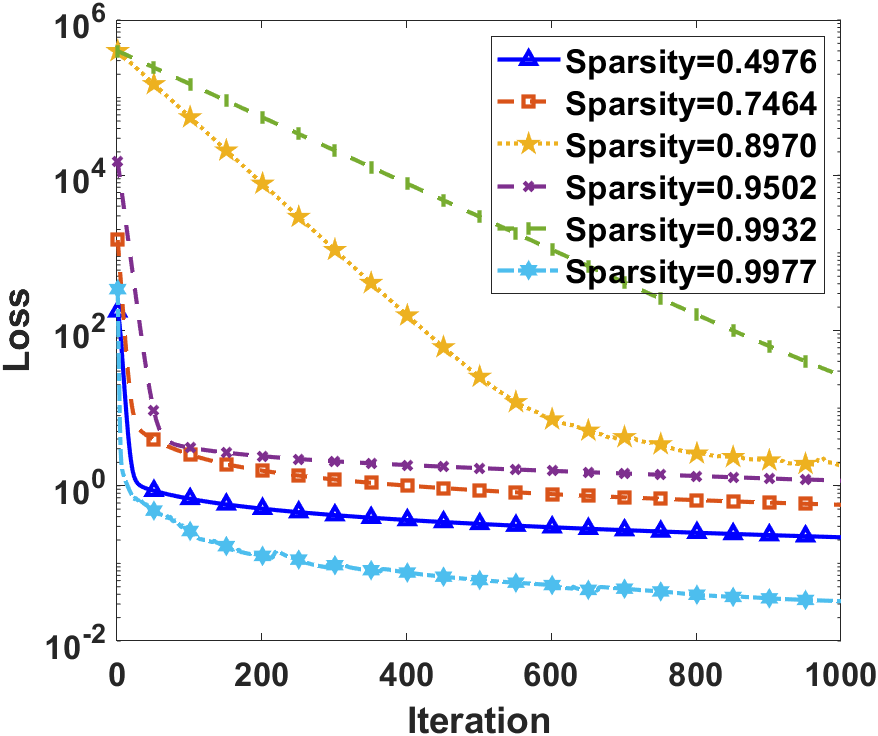} \label{subfig: Flare_Sparsity_trainingloss}}
     \subfloat[Training BACC]{\includegraphics[width=0.32\textwidth]{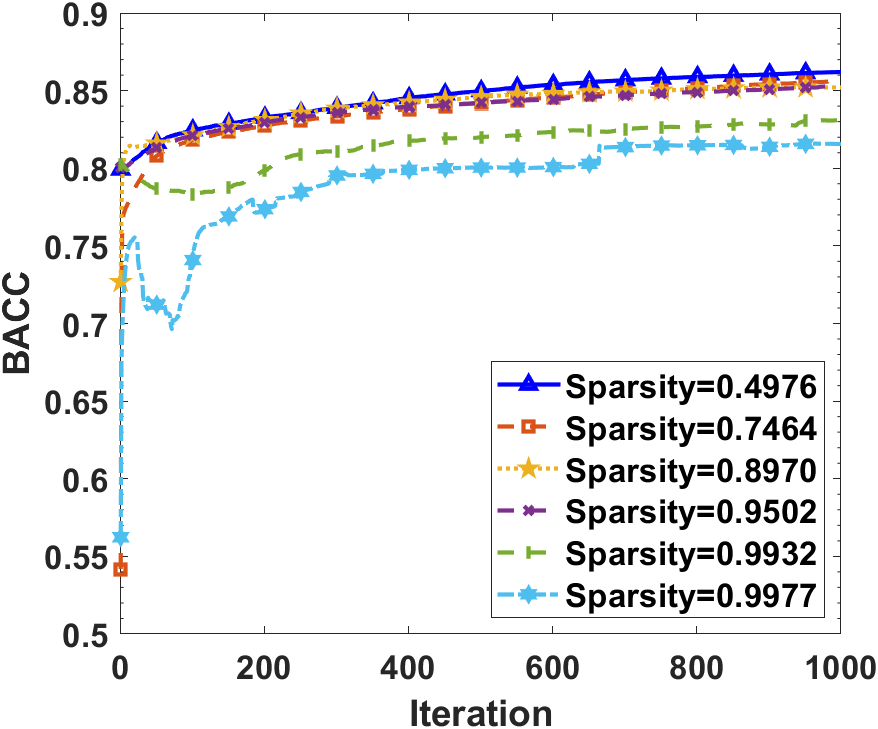} \label{subfig: Flare_Sparsity_trainingBACC}}
    \subfloat[Test BACC]{\includegraphics[width=0.32\textwidth]{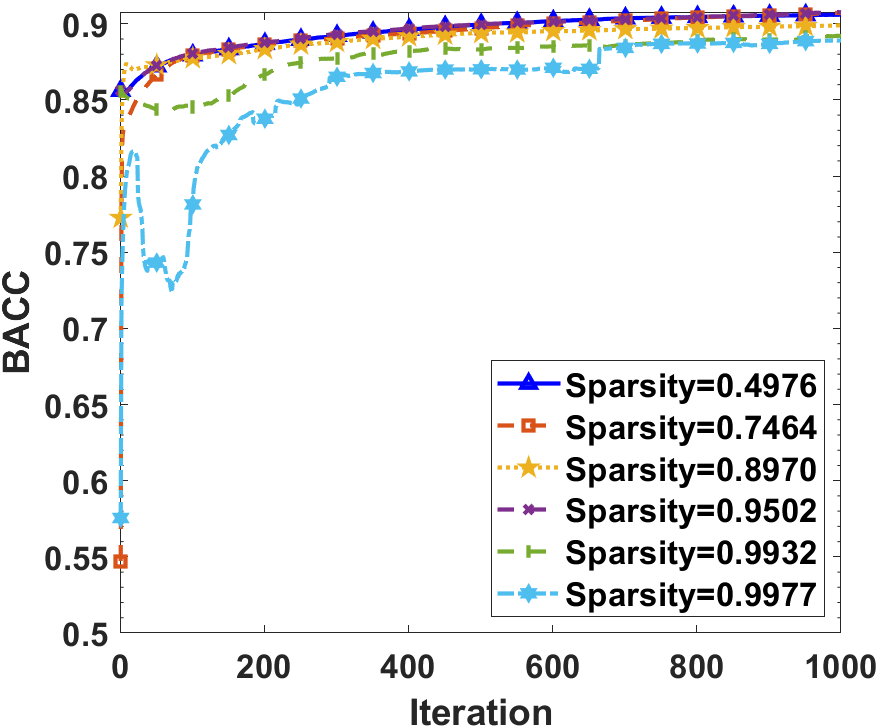} \label{subfig: Flare_Sparsity_testBACC}}
\caption{The convergence analysis of NN-BCD algorithm with different sparsity levels (LetNet-300-100 Flare): (a) training loss; (b) training BACC; (c) test BACC.} 
\label{fig: Flare Spare LeNet convergence}
\end{figure} \begin{figure}[!htbp]
\centering
 \subfloat[Effect of Different Hyperparameters]{\includegraphics[width=0.5\textwidth]{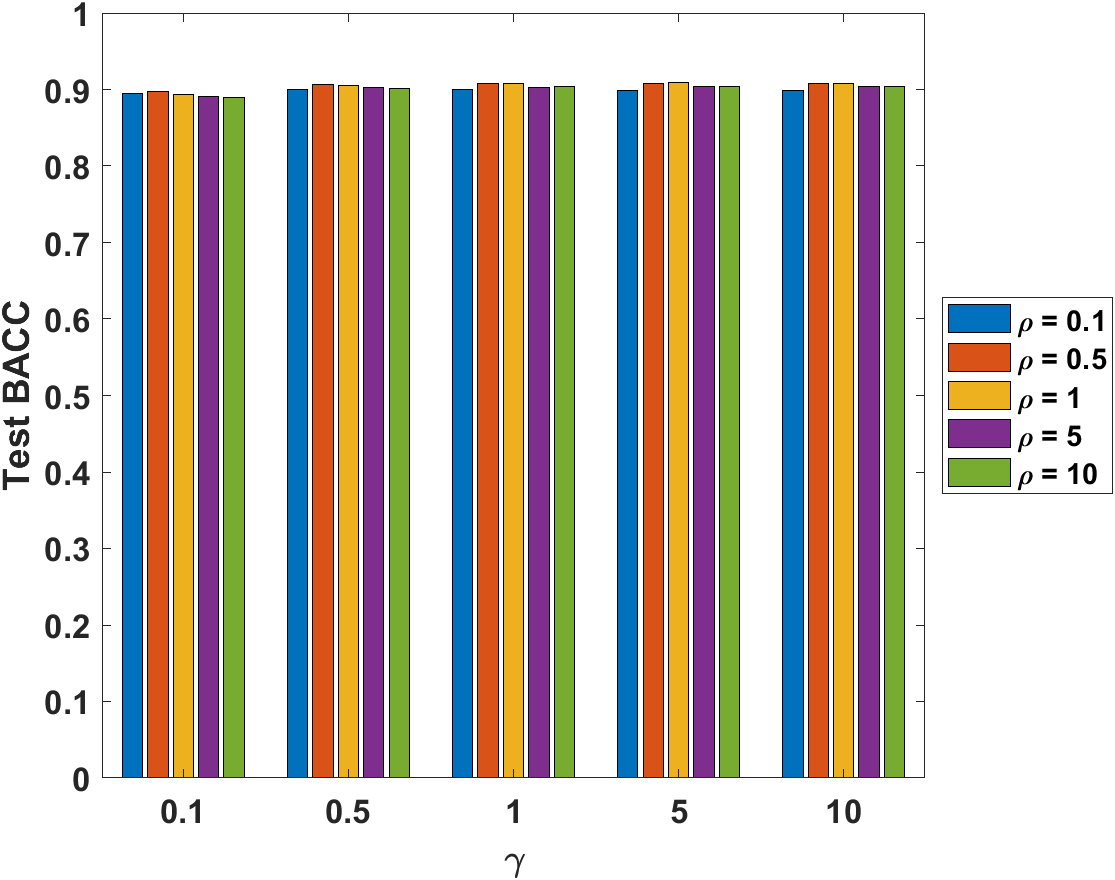} \label{subfig: FlareSparseDiffGammaRho}}
 \subfloat[Stability of Initialization]{\includegraphics[width=0.49\textwidth]{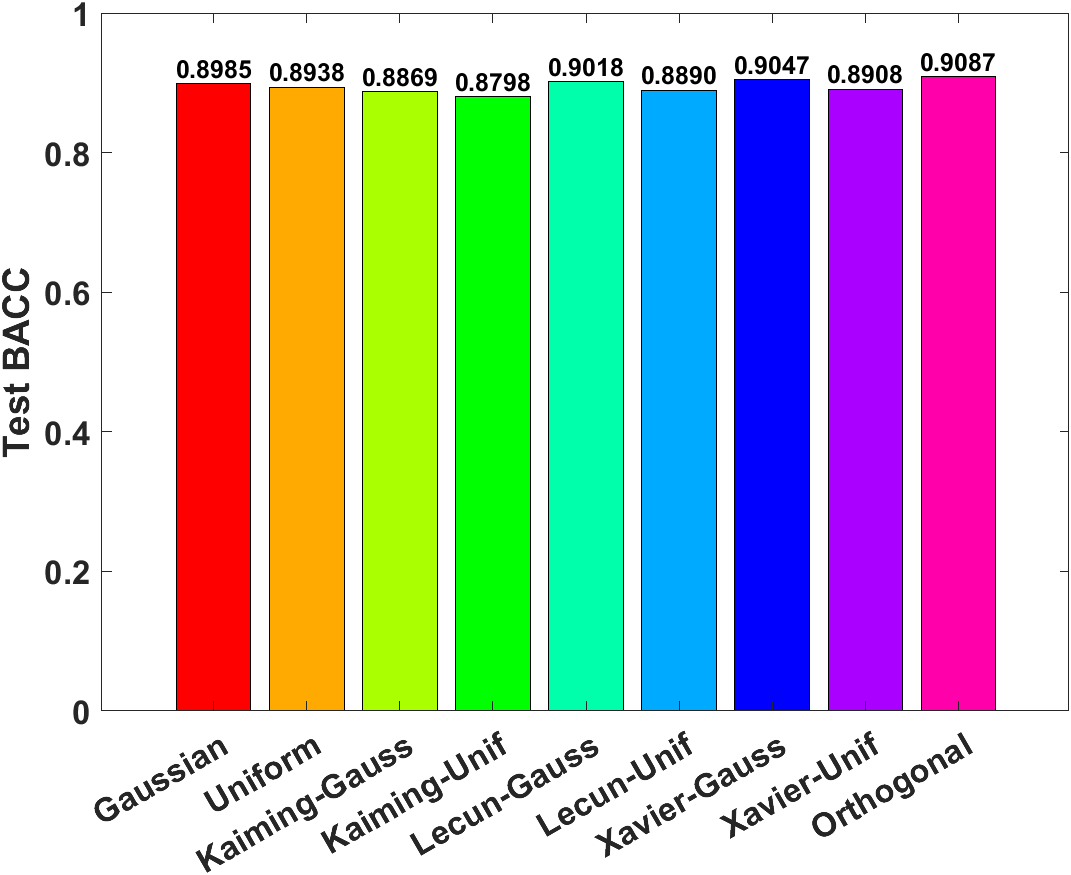} \label{subfig: FlareSparseDiffInit}}
\caption{Effect of different hyperparameters and weight initialization of NN-BCD algorithm (LetNet-300-100 Flare) when sparsity = 0.8970: (a) Effect of hyperparameters of NN-BCD; (b) Stability of different weight initialization methods.} 
\label{fig: Flare stability}
\end{figure}

The convergence of our NN-BCD is shown in Figure~\ref{fig: Flare Spare LeNet convergence}. Figure~\ref{subfig: Flare_Sparsity_trainingloss} shows the training loss with different sparsity levels, where we can observe a clear decreasing trend. The training and test BACC for different sparsity levels are shown in Figure~\ref{subfig: Flare_Sparsity_trainingBACC} and Figure~\ref{subfig: Flare_Sparsity_testBACC}, respectively. Most of the experiments with different sparsity levels show an increasing trend for both training and test BACC as the number of iterations increased. However, when the sparsity level is high, \ie, sparsity = 0.9932 and 0.9977, there are some fluctuations for their training and test BACC before 300 iterations. But after that, the BACC starts to increase steadily. This is due to the constraint $\|\bm{W}_i^{MC}\|_0=\beta_i$ is not smooth, leading to unstable learning in the initial learning.

Figure~\ref{subfig: FlareSparseDiffGammaRho} shows that the test BACCs are all around 0.9 with different choices of $\gamma$ and $\rho$, which demonstrates that our model can maintain good classification performance for different scales of hyperparameters. Figure~\ref{subfig: FlareSparseDiffInit} demonstrates that our model is suitable for many different kinds of weight initialization methods without impacting the classification performance of LeNet-300-100. 

\section{Conclusion} \label{sec: conclusion}
In this paper, we establish a holistic framework for two important model compression techniques: low-rank approximation and weight pruning, which are the two most widely used techniques in the literature. Specifically, the two techniques can be formulated into a single unified nonconvex optimization problem in our framework. Accordingly, the NN-BCD algorithm is proposed to solve the nonconvex optimization problem, where the sublinear convergence rate can be obtained for the NN-BCD algorithm under mild conditions. Extensive experiments on tensor train decomposition-based NN and weight pruning are conducted with different datasets to demonstrate the effectiveness of NN-BCD. The empirical convergence experiment shows that the proposed NN-BCD can converge and run efficiently in practice. Furthermore, NN-BCD can maintain a small compression ratio and high accuracy simultaneously.   

\acks{This work was supported by the National Science Foundation grants AGS-1927578 and AGS-2228996, and CDC/The National Institute for Occupational Safety and Health (NIOSH) under grant number 75D30120P08812.}

\newpage
\appendix
\noindent The appendix is organized as follows.
\begin{itemize}
    \item Appendix~\ref{appendix: convolutional layer}: Kernel $\bm{K}$ Update for Convolutional Layers.
    \item Appendix~\ref{appendix: subproblems}: Solutions of Some Subproblems.
    \item Appendix~\ref{appendix: proof of sufficient decrease}: Proof of Lemma~\ref{lemma: sufficient decrease}.
    \item Appendix~\ref{appendix: proof of subgradient low bound}: Proof of Lemma~\ref{lemma: subgradient low bound}.
    \item Appendix~\ref{appendix: theorem of global convergence}: Proof of Theorem~\ref{thm: global convergence}.
    \item Appendix~\ref{appendix: additional experiments}: Additional Experiments.
\end{itemize}

\begin{figure}[!htbp]
\centering
\includegraphics[width=1\textwidth]{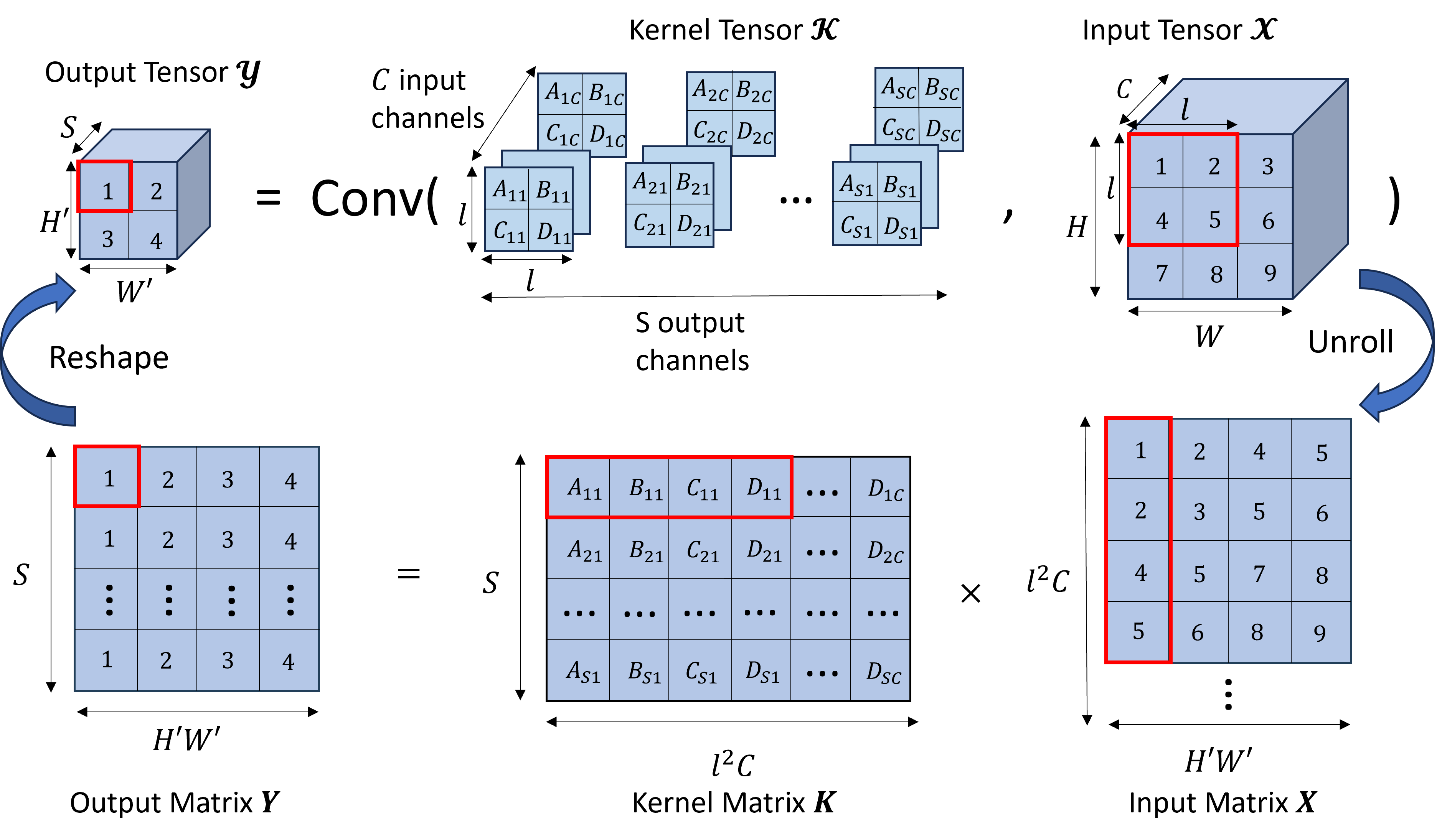} 
\caption{Transformation of convolution to matrix multiplication.} 
    \label{fig: tensor to matrix mutiplication} 
\end{figure}
\section{Kernel \textit{K} Update for Convolutional Layers} \label{appendix: convolutional layer}
CNNs have an advantage over MLPs as they incorporate convolutional layers, enabling them to efficiently capture spatial information and exploit local patterns. These advantages make them highly effective for tasks such as image recognition and computer vision. We want to deploy our NN-BCD algorithm into a CNN structure. The main building block of such a structure is a convolutional layer, that transforms the 3-dimensional input tensor $\bm{\mathcal{X}} \in \mathbb{R}^{W \times H \times C}$ into the output tensor $\bm{\mathcal{Y}} \in \mathbb{R}^{(W-l+1) \times (H-l+1) \times S}$ by convolving $\bm{\mathcal{X}}$ with the kernel tensor $\bm{\mathcal{K}} \in \mathbb{R}^{\ell \times \ell \times C \times S}$:
\begin{equation}\label{eq: CNN}
    \bm{\mathcal{Y}}(x, y, s)=\sum_{i=1}^{\ell} \sum_{j=1}^{\ell} \sum_{c=1}^{C} \bm{\mathcal{K}}(i, j, c, s) \bm{\mathcal{X}}(x+i-1, y+j-1, c).
\end{equation}

To improve the computational performance, many deep learning frameworks reduce the convolution~\eqref{eq: CNN} to a matrix-by-matrix multiplication~\citep{garipov2016ultimate} as shown in Figure~\ref{fig: tensor to matrix mutiplication}.  Note that the compression approach presented in this paper works with other types of convolutions, such as convolutions with padding, stride larger than 1, or rectangular filters. But for clarity, we illustrate the proposed idea on the basic convolution~\eqref{eq: CNN}. 

The notation needed to reformulate convolution~\eqref{eq: CNN} as a matrix-by-matrix multiplication $\bm{Y}=\bm{X} \bm{K}$. For convenience, we denote $H^{\prime}=H-\ell+1$ and $W^{\prime}=W-\ell+1$. The output tensor $\bm{\mathcal{Y}} \in \mathbb{R}^{W^{\prime} \times H^{\prime} \times S}$ is reshaped into a matrix $\bm{Y}$ of size $W^{\prime} H^{\prime} \times S$. The $k$-th row of matrix $\bm{X}\in \mathbb{R}^{W^{\prime} H^{\prime} \times \ell^{2} C}$ corresponds to the $\ell \times \ell \times C$ patch of the input tensor that is used to compute the $k$-th row of the matrix $\bm{Y}$. Finally, we reshape the kernel tensor $\bm{\mathcal{K}}$ into a matrix $\bm{K}$ of size $\ell^{2} C \times S$. Using the matrices defined above, we can rewrite the convolution definition~\eqref{eq: CNN} as $\bm{Y}=\bm{X} \bm{K}$. Now, we are ready to introduce how to update the convolution kernel tensor $\bm{\mathcal{K}}$ through matrix-by-matrix multiplication.

\textbf{Optimization over $\bm{K}_i$:} At iteration $k$, $\bm{K}_i$ can be updated through the following optimization problem 
\begin{equation}\label{eq: K update CNN}
\bm{K}_i^k=\underset{\bm{K}_i}{\operatorname{argmin}} \left\{\frac{\rho}{2}\|\bm{U}_i^k- \bm{V}_{i-1}^{k-1}\bm{K}_i\|_F^2+\frac{\tau}{2}\|\bm{K}_i - \bm{K}_i^{MC}\|_{F}^{2} \right\}, 
\end{equation}
which has a closed-form solution since the above problem~\eqref{eq: K update CNN} is a least square problem. Similarly, the compressed convolution kernel $\bm{K}^{MC}_i$ can also be obtained by following the update used in~\eqref{eq: W_i^MC update}.


\section{Solutions of Some Subproblems}\label{appendix: subproblems}
In this section, we provide the solutions to subproblem~\eqref{eq: V_N update}, the ReLU-involved subproblem~\eqref{eq: U_i update}, the compressed weight update subproblem~\eqref{eq: W_i^MC update}.
\subsection{Solutions to Subproblem~\texorpdfstring{\eqref{eq: V_N update}}{ } }\label{appendix: subproblems V_N}
\textbf{Prox-linear algorithm to subproblem~\eqref{eq: V_N update}}: in the $\bm{V}_{N}$-update of Algorithm~\ref{alg: NN-BCD}, the empirical risk is involved in the optimization problems. It is generally hard to obtain its closed-form solution except for some special cases such as the case where the loss is the square loss. For other smooth losses such as the logistic, cross-entropy, and exponential losses, we suggest using the following prox-linear update strategies, that is, for some parameter $\alpha>0$, the $\bm{V}_{N}$-update in  Algorithm~\ref{alg: NN-BCD} is
\begin{equation}\label{eq: V_N proximal-linear}
\bm{V}_{N}^{k}=\underset{\bm{V}_{N}}{\operatorname{argmin}}\{s_{N}(\bm{V}_{N})+\langle\nabla \mathcal{R}_{n}(\bm{V}_{N}^{k-1} ; \bm{Y}), \bm{V}_{N}-\bm{V}_{N}^{k-1}\rangle +\frac{\gamma}{2}\|\bm{V}_{N}-\bm{U}_{N}^{k-1}\|_{F}^{2}+\frac{\alpha}{2}\|\bm{V}_{N}-\bm{V}_{N}^{k-1}\|_{F}^{2}\},
\end{equation}
This $\bm{V}_{N}$-update can be implemented with explicit expressions. Therefore, the specific uses of these NN-BCD methods are very flexible, mainly depending on users' understanding of their own problems.

\textbf{The closed-form of the proximal operator of hinge loss:} consider the following optimization problem
\begin{equation}\label{eq: hinge loss}
    u^{*}=\underset{u}{\operatorname{argmin}} g(u):=\max \{0,1-a \cdot u\}+\frac{\gamma}{2}(u-b)^{2},
\end{equation}
where $\gamma>0$.
\begin{lemma}
    The optimal solution to Problem~\eqref{eq: hinge loss} is shown as follows
    \begin{equation*}
\textnormal{hinge}_{\gamma}(a, b)= \begin{cases}b, & \text{ if } a=0, \\ b+\gamma^{-1} a, & \text{ if } a \neq 0 \text{ and } a b \leq 1-\gamma^{-1} a^{2}, \\ a^{-1}, & \text{ if } a \neq 0 \text{ and } 1-\gamma^{-1} a^{2}<a b<1, \\ b, & \text{ if } a \neq 0 \text{ and } a b \geq 1 .\end{cases}
    \end{equation*}
\end{lemma}

\subsection{The Closed-form Solution to Subproblem~\texorpdfstring{\eqref{eq: U_i update}}{} }\label{appendix: subproblems Relu}
From Algorithm~\ref{alg: NN-BCD}, when $\sigma_{i}$ is ReLU, then the $\bm{U}_{i}^{k}$-update actually reduces to the following one-dimensional minimization problem
\begin{equation}\label{eq: Relu update}
    u^{*}=\underset{u}{\operatorname{argmin}} f(u):=\frac{1}{2}(\sigma(u)-a)^{2}+\frac{\gamma}{2}(u-b)^{2},
\end{equation}
where $\sigma(u)=\max \{0, u\}$ and $\gamma>0$. The solution to the above one-dimensional minimization problem can be presented in the following lemma.
\begin{lemma}
The optimal solution to Problem~\eqref{eq: Relu update} is shown as follows
\begin{equation}
    \operatorname{prox}_{\frac{1}{2 \gamma}(\sigma(\cdot)-a)^{2}}(b)=\left\{
     \begin{array}{ll}
   \frac{a+\gamma b}{1+\gamma}, & \text{ if } a+\gamma b \geq 0, b \geq 0, \\
\frac{a+\gamma b}{1+\gamma}, & \text{ if }-(\sqrt{\gamma(\gamma+1)}-\gamma) a \leq \gamma b<0, \\b, & \text{ if }-a \leq \gamma b \leq-(\sqrt{\gamma(\gamma+1)}-\gamma) a<0, \\
\min \{b, 0\}, & \text{ if } a+\gamma b<0.
 \end{array}\right. 
\end{equation}
\end{lemma}

\subsection{The Closed-form Solution to Subproblem~\texorpdfstring{\eqref{eq: W_i^MC update}}{} }\label{appendix: subproblems MC}
The following cases are considered in our paper

1) Lasso regularization:
\begin{equation}\label{eq: W_i^MC lasso}
  \bm{W}_i^{MC,k} = \underset{\bm{W}_i^{MC}}{\operatorname{argmin}} \left\{\|\bm{W}_i^{MC}\|_1+\frac{\tau}{2}\|\bm{W}_i^k - \bm{W}_i^{MC}\|_{F}^{2} +\frac{\alpha}{2}\|\bm{W}_i^{MC}-\bm{W}_i^{MC,k-1}\|_F^2 \right\}.
\end{equation}
The above optimization problem~\eqref{eq: W_i^MC lasso} can be solved in closed-form with soft thresholding~\citep{donoho1995noising}.

 2) $\ell_0$ norm regularization:
\begin{equation}\label{eq: W_i^MC l0 regularization}
  \bm{W}_i^{MC,k} = \underset{\bm{W}_i^{MC}}{\operatorname{argmin}} \left\{\|\bm{W}_i^{MC}\|_0+\frac{\tau}{2}\|\bm{W}_i^k - \bm{W}_i^{MC}\|_{F}^{2} +\frac{\alpha}{2}\|\bm{W}_i^{MC}-\bm{W}_i^{MC,k-1}\|_F^2 \right\}.
\end{equation}
The above optimization problem~\eqref{eq: W_i^MC l0 regularization} can be solved in closed-form with hard thresholding~\citep{blumensath2009iterative}.

 3) $\ell_0$ norm constrained:
\begin{equation}\label{eq: W_i^MC l0 constrained}
      \bm{W}_i^{MC,k}  = \underset{\|\bm{W}_i^{MC}\|_0 = \beta_i}{\operatorname{argmin}} \left\{\frac{\tau}{2}\|\bm{W}_i^k - \bm{W}_i^{MC}\|_{F}^{2} +\frac{\alpha}{2}\|\bm{W}_i^{MC}-\bm{W}_i^{MC,k-1}\|_F^2 \right\}.
\end{equation}
The above optimization problem~\eqref{eq: W_i^MC l0 constrained} can be reformulated into
\begin{equation*}\label{eq: W_i^MC l0 constrained-1}
      \bm{W}_i^{MC,k}  = \underset{\|\bm{W}_i^{MC}\|_0 = \beta_i}{\operatorname{argmin}} \left\{ \|\bm{W}_i^{MC} - \big(\frac{\tau}{\tau+\alpha} \bm{W}_i^{MC}+\frac{\alpha}{\tau+\alpha}\bm{W}_i^{MC,k-1}\big)\|_{F}^{2}  \right\},
\end{equation*}
where the solution can be obtained by selecting the $\beta_i$-largest elements from $\big|\frac{\tau}{\tau+\alpha} \bm{W}_i^{MC}+\frac{\alpha}{\tau+\alpha}\bm{W}_i^{MC,k-1}\big|$.

4) tensor-train decomposition: define $\mathcal{G}_{i}$  as the set of TT-cores from $i$-th layer.
\begin{equation}\label{eq: G_i update}
    \mathcal{G}_i^k = \underset{\mathcal{G}_i}{\operatorname{argmin}} \left\{\frac{\tau}{2}\|\bm{\mathcal{W}}_i^k - \textnormal{TTD}(\bm{r}_i)\|_{F}^{2} +\frac{\alpha}{2}\|\mathcal{G}_i-\mathcal{G}_i^{k-1}\|_F^2 \right\}
\end{equation}  where $\frac{\alpha}{2}\|\mathcal{G}_i-\mathcal{G}_i^{k-1}\|_F^2$ is the proximal terms. This subproblem is implemented in TensorLy package~\citep{tensorly}. If one wants to use Tucker decomposition with our NN-BCD, the closed-form update proposed by our previous work~\citep{shen2022smooth} can be used. For CP decomposition, we can refer to~\citep {shen2022super}.

\section{Proof of Lemma~\ref{lemma: sufficient decrease}} \label{appendix: proof of sufficient decrease}
Let $h: \mathbb{R}^{p} \rightarrow \mathbb{R} \cup\{+\infty\}$ be an extended-real-valued function, its graph is defined by
${Graph}(h) :=\{(\bm{x}, y) \in \mathbb{R}^{p} \times \mathbb{R}: y=h(\bm{x})\},$ and its domain by $\operatorname{dom}(h):=\{\bm{x} \in \mathbb{R}^{p}: h(\bm{x})<+\infty\}$. The subdifferential of a function is defined as follows. 
\begin{definition}[\normalfont Subdifferentials~\citep{attouch2009convergence, attouch2010proximal}]\label{def: Subdifferentials}~Assume that $f:~\mathbb{R}^p \to (-\infty,+\infty)$ is a proper and lower semicontinuous function. 
\begin{enumerate}[topsep=0pt,itemsep=0pt,partopsep=0pt,parsep=0pt]  
	\item The domain of $f$ is defined and denoted by $\textnormal{dom}f\coloneqq \{\bm{x}\in\mathbb{R}^p:f(\bm{x})<+\infty\}$
	\item For a given $\bm{x}\in \textnormal{dom}f$, the Fréchet subdifferential of $f$ at $\bm{x}$, written $\hat{\partial}f(\bm{x})$, is the set of all vectors $\bm{u}\in  \mathbb{R}^p$ that satisfy
\[\lim\limits_{\bm{y} \neq \bm{x}} \inf_{\bm{y} \to \bm{x}} \frac{f(\bm{y})-f(\bm{x}) - \langle \bm{u} , \bm{y}-\bm{x}\rangle}{\|\bm{y}-\bm{x}\|}\geq 0.\]
\item The limiting-subdifferential, or simply the subdifferential, of $f$ at $\bm{x}$, written $\partial f(\bm{x})$ is defined through the following closure process
\[\partial f(\bm{x}):=\{\bm{u}\in  \mathbb{R}^p: \exists \bm{x}^k \to \bm{x},f(\bm{x}^k) \to f(\bm{x}) \textnormal{ and } \bm{u}^k \in \hat{\partial}f(\bm{x}^k) \to \bm{u} \textnormal{ as } k\to \infty \}.\]
	\end{enumerate}
\end{definition}
The detailed proof for Lemma~\ref{lemma: sufficient decrease} is provided below.
\begin{proof}
The inequality~\eqref{eq: sufficient decrease} can be developed by considering the descent quantity along the update of each block variable, \ie, $\{\bm{V}_i\}_{i=1}^N$, $\{\bm{U}_i\}_{i=1}^N$, $\{\bm{W}_i\}_{i=1}^N$, and $\{\bm{W}_i^{MC}\}_{i=1}^N$. To begin with, the following notations are introduced. Specifically, 
$\bm{W}_{<i}:=\left(\bm{W}_{1}, \bm{W}_{2}, \ldots, \bm{W}_{i-1}\right)$, $\bm{W}_{>i}:=\left(\bm{W}_{i+1}, \bm{W}_{i+1}, \ldots, \bm{W}_{N}\right)$, and $\bm{V}_{<i}, \bm{V}_{>i}, \bm{U}_{<i}, \bm{U}_{>i},\bm{W}^{MC}_{<i},\bm{W}^{MC}_{>i}$ are defined similarly. We will consider each case separately. 

\subsubsection*{Optimization over $\bm{V}_i$}
\textbf{$\bm{V}_{N}^{k}$-block:} at iteration $k$,  there are two ways to update the variable: 

1) proximal update with closed-form solution: the following inequality can be derived
\begin{equation}\label{eq: V_N proximal update}
  \begin{aligned}
      & \mathcal{L}\left(\{\bm{W}_{i}^{k-1}\}_{i=1}^{N}, \bm{V}_{i<N}^{k-1}, \textcolor{blue}{\bm{V}_{N}^{k}},\{\bm{U}_{i}^{k-1}\}_{i=1}^{N}, \{\bm{W}_{i}^{MC,k-1}\}_{i=1}^{N}\right)\\
        \leq &  \mathcal{L}\left(\{\bm{W}_{i}^{k-1}\}_{i=1}^{N}, \bm{V}_{i<N}^{k-1}, \textcolor{blue}{\bm{V}_{N}^{k-1}},\{\bm{U}_{i}^{k-1}\}_{i=1}^{N}, \{\bm{W}_{i}^{MC,k-1}\}_{i=1}^{N}\right)-\textcolor{blue}{\frac{\alpha}{2}\|\bm{V}_{N}^{k}-\bm{V}_{N}^{k-1}\|_{F}^{2}}.
  \end{aligned}
\end{equation} 
The above inequality~\eqref{eq: V_N proximal update} is due to the fact that $\bm{V}_{N}^{k}$ is the optimal solution for subproblem~\eqref{eq: V_N update}. 

2) proximal-linear case: let $h^{k}(\bm{V}_{N}):=s_{N}(\bm{V}_{N})+\mathcal{R}_{n}(\bm{V}_{N} ; \bm{Y})+\frac{\gamma}{2}\|\bm{V}_{N}-\bm{U}_{N}^{k-1}\|_{F}^{2}$ and $\bar{h}^{k}(\bm{V}_{N}):=s_{N}(\bm{V}_{N})+\mathcal{R}_{n}(\bm{V}_{N}^{k-1} ; \bm{Y})+\langle\nabla \mathcal{R}_{n}(\bm{V}_{N}^{k-1} ; \bm{Y}), \bm{V}_{N}-\bm{V}_{N}^{k-1}\rangle+\frac{\alpha}{2}\|\bm{V}_{N}-\bm{V}_{N}^{k-1}\|_{F}^{2} +\frac{\gamma}{2}\|\bm{V}_{N}-\bm{U}_{N}^{k-1}\|_{F}^{2}$. By the optimality of $\bm{V}_N^k$ and the strong convexity\footnote{The function $h$ is called a strongly convex function with parameter $\gamma>0$ if $h(u) \geq h(v)+\langle\nabla h(v), u-v\rangle+\frac{\gamma}{2}\|u-v\|^{2}$.} of $\bar{h}^{k}(\bm{V}_{N})$ with modulus at least $\alpha +\gamma$, the following holds
\begin{equation}\label{eq: V_N proximal-linear update}
     \bar{h}^{k}(\bm{V}_{N}^{k}) \leq \bar{h}^{k}(\bm{V}_{N}^{k-1}) -\frac{\alpha+\gamma}{2}\|\bm{V}_{N}^{k}-\bm{V}_{N}^{k-1}\|_{F}^{2},
\end{equation}
which implies 
\begin{subequations}\label{eq: V_N proximal-linear update1}
\begin{align}
h^{k}(\bm{V}_{N}^{k}) &\leq  h^{k}(\bm{V}_{N}^{k-1}) + \mathcal{R}_{n}(\bm{V}_{N}^{k} ; \bm{Y})-\mathcal{R}_{n}(\bm{V}_{N}^{k-1} ; \bm{Y})-\langle\nabla \mathcal{R}_{n}(\bm{V}_{N}^{k-1} ; \bm{Y}), \bm{V}_{N}^{k}-\bm{V}_{N}^{k-1}\rangle \nonumber\\ 
&-(\alpha+\frac{\gamma}{2})\|\bm{V}_{N}^{k}-\bm{V}_{N}^{k-1}\|_{F}^{2} \label{eq: V_N proximal-linear update1-1}  \\ 
&\leq h^{k}(\bm{V}_{N}^{k-1})-(\alpha+\frac{\gamma-L_{R}}{2})\|\bm{V}_{N}^{k}-\bm{V}_{N}^{k-1}\|_{F}^{2}, \label{eq: V_N proximal-linear update1-2} 
\end{align}
\end{subequations}
where inequality~\eqref{eq: V_N proximal-linear update1-1} is due to the inequality~\eqref{eq: V_N proximal-linear update},  the relationship between $h^{k}(\bm{V}_{N}^{k-1})$ and $\bar{h}^{k}(\bm{V}_{N}^{k-1})$,  and the relationship between $h^{k}(\bm{V}_{N}^{k})$ and $\bar{h}^{k}(\bm{V}_{N}^{k})$. The inequality~\eqref{eq: V_N proximal-linear update1-2} holds for the $L_{R}$-Lipschitz continuity of $\nabla \mathcal{R}_{n}$, \ie, the following inequality by~\citep{boyd2004convex}
\begin{equation*}
    \mathcal{R}_{n}(\bm{V}_{N}^{k} ; \bm{Y}) \leq \mathcal{R}_{n}(\bm{V}_{N}^{k-1} ; \bm{Y})+\langle\nabla \mathcal{R}_{n}(\bm{V}_{N}^{k-1} ; \bm{Y}), \bm{V}_{N}^{k}-\bm{V}_{N}^{k-1}\rangle+\frac{L_{R}}{2}\|\bm{V}_{N}^{k}-\bm{V}_{N}^{k-1}\|_{F}^{2}.
\end{equation*}
According to  the relationship between  $\mathcal{L}\left(\{\bm{W}_{i}^{k-1}\}_{i=1}^{N}, \bm{V}_{i<N}^{k-1}, \bm{V}_{N},\{\bm{U}_{i}^{k-1}\}_{i=1}^{N}, \{\bm{W}_{i}^{MC,k-1}\}_{i=1}^{N}\right)$ and $h^{k}(\bm{V}_{N})$, and the inequality~\eqref{eq: V_N proximal-linear update1},
\begin{equation}\label{eq: V_N proximal-linear update2}
  \begin{aligned}
      & \mathcal{L}\left(\{\bm{W}_{i}^{k-1}\}_{i=1}^{N}, \bm{V}_{i<N}^{k-1}, \textcolor{blue}{\bm{V}_{N}^{k}},\{\bm{U}_{i}^{k-1}\}_{i=1}^{N}, \{\bm{W}_{i}^{MC,k-1}\}_{i=1}^{N}\right)\\
        \leq &  \mathcal{L}\left(\{\bm{W}_{i}^{k-1}\}_{i=1}^{N}, \bm{V}_{i<N}^{k-1}, \textcolor{blue}{\bm{V}_{N}^{k-1}},\{\bm{U}_{i}^{k-1}\}_{i=1}^{N}, \{\bm{W}_{i}^{MC,k-1}\}_{i=1}^{N}\right) \\
        - & \textcolor{blue}{(\alpha+\frac{\gamma-L_{R}}{2})\|\bm{V}_{N}^{k}-\bm{V}_{N}^{k-1}\|_{F}^{2}}.
  \end{aligned}
\end{equation} 

\noindent \textbf{$\bm{V}_{i}^{k}$-block ($i<N$):} $\bm{V}_{i}^{k}$ is updated according to the following
\begin{equation*}
    \bm{V}_{i}^{k} \leftarrow \underset{\bm{V}_{i}}{\operatorname{argmin}}\left\{s_{i}(\bm{V}_{i})+\frac{\gamma}{2}\|\bm{V}_{i}-\sigma_{i}(\bm{U}_{i}^{k-1})\|_{F}^{2}+\frac{\rho}{2}\|\bm{U}_{i+1}^{k}-\bm{W}_{i+1}^{k} \bm{V}_{i}\|_{F}^{2}\right\}.
\end{equation*}
Let $h^{k}(\bm{V}_{i})=s_{i}(\bm{V}_{i})+\frac{\gamma}{2}\|\bm{V}_{i}-\sigma_{i}(\bm{U}_{i}^{k-1})\|_{F}^{2}+\frac{\rho}{2}\|\bm{U}_{i+1}^{k}-\bm{W}_{i+1}^{k} \bm{V}_{i}\|_{F}^{2}$. By the convexity of $s_{i}$, the function $h^{k}(\bm{V}_{i})$ is a strongly convex function with modulus no less than $\gamma$. By the optimality of $\bm{V}_{i}^{k}$, the following holds
\begin{equation}\label{eq: V_i proximal-linear update}
     h^{k}(\bm{V}_{i}^{k}) \leq h^{k}(\bm{V}_{i}^{k-1}) - \frac{\gamma}{2}\|\bm{V}_{i}^{k}-\bm{V}_{i}^{k-1}\|_{F}^{2}.
\end{equation}
Based on the inequality~\eqref{eq: V_i proximal-linear update}, it yields for 
\begin{equation} \label{eq: V_i proximal-linear update1}
    \begin{aligned}
& \mathcal{L}(\bm{W}_{\leq i}^{k-1}, \bm{W}_{>i}^{k}, \bm{V}_{<i}^{k-1}, \textcolor{blue}{\bm{V}_{i}^k}, \bm{V}_{>i}^{k}, \bm{U}_{\leq i}^{k-1}, \bm{U}_{>i}^{k}, \bm{W}_{\leq i}^{MC,k-1}, \bm{W}_{>i}^{MC,k})\\
\leq & \mathcal{L}(\bm{W}_{\leq i}^{k-1}, \bm{W}_{>i}^{k}, \bm{V}_{<i}^{k-1}, \textcolor{blue}{\bm{V}_{i}^{k-1}}, \bm{V}_{>i}^{k}, \bm{U}_{\leq i}^{k-1}, \bm{U}_{>i}^{k},\bm{W}_{\leq i}^{MC,k-1}, \bm{W}_{>i}^{MC,k})\\
 - &  \textcolor{blue}{\frac{\gamma}{2}\|\bm{V}_{i}^{k}-\bm{V}_{i}^{k-1}\|_{F}^{2}}
\end{aligned}
\end{equation}
for $i=1, \dots, N-1$, where 
\begin{equation*}
\begin{aligned}
    & h^{k}(\bm{V}_{i}^k) -  h^{k}(\bm{V}_{i}^{k-1}) \\ 
    =&  \mathcal{L}(\bm{W}_{\leq i}^{k-1}, \bm{W}_{>i}^{k}, \bm{V}_{<i}^{k-1}, \bm{V}_{i}^k, \bm{V}_{>i}^{k}, \bm{U}_{\leq i}^{k-1}, \bm{U}_{>i}^{k}, \bm{W}_{\leq i}^{MC,k-1}, \bm{W}_{>i}^{MC,k}) \\
    - &  \mathcal{L}(\bm{W}_{\leq i}^{k-1}, \bm{W}_{>i}^{k}, \bm{V}_{<i}^{k-1}, \bm{V}_{i}^{k-1}, \bm{V}_{>i}^{k}, \bm{U}_{\leq i}^{k-1}, \bm{U}_{>i}^{k}, \bm{W}_{\leq i}^{MC,k-1}, \bm{W}_{>i}^{MC,k}).
\end{aligned}
\end{equation*}

\subsubsection*{Optimization over $\bm{U}_i$}
\textbf{$\bm{U}_{N}^{k}$-block}: similar to the inequality~\eqref{eq: V_i proximal-linear update1}, the  descent quantity is established as follows
\begin{equation} \label{eq: U_N strongly convex update}
    \begin{aligned}
       & \mathcal{L}(\bm{W}_{\leq N}^{k-1}, \bm{V}_{<N}^{k-1}, \bm{V}_{N}^{k}, \bm{U}_{<N}^{k-1}, \textcolor{blue}{\bm{U}_{N}^{k}}, \bm{W}_{\leq N}^{MC,k-1}) \\
        \leq & \mathcal{L}(\bm{W}_{\leq N}^{k-1}, \bm{V}_{<N}^{k-1}, \bm{V}_{N}^{k}, \bm{U}_{<N}^{k-1}, \textcolor{blue}{\bm{U}_{N}^{k-1}}, \bm{W}_{\leq N}^{MC,k-1}) - \textcolor{blue}{\frac{\gamma+\rho}{2}\|\bm{U}_{N}^{k}-\bm{U}_{N}^{k-1}\|_{F}^{2}},
    \end{aligned}
\end{equation}
where the above inequality is because the objective function in subproblem~\eqref{eq: U_N update} is a strongly convex function with a modulus at least $\gamma+\rho$.

\noindent \textbf{$\bm{U}_{i}^{k}$-block ($i<N$)}: the following can be obtained
\begin{equation} \label{eq: U_i proximal update}
\begin{aligned}
& \mathcal{L}(\bm{W}_{\leq i}^{k-1},  \bm{W}_{>i}^{k}, \bm{V}_{<i}^{k-1}, \bm{V}_{\geq i}^{k}, \bm{U}_{<i}^{k-1}, \textcolor{blue}{\bm{U}_{i}^{k}}, \bm{U}_{>i}^{k},\bm{W}_{\leq i}^{MC,k-1}, \bm{W}_{>i}^{MC,k}) \\
\leq & \mathcal{L}(\bm{W}_{\leq i}^{k-1},  \bm{W}_{>i}^{k}, \bm{V}_{<i}^{k-1}, \bm{V}_{\geq i}^{k}, \bm{U}_{<i}^{k-1}, \textcolor{blue}{\bm{U}_{i}^{k-1}}, \bm{U}_{>i}^{k},\bm{W}_{\leq i}^{MC,k-1}, \bm{W}_{>i}^{MC,k}) \\
-& \textcolor{blue}{\frac{\alpha}{2}\|\bm{U}_{i}^{k}-\bm{U}_{i}^{k-1}\|_{F}^{2}}
\end{aligned}
\end{equation}
for $i=1, \dots, N-1$ since $\bm{U}_{i}^{k}$ is the optimal solution for subproblem~\eqref{eq: U_i update}.

\subsubsection*{Optimization over $\bm{W}_i$}
\textbf{$\bm{W}_{i}^{k}$-block ($i\leq N$)}: $\bm{W}_{i}^{k}$ is updated according to the following
\begin{equation*}
\bm{W}_{i}^{k} \rightarrow \underset{\bm{W}_{i}}{\operatorname{argmin}}\left\{\frac{\rho}{2}\|\bm{U}_{i}^{k}-\bm{W}_{i} \bm{V}_{i-1}^{k-1}\|_{F}^{2}+\frac{\tau}{2}\|\bm{W}_i - \bm{W}_i^{MC,k-1}\|_{F}^{2}\right\},    
\end{equation*}
where $h^{k}(\bm{W}_{i})=\frac{\rho}{2}\|\bm{U}_{i}^{k}-\bm{W}_{i} \bm{V}_{i-1}^{k-1}\|_{F}^{2}+\frac{\tau}{2}\|\bm{W}_i - \bm{W}_i^{MC,k-1}\|_{F}^{2}$ is  a strongly convex function with modulus at least $\tau$. Accordingly, the following holds
\begin{equation} \label{eq: W_i strongly convex update}
    \begin{aligned}
& \mathcal{L}(\bm{W}_{<i}^{k-1}, \textcolor{blue}{\bm{W}_{i}^{k}}, \bm{W}_{>i}^{k}, \bm{V}_{<i}^{k-1},  \bm{V}_{\geq i}^{k}, \bm{U}_{<i}^{k-1},  \bm{U}_{\geq i}^{k}, \bm{W}_{\leq i}^{MC,k-1}, \bm{W}_{>i}^{MC,k}) \\
\leq &  \mathcal{L}(\bm{W}_{<i}^{k-1}, \textcolor{blue}{\bm{W}_{i}^{k-1}}, \bm{W}_{>i}^{k}, \bm{V}_{<i}^{k-1},  \bm{V}_{\geq i}^{k}, \bm{U}_{<i}^{k-1},  \bm{U}_{\geq i}^{k}, \bm{W}_{\leq i}^{MC,k-1}, \bm{W}_{>i}^{MC,k}) \\
 -&\textcolor{blue}{\frac{\tau}{2}\|\bm{W}_{i}^{k}-\bm{W}_{i}^{k-1}\|_{F}^{2}}
\end{aligned}
\end{equation}
due to the relationship between $\mathcal{L}(\bm{W}_{<i}^{k-1}, \bm{W}_{i}, \bm{W}_{>i}^{k}, \bm{V}_{<i}^{k-1},  \bm{V}_{\geq i}^{k}, \bm{U}_{<i}^{k-1},  \bm{U}_{\geq i}^{k}, \bm{W}_{\leq i}^{MC,k-1}, \bm{W}_{>i}^{MC,k})$ and $h^{k}(\bm{W}_{i})$.

\subsubsection*{Optimization over $\bm{W}_i^{MC}$}
\textbf{$\bm{W}_i^{MC}$-block ($i\leq N$)}: the descent quantity for $\bm{W}_i^{MC}$ can be derived as follows
\begin{equation} \label{eq: G_i proximal update}
    \begin{aligned}
& \mathcal{L}(\bm{W}_{<i}^{k-1}, \bm{W}_{\geq i}^{k}, \bm{V}_{<i}^{k-1},  \bm{V}_{\geq i}^{k}, \bm{U}_{<i}^{k-1},  \bm{U}_{\geq i}^{k}, \bm{W}_{<i}^{MC,k-1},  \textcolor{blue}{\bm{W}_{i}^{MC,k}}, \bm{W}_{>i}^{MC,k}) \\
\leq &  \mathcal{L}(\bm{W}_{<i}^{k-1}, \bm{W}_{\geq i}^{k}, \bm{V}_{<i}^{k-1},  \bm{V}_{\geq i}^{k}, \bm{U}_{<i}^{k-1},  \bm{U}_{\geq i}^{k}, \bm{W}_{<i}^{MC,k-1},  \textcolor{blue}{\bm{W}_{i}^{MC,k-1}}, \bm{W}_{>i}^{MC,k})\\
-& \textcolor{blue}{\frac{\alpha}{2}\|\bm{W}_{i}^{MC,k}-\bm{W}_{i}^{MC,k-1}\|_{F}^{2}},
\end{aligned}
\end{equation}
where the above inequality~\eqref{eq: G_i proximal update} is due to the fact that $\bm{W}_{i}^{MC,k}$ (or the low-rank matrix/tensor factors) is the optimal solution for subproblem~\eqref{eq: W_i^MC update}.

By summing up inequalities~\eqref{eq: V_N proximal update} (or~\eqref{eq: V_N proximal-linear update2}), \eqref{eq: U_N strongly convex update}, \eqref{eq: U_i proximal update}, \eqref{eq: W_i strongly convex update}, and~\eqref{eq: G_i proximal update}, it yields the
\begin{equation*}
 \mathcal{L}(\mathcal{P}^{k}) \leq \mathcal{L}(\mathcal{P}^{k-1})-\lambda\|\mathcal{P}^{k}-\mathcal{P}^{k-1}\|_{F}^{2},
\end{equation*}
where $\lambda:=\min \left\{\frac{\alpha}{2}, \frac{\gamma+\rho}{2},\frac{\tau}{2}\right\}$ (or $\lambda:=\min \left\{\frac{\alpha}{2}, \frac{\gamma+\rho}{2},\frac{\tau}{2}, \alpha+\frac{\gamma-L_{R}}{2}\right\}$). 
\end{proof}

\section{Proof of Lemma~\ref{lemma: subgradient low bound}} \label{appendix: proof of subgradient low bound}
The detailed proof of Lemma~\ref{lemma: subgradient low bound} is provided below.
\begin{proof}
    The inequality~\eqref{eq: subgradient bound} is established via bounding each term of $\partial \mathcal{L}(\mathcal{P}^{k})$. Specifically, the following holds
\begin{subequations}\label{eq: gradient layerN}
\begin{align}
& \mathbf{0} \in \partial s_{N}(\bm{V}_{N}^{k})+\partial \mathcal{R}_{n}(\bm{V}_{N}^{k} ; \bm{Y})+\gamma(\bm{V}_{N}^{k}-\bm{U}_{N}^{k-1})+\alpha(\bm{V}_{N}^{k}-\bm{V}_{N}^{k-1}), \label{eq: gradient layerN-1}\\
 & \mathbf{0} \in \partial s_{N}(\bm{V}_{N}^{k})+\nabla \mathcal{R}_{n}(\bm{V}_{N}^{k-1} ; \bm{Y})+\gamma(\bm{V}_{N}^{k}-\bm{U}_{N}^{k-1})+\alpha(\bm{V}_{N}^{k}-\bm{V}_{N}^{k-1}), \textnormal{ (proximal-linear)} \label{eq: gradient layerN-2}\\
& \mathbf{0}=\gamma(\bm{U}_{N}^{k}-\bm{V}_{N}^{k})+\rho(\bm{U}_{N}^{k}-\bm{W}_{N}^{k-1} \bm{V}_{N-1}^{k-1}), \label{eq: gradient layerN-3}\\
& \mathbf{0} \in \rho(\bm{W}_{N}^{k} \bm{V}_{N-1}^{k-1}-\bm{U}_{N}^{k}) \bm{V}_{N-1}^{k-1}{}^{\top}+\tau\left(\bm{W}_{N}^{k}-\bm{W}_{N}^{MC,k-1}\right), \label{eq: gradient layerN-4}\\
& \mathbf{0} \in \partial  \tau_{N}(\bm{W}_{N}^{MC,k})+ \tau(\bm{W}_N^{MC,k} - \bm{W}_N^{k})+\alpha(\bm{W}_N^{MC,k} - \bm{W}_N^{MC,k-1})+\partial \chi_{\mathcal{MC}}(\bm{W}_{N}^{MC,k}), \label{eq: gradient layerN-5}
\end{align} 
\end{subequations}
where  \begin{equation*}
    \chi_{\mathcal{MC}}(\bm{W})=\left\{
     \begin{array}{ll}
   0, & \text{ if } \mathcal{MC}(\bm{W})=0, \\
+\infty, & \text{ else}.
 \end{array}\right. 
\end{equation*}
\eqref{eq: gradient layerN-1}, \eqref{eq: gradient layerN-2}, \eqref{eq: gradient layerN-3},  \eqref{eq: gradient layerN-4}, and  \eqref{eq: gradient layerN-5} are  due to the optimality conditions of all updates in \eqref{eq: V_N update}, \eqref{eq: V_N proximal-linear}, \eqref{eq: U_N update}, \eqref{eq: W_i update}, and \eqref{eq: W_i^MC update}, respectively.

For $i=N-1, \dots, 1$, the following holds
\begin{subequations}\label{eq: gradient layeri}
\begin{align}
& \mathbf{0} \in \partial s_{i}(\bm{V}_{i}^{k})+\gamma(\bm{V}_{i}^{k}-\sigma_{i}(\bm{U}_{i}^{k-1}))+\rho \bm{W}_{i+1}^{k}{}^{\top}(\bm{W}_{i+1}^{k} \bm{V}_{i}^{k}-\bm{U}_{i+1}^{k}), \label{eq: gradient layeri-1} \\
& \mathbf{0} \in \gamma\left[(\sigma_{i}(\bm{U}_{i}^{k})-\bm{V}_{i}^{k}) \odot \partial \sigma_{i}(\bm{U}_{i}^{k})\right]+\rho(\bm{U}_{i}^{k}-\bm{W}_{i}^{k-1} \bm{V}_{i-1}^{k-1})+\alpha(\bm{U}_{i}^{k}-\bm{U}_{i}^{k-1}), \label{eq: gradient layeri-2} \\
& \mathbf{0} \in \rho(\bm{W}_{i}^{k} \bm{V}_{i-1}^{k-1}-\bm{U}_{i}^{k}) \bm{V}_{i-1}^{k-1}{}^{\top}+\tau(\bm{W}_{i}^{k}-\bm{W}_{i}^{MC,k-1}), \label{eq: gradient layeri-3} \\
& \mathbf{0} \in \partial  r_{i}(\bm{W}_{i}^{MC,k})+ \tau(\bm{W}_i^{MC,k} - \bm{W}_i^{k})+\alpha(\bm{W}_i^{MC,k} - \bm{W}_i^{MC,k-1})+\partial \chi_{\mathcal{MC}}(\bm{W}_{i}^{MC,k}), \label{eq: gradient layeri-4}
\end{align} 
\end{subequations}
where \eqref{eq: gradient layeri-1}, \eqref{eq: gradient layeri-2}, \eqref{eq: gradient layeri-3}, and \eqref{eq: gradient layeri-4} are due to the optimality conditions of all updates in \eqref{eq: V_i update}, \eqref{eq: U_i update}, \eqref{eq: W_i update}, and \eqref{eq: W_i^MC update}, respectively. $\bm{V}_{0}^{k} \equiv \bm{V}_{0}=\bm{X}$ for all $k$, and $\odot$ is the Hadamard product. Through the above relationship~\eqref{eq: gradient layerN}, we have
\begin{equation}\label{eq: gradient difference layerN}
\begin{aligned}
& -\alpha(\bm{V}_N^k-\bm{V}_N^{k-1})-\gamma(\bm{U}_N^k-\bm{U}_N^{k-1}) \in \partial s_N(\bm{V}_N^k)+\partial \mathcal{R}_n(\bm{V}_N^k ; \bm{Y})+\gamma(\bm{V}_N^k-\bm{U}_N^k)=\partial_{\bm{V}_N} \mathcal{L}(\mathcal{P}^k), \\
& \left(\nabla \mathcal{R}_n(\bm{V}_N^k ; \bm{Y})-\nabla \mathcal{R}_n(\bm{V}_N^{k-1} ; \bm{Y})\right)-\alpha(\bm{V}_N^k-\bm{V}_N^{k-1})-\gamma(\bm{U}_N^k-\bm{U}_N^{k-1})\in \partial_{\bm{V}_N} \mathcal{L}(\mathcal{P}^k), \textnormal{ (proximal-linear)} \\
& -\rho(\bm{W}_N^k-\bm{W}_N^{k-1}) \bm{V}_{N-1}^k-\rho \bm{W}_N^{k-1}(\bm{V}_{N-1}^k-\bm{V}_{N-1}^{k-1}) \\
& =\gamma(\bm{U}_N^k-\bm{V}_N^k)+\rho(\bm{U}_N^k-\bm{W}_N^k \bm{V}_{N-1}^k)=\partial_{\bm{U}_N} \mathcal{L}(\mathcal{P}^k), \\
& \rho \bm{W}_N^k[\bm{V}_{N-1}^k(\bm{V}_{N-1}^k-\bm{V}_{N-1}^{k-1})^{\top}+(\bm{V}_{N-1}^k-\bm{V}_{N-1}^{k-1}) \bm{V}_{N-1}^{k-1}{ }^{\top}]-\rho \bm{U}_N^k(\bm{V}_N^k-\bm{V}_N^{k-1})^{\top}  \\
& +\tau(\bm{W}_N^{MC,k} - \bm{W}_N^{MC,k-1}) \in \partial \rho(\bm{W}_N^k \bm{V}_{N-1}^k-\bm{U}_N^k) \bm{V}_{N-1}^k{}^{\top}+\tau(\bm{W}_N^{k}-\bm{W}_N^{MC,k})=\partial_{\bm{W}_N} \mathcal{L}(\mathcal{P}^k), \\
& -\alpha(\bm{W}_N^{MC,k}-\bm{W}_N^{MC,k-1}) \in \partial_{\bm{W}_N^{MC}} \mathcal{L}(\mathcal{P}^k).
\end{aligned}
\end{equation}

For $i=N-1, \ldots, 1$, the relationship~\eqref{eq: gradient layeri} implies
\begin{equation} \label{eq: gradient difference layeri}
\begin{aligned}
& -\gamma(\sigma_i(\bm{U}_i^k)-\sigma_i(\bm{U}_i^{k-1})) \in \partial s_i(\bm{V}_i^k)+\rho(\bm{V}_i^k-\sigma_i(\bm{U}_i^k)) \\
& +\gamma \bm{W}_{i+1}^k{ }^{\top}(\bm{W}_{i+1}^k \bm{V}_i^k-\bm{U}_{i+1}^k)=\partial_{\bm{V}_i} \mathcal{L}(\mathcal{P}^k), \\
& -\rho \bm{W}_i^{k-1}(\bm{V}_{i-1}^k-\bm{V}_{i-1}^{k-1})-\rho(\bm{W}_i^k-\bm{W}_i^{k-1}) \bm{V}_{i-1}^k-\alpha(\bm{U}_i^k-\bm{U}_i^{k-1}) \\
& \in \gamma\left[(\sigma_i(\bm{U}_i^k)-\bm{V}_i^k) \odot \partial \sigma_i(\bm{U}_i^k)\right]+\rho(\bm{U}_i^k-\bm{W}_i^k \bm{V}_{i-1}^k)=\partial_{\bm{U}_i} \mathcal{L}(\mathcal{P}^k) \text{, } \\
& \rho\bm{W}_i^k[\bm{V}_{i-1}^k(\bm{V}_{i-1}^k-\bm{V}_{i-1}^{k-1})^{\top}+(\bm{V}_{i-1}^k-\bm{V}_{i-1}^{k-1}) \bm{V}_{i-1}^{k-1}]-\rho \bm{U}_i^k(\bm{V}_{i-1}^k-\bm{V}_{i-1}^{k-1})^{\top} \\
& +\tau(\bm{W}_i^{MC,k}-\bm{W}_i^{MC,k-1}) \in \rho(\bm{W}_i^k \bm{V}_{i-1}^k-\bm{U}_i^k) \bm{V}_{i-1}^k{}^{\top}=\partial_{\bm{W}_i} \mathcal{L}(\mathcal{P}^k),\\
& -\alpha(\bm{W}_i^{MC,k}-\bm{W}_i^{MC,k-1}) \in \partial_{\bm{W}_i^{MC}} \mathcal{L}(\mathcal{P}^k).
\end{aligned}
\end{equation}

Based on the above relationships, and by the Lipschitz continuity of the activation function on the bounded set $\{\mathcal{P}:\|\mathcal{P}\|_{F} \leq \mathcal{B}\}$ and the bounded assumption of both $\mathcal{P}^{k-1}$ and $\mathcal{P}^{k}$, we have
\begin{equation} \label{eq: gradient error layerN}
   \begin{array}{ll}
\|\xi_{\bm{V}_{N}}^{k}\|_{F} \leq \alpha\|\bm{V}_{N}^{k}-\bm{V}_{N}^{k-1}\|_{F}+\gamma\|\bm{U}_{N}^{k}-\bm{U}_{N}^{k-1}\|_{F}, & \xi_{\bm{V}_{N}}^{k} \in \partial_{\bm{V}_{N}} \mathcal{L}(\mathcal{P}^{k}), \\
(\text{or }\|\xi_{\bm{V}_{N}}^{k}\|_{F} \leq(L_{R}+\alpha)\|\bm{V}_{N}^{k}-\bm{V}_{N}^{k-1}\|_{F}+\gamma\|\bm{U}_{N}^{k}-\bm{U}_{N}^{k-1}\|_{F}) \text{ proximal-linear}& \\
\|\xi_{\bm{U}_{N}}^{k}\|_{F} \leq \rho \mathcal{B}\|\bm{W}_{N}^{k}-\bm{W}_{N}^{k-1}\|_{F}+\rho \mathcal{B}\|\bm{V}_{N-1}^{k}-\bm{V}_{N-1}^{k-1}\|_{F}, & \xi_{\bm{U}_{N}}^{k} \in \partial_{\bm{U}_{N}} \mathcal{L}(\mathcal{P}^{k}),  \\
\|\xi_{\bm{W}_{N}}^{k}\|_{F} \leq 2 \rho \mathcal{B}^{2}\|\bm{V}_{N-1}^{k}-\bm{V}_{N-1}^{k-1}\|_{F}+\rho \mathcal{B}\|\bm{V}_{N}^{k}-\bm{V}_{N}^{k-1}\|_{F} + \tau\|\bm{W}_{N}^{MC,k}-\bm{W}_{N}^{MC,k-1}\|_{F}, & \xi_{\bm{W}_{N}}^k \in \partial_{\bm{W}_{N}} \mathcal{L}(\mathcal{P}^{k}),\\
\|\xi_{\bm{W}_N^{MC}}^{k}\|_{F}  \leq \alpha\|\bm{W}_N^{MC,k}-\bm{W}_N^{MC,k-1}\|_F,  & \xi_{\bm{W}_N^{MC}}^k \in \partial_{\bm{W}_N^{MC}} \mathcal{L}(\mathcal{P}^{k}),
\end{array} 
\end{equation}
and for $i=N-1, \ldots, 1$,
\begin{equation} \label{eq: gradient error layeri}
    \begin{array}{ll}
\|\xi_{\bm{V}_{i}}^{k}\|_{F} \leq \gamma L_{\mathcal{B}}\|\bm{U}_{i}^{k}-\bm{U}_{i}^{k-1}\|_{F}, & \xi_{\bm{V}_{i}}^{k} \in \partial_{\bm{V}_{i}} \mathcal{L}(\mathcal{P}^{k}), \\
\|\xi_{\bm{U}_{i}}^{k}\|_{F} \leq \rho \mathcal{B}\|\bm{V}_{i-1}^{k}-\bm{V}_{i-1}^{k-1}\|_{F}+\rho \mathcal{B}\|\bm{W}_{i}^{k}-\bm{W}_{i}^{k-1}\|_{F}+\alpha\|\bm{U}_{i}^{k}-\bm{U}_{i}^{k-1}\|_{F}, & \xi_{\bm{U}_{i}}^{k} \in \partial_{\bm{U}_{i}} \mathcal{L}(\mathcal{P}^{k}), \\
\|\xi_{\bm{W}_{i}}^{k}\|_{F} \leq(2\rho \mathcal{B}^{2}+\rho \mathcal{B})\|\bm{V}_{i-1}^{k}-\bm{V}_{i-1}^{k-1}\|_{F}+\tau\|\bm{W}_i^{MC,k}-\bm{W}_i^{MC,k-1}\|_{F}, & \xi_{\bm{W}_{i}}^{k} \in \partial_{\bm{W}_{i}} \mathcal{L}(\mathcal{P}^{k}), \\
\|\xi_{\bm{W}_i^{MC}}^{k}\|_{F}  \leq \alpha\|\bm{W}_i^{MC,k}-\bm{W}_i^{MC,k-1}\|_F,  & \xi_{\bm{W}_i^{MC}}^k \in \partial_{\bm{W}_i^{MC}} \mathcal{L}(\mathcal{P}^{k}).
\end{array}
\end{equation}
Summing the above inequalities~\eqref{eq: gradient error layerN} and \eqref{eq: gradient error layeri}, the subgradient bound~\eqref{eq: subgradient bound} can be obtained for any positive integer $k$,
\begin{equation*}
    \begin{aligned}
         &\operatorname{dist}(\mathbf{0}, \partial \mathcal{L} (\mathcal{P}^{k})) \\
        \leq&  \delta \sum_{i=1}^{N}\left[\|\bm{W}_{i}^{k}-\bm{W}_{i}^{k-1}\|_{F}+\|\bm{V}_{i}^{k}-\bm{V}_{i}^{k-1}\|_{F}+\|\bm{U}_{i}^{k}-\bm{U}_{i}^{k-1}\|_{F}+\|\bm{W}_{i}^{MC,k}-\bm{W}_{i}^{MC,k-1}\|_{F}\right] \\
       \leq &  \bar{\delta}\|\mathcal{P}^{k}-\mathcal{P}^{k-1}\|_{F},
    \end{aligned}
\end{equation*}
where
\begin{equation*}
    \delta:=\max \{\gamma, \alpha+\rho \mathcal{B}, \alpha+\gamma L_{\mathcal{B}}, 2 \rho \mathcal{B}+ 2\rho \mathcal{B}^{2}, \alpha + \tau\},
\end{equation*}
(or, for the proximal-linear case, $ \delta:=\max \{\gamma, L_R+\alpha+\rho \mathcal{B}, \alpha+\gamma L_{\mathcal{B}}, 2 \rho \mathcal{B}+ 2 \rho \mathcal{B}^{2}, \alpha + \tau\}$). 
\end{proof}

\section{Proof of Theorem~\texorpdfstring{\ref{thm: global convergence}}{} } \label{appendix: theorem of global convergence}

To build the global convergence of our iterative sequence  $\{\mathcal{P}^{k}\}_{k \in \mathbb{N}}$ from Algorithm~\ref{alg: NN-BCD}, the function $\mathcal{L}(\Btheta, \Btheta_{MC}, \bm{\mathcal{V}},\bm{\mathcal{U}})$ needs to have the Kurdyka \L{}ojasiewicz (K\L)  property as follows.
\begin{definition}[\normalfont K\L~property~\citep{attouch2013convergence,bolte2014proximal}]\label{df: kl}
A real function $f: \mathbb{R}^p \to (-\infty,+\infty]$ has the  Kurdyka \L{}ojasiewicz (K\L) property, namely, for any point $\bar{\bm{u}}\in \mathbb{R}^p$, in a neighborhood $N(\bar{\bm{u}},\sigma)$, there exists a desingularizing function $\phi(s)=cs^{1-\theta}$ for some $c>0$ and $\theta \in [0,1)$ such that 
\begin{equation}\label{eq: KL inequality}
    \phi'(|f(\bm{u})-f(\bar{\bm{u}})|)\mathrm{d}(0,\partial f(\bm{u}))\geq 1
\end{equation}
for any  $\bm{u} \in N(\bar{\bm{u}},\sigma)$ and  $f(\bm{u})\neq f(\bar{\bm{u}})$.
\end{definition}    
The real analytic and semi-algebraic functions, which are related to K\L{}  property, are introduced below.
\begin{definition}[\normalfont Real analytic~\citep{krantz2002primer}]
    A function $h$ with domain an open set $U \subset \mathbb{R}$ and range the set of either all real or complex numbers, is said to be real analytic at $u$ if the function $h$ may be represented by a convergent power series on some interval of positive radius centered at $u$, \ie, $h(x)=$ $\sum_{j=0}^{\infty} \alpha_{j}(x-u)^{j}$, for some $\left\{\alpha_{j}\right\} \subset \mathbb{R}$. The function is said to be real analytic on $V \subset U$ if it is real analytic at each $u \in V$ \citep[][Definition 1.1.5]{krantz2002primer}. The real analytic function $f$ over $\mathbb{R}^{p}$ for some positive integer $p>1$ can be defined similarly.
\end{definition}

\begin{definition}[\normalfont Semi-algebraic~\citep{bolte2014proximal}]  \label{df: semi-algerba} 
A subset $S$ of $\mathbb{R}^p$ is a real \textbf{semi-algebraic set} if there exists  a finite number of real polynomial functions $g_{ij},h_{ij}$:  $\mathbb{R}^p \to \mathbb{R}$ such that $S=\cup_{j=1}^q\cap_{i=1}^m\{\bm{u}\in \mathbb{R}^p:g_{ij}(\bm{u})=0 \textrm{ and }h_{ij}(\bm{u})<0 \}.$ In addition, a function $h:\mathbb{R}^{p+1} \to \mathbb{R}\cup{+\infty}$ is called 	\textbf{semi-algebraic} if its graph $\{(\bm{u}, t)\in \mathbb{R}^{p+1}: h(\bm{u})=t \}$ is a real semi-algebraic set.
 \end{definition}
 
Based on the above definitions, the following lemma can be obtained.
\begin{lemma}\label{lemma: KL property}
Most of the commonly used NN training models~\eqref{eq: final formulation} can be verified to satisfy the following 
\begin{enumerate}[topsep=0pt,itemsep=0pt,partopsep=0pt,label=(\alph*)]  
    \item the loss function $\ell$ is a proper lower semicontinuous and nonnegative function. For example,  the squared, logistic, hinge, or cross-entropy losses.
    \item the activation functions $\sigma_{i} (i=1 \ldots, N-1)$ are Lipschitz continuous on any bounded set. For example, ReLU, leaky ReLU, sigmoid, hyperbolic tangent, linear, polynomial, or softplus activations.
    \item the regularizers $r_{i}$ and $s_{i}(i=1, \ldots, N)$ are nonegative lower semicontinuous convex functions. $r_{i}$ and $s_{i}$ are the squared $\ell_{2}$ norm, the $\ell_{1}$ norm, the elastic net, the indicator function of some nonempty closed convex set (such as the nonnegative closed half-space, box set or a closed interval $[0,1]$) and semi-algebraic sets, or 0 if no regularization.
    \item all these functions $\ell, \sigma_{i}, r_{i}$ and $s_{i}(i=1, \ldots, N)$ are either real analytic or semi-algebraic, and continuous on their domains.
\end{enumerate}
Accordingly, the objective function $\mathcal{L}(\Btheta, \Btheta_{MC}, \bm{\mathcal{V}},\bm{\mathcal{U}})$ in~\eqref{eq: final formulation} has \textbf{Kurdyka \L{}ojasiewicz (K\L)} property. 
\end{lemma}

\begin{proof}[Proof of Lemma~\ref{lemma: KL property}]

\textbf{On the loss function $\ell$:} Since these losses are all nonnegative and continuous on their domains, they are proper lower semicontinuous and lower bounded by 0. In the following, we only verify that they are either real analytic or semi-algebraic.
\begin{enumerate}[topsep=0pt,itemsep=0pt,partopsep=0pt,left=0pt,label=(a\theenumi)]  
    \item If $\ell(t)$ is the squared $(t^{2})$ or exponential  $(e^{t})$ loss, then according to~\citet{krantz2002primer}, they are real analytic.
    \item If $\ell(t)$ is the logistic loss $(\log (1+\mathrm{e}^{-t}))$, since it is a composition of logarithm and exponential functions which both are real analytic, thus according to~\citet{krantz2002primer}, the logistic loss is real analytic.
    \item If $\ell(\bm{u} ; \bm{y})$ is the cross-entropy loss, \ie, given $\bm{y} \in \mathbb{R}^{d_{N}}, \ell(\bm{u} ; \bm{y})=-\frac{1}{d_{N}}[\langle\bm{y}, \log \widehat{\bm{y}}(\bm{u})\rangle+\langle\mathbf{1}-\bm{y}, \log (\mathbf{1}-\widehat{\bm{y}}(\bm{u}))\rangle]$, where $\log$ is performed elementwise and $(\widehat{\bm{y}}(\bm{u})_{i})_{1 \leq i \leq d_{N}}:=((1+\mathrm{e}^{-u_{i}})^{-1})_{1 \leq i \leq d_{N}}$ for any $\bm{u} \in \mathbb{R}^{d_{N}}$, which can be viewed as a linear combination of logistic functions, then by (a2) and \citep{krantz2002primer}, it is also analytic.
    \item If $\ell$ is the hinge loss, \ie, given $\bm{y} \in \mathbb{R}^{d_{N}}, \ell(\bm{u} ; \bm{y}):=\max \{0,1-\langle\bm{u}, \bm{y}\rangle\}$ for any $\bm{u} \in \mathbb{R}^{d_{N}}$, by~\citet{bochnak2013real}, it is semi-algebraic, because its graph is $\operatorname{cl}(\mathcal{D})$, the closure of the set $\mathcal{D}$, where $\mathcal{D}=\{(\bm{u}, z): 1-\langle\bm{u}, \bm{y}\rangle-z=0, \mathbf{1}-\bm{u} \succ 0\} \cup\{(\bm{u}, z): z=0,\langle\bm{u}, \bm{y}\rangle-1>0\}.$
\end{enumerate}

\textbf{On the activation function $\sigma_{i}$:} Since all the considered specific activations are continuous on their domains, they are Lipschitz continuous on any bounded set. In the following, we only need to check that they are either real analytic or semi-algebraic.
\begin{enumerate}[topsep=0pt,itemsep=0pt,partopsep=0pt,left=0pt,label=(b\theenumi)]  
    \item If $\sigma_{i}$ is a linear or polynomial function, then according to~\citep{krantz2002primer} is real analytic.
    \item If $\sigma_{i}(t)$ is sigmoid, $(1+\mathrm{e}^{-t})^{-1}$, or hyperbolic tangent, $\tanh(t):=\frac{\mathrm{e}^{t}-\mathrm{e}^{-t}}{\mathrm{e}^{t}+\mathrm{e}^{-t}}$, then the sigmoid function is a composition $g \circ h$ of these two functions where $g(u)=\frac{1}{1+u}, u>0$ and $h(t)=\mathrm{e}^{-t}$ (resp. $g(u)=1-\frac{2}{u+1}, u>0$ and $h(t)=\mathrm{e}^{2 t}$ in the hyperbolic tangent case). According to~\citep{krantz2002primer}, $g$ and $h$ in both cases are real analytic. Thus,  sigmoid and hyperbolic tangent functions are real analytic.
    \item If $\sigma_{i}$ is $\operatorname{ReLU}$, \ie, $\sigma_{i}(u):=\max \{0, u\}$, then we can show that $\operatorname{ReLU}$ is semi-algebraic since its graph is cl($\mathcal{D})$, the closure of the set $\mathcal{D}$, where $\mathcal{D}=\{(u, z): u-z=0, u>0\} \cup\{(u, z): z=0,-u>0\}.$
    \item Similar to the ReLU case, if $\sigma_{i}$ is leaky $\operatorname{ReLU}$, \ie, $\sigma_{i}(u)=u$ if $u>0$, otherwise $\sigma_{i}(u)=a u$ for some $a>0$, then we can similarly show that leaky ReLU is semi-algebraic since its graph is $\operatorname{cl}(\mathcal{D})$, the closure of the set $\mathcal{D}$, where $\mathcal{D}=\{(u, z): u-z=0, u>0\} \cup\{(u, z): a u-z=0,-u>0\}.$
    \item  If $\sigma_{i}$ is polynomial, then according to~\citep{krantz2002primer}, it is real analytic.
    \item  If $\sigma_{i}$ is softplus, \ie, $\sigma_{i}(u)=\frac{1}{t} \log (1+\mathrm{e}^{t u})$ for some $t>0$, since it is a composition of two analytic functions $\frac{1}{t} \log (1+u)$ and $\mathrm{e}^{t u}$, then according to~\citep{krantz2002primer}, it is real analytic.
\end{enumerate}

\textbf{On $r_{i}(\bm{W}_{i}), s_{i}(\bm{V}_{i})$:} By the specific forms of these regularizers, they are nonnegative, lower semicontinuous and continuous on their domain. In the following, we only need to verify they are either real analytic or semi-algebraic.
\begin{enumerate}[topsep=0pt,itemsep=0pt,partopsep=0pt,left=0pt,label=(c\theenumi)]  
\item the squared $\ell_{2}$ norm $\|\cdot\|_{2}^{2}$: According to~\citep{bochnak2013real}, the $\ell_{2}$ norm is semi-algebraic, so is its square where $g(t)=t^{2}$ and $h(\bm{W})=\|\bm{W}\|_{2}$.

\item the squared Frobenius norm $\|\cdot\|_{F}^{2}$: The squared Frobenius norm is semiaglebraic since it is a finite sum of several univariate squared functions.

\item the elementwise $\ell_1$ norm $\|\cdot\|_{1}$: Note that $\|\bm{W}\|_{1}=\sum_{i, j}|\bm{W}_{i j}|$ is the finite sum of absolute functions $h(t)=|t|$. According to~\citep{bochnak2013real}, the absolute value function is semi-algebraic since its graph is the closure of the following semi-algebraic set $\mathcal{D}=\{(t, s): t+s=0,-t>0\} \cup\{(t, s): t-s=0, t>0\}$. Thus, the elementwise 1-norm is semi-algebraic.

\item the elastic net: Note that the elastic net is the sum of the elementwise 1-norm and the squared Frobenius norm. Thus, by (c2), (c3), and~\citep{bochnak2013real}, the elastic net is semi-algebraic.

\item If $r_{i}$ or $s_{i}$ is the indicator function of nonnegative closed half-space or a closed interval (box constraints) or semi-algebraic sets,  by~\citep{bochnak2013real, bolte2014proximal}, any polyhedral set is semi-algebraic such as the nonnegative orthant $\mathbb{R}_{+}^{p \times q}=\{\bm{W} \in \mathbb{R}^{p \times q}, \bm{W}_{i j} \geq 0, \forall i, j\}$ and the closed interval. In addition, $\ell_0$ and $\ell_p(p>0)$ are semi-algebraic. Thus, $r_{i}$ or $s_{i}$ is semi-algebraic in this case.
\end{enumerate}
We first verify the K\L{} property of $\mathcal{L}$. From~\eqref{eq: final formulation}, we have
\begin{equation*}
\begin{aligned}
        &\mathcal{L}(\Btheta, \Btheta_{MC}, \bm{\mathcal{V}},\bm{\mathcal{U}})
        :=   \mathcal{R}_{n}\left(\bm{V}_{N} ; \bm{Y}\right)+\sum_{i=1}^{N} r_{i}\left(\bm{W}_{i}^{MC}\right)+\sum_{i=1}^{N} s_{i}\left(\bm{V}_{i}\right)+\frac{\rho}{2} \sum_{i=1}^{N}\|\bm{U}_{i}-\bm{W}_{i} \bm{V}_{i-1}\|_{F}^{2}  \\
    &+\frac{\gamma}{2} \sum_{i=1}^{N}\|\bm{V}_{i}-\sigma_{i}(\bm{U}_{i})\|_{F}^{2}+ \frac{\tau}{2} \sum_{i=1}^{N}\|\bm{W}_i - \bm{W}_i^{MC}\|_{F}^{2}+ \chi_{\mathcal{MC}}(\Btheta_{MC}),
\end{aligned}
\end{equation*}
which mainly includes the following types of functions, \ie,
\begin{equation*}
    \mathcal{R}_{n}\left(\bm{V}_{N} ; \bm{Y}\right),  r_{i}\left(\bm{W}_{i}^{MC}\right), s_{i}\left(\bm{V}_{i}\right),\|\bm{U}_{i}-\bm{W}_{i} \bm{V}_{i-1}\|_{F}^{2},\|\bm{V}_{i}-\sigma_{i}(\bm{U}_{i})\|_{F}^{2},\|\bm{W}_i - \bm{W}_i^{MC}\|_{F}^{2}, \chi_{\mathcal{MC}}(\Btheta_{MC}).
\end{equation*}
To verify the K\L{} property of the function $\mathcal{L}$, we consider the above functions one. 

On $\mathcal{R}_{n}(\bm{V}_{N} ; \bm{Y})$: Note that given the output data $\bm{Y}, \mathcal{R}_{n}(\bm{V}_{N} ; \bm{Y}):=\frac{1}{n} \sum_{j=1}^{n} \ell((\bm{V}_{N})_{: j}, \bm{y}_{j})$, where $\ell: \mathbb{R}^{d_{N}} \times \mathbb{R}^{d_{N}} \rightarrow$ $\mathbb{R}_{+} \cup\{0\}$ is some loss function. If $\ell$ is real analytic (resp. semi-algebraic), then $\mathcal{R}_{n}(\bm{V}_{N} ; \bm{Y})$ is real-analytic (resp. semi-algebraic).

On $\|\bm{V}_{i}-\sigma_{i}(\bm{U}_{i})\|_{F}^{2}$ : Note that $\|\bm{V}_{i}-\sigma_{i}(\bm{U}_{i})\|_{F}^{2}$ is a finite sum of simple functions of the form, $|v-\sigma_{i}(u)|^{2}$ for any $u, v \in \mathbb{R}$. If $\sigma_{i}$ is real analytic (resp. semi-algebraic), then $v-\sigma_{i}(u)$ is real analytic (resp. semi-algebraic), and further $|v-\sigma_{i}(u)|^{2}$ is also real analytic (resp. semi-algebraic) since $|v-\sigma_{i}(u)|^{2}$ can be viewed as the composition $g \circ h$ of these two functions where $g(t)=t^{2}$ and $h(u, v)=v-\sigma_{i}(u)$.

On $\|\bm{U}_{i}-\bm{W}_{i} \bm{V}_{i-1}\|_{F}^{2}$: Note that the function $\|\bm{U}_{i}-\bm{W}_{i} \bm{V}_{i-1}\|_{F}^{2}$ is a polynomial function with the variables $\bm{U}_{i}, \bm{W}_{i}$ and $\bm{V}_{i-1}$, and thus according to~\citep{krantz2002primer,bochnak2013real}, it is both real analytic and semi-algebraic.

On $r_{i}(\bm{W}_{i}), s_{i}(\bm{V}_{i}), \chi_{\mathcal{MC}}(\Btheta_{MC}):$ All $r_{i}$'s, $s_{i}$'s, $\chi_{\mathcal{MC}}(\Btheta_{MC})$'s that we discussed above are real analytic or semi-algebraic.

On $\|\bm{W}_i - \bm{W}_i^{MC}\|_{F}^{2}:$ Note that the function $\|\bm{W}_i - \bm{W}_i^{MC}\|_{F}^{2}$ is a polynomial function with the variables $\bm{W}_{i}, \bm{W}_{i}^{MC}$.

Since each part of the function $\mathcal{L}$ is either real analytic or semi-algebraic, $\mathcal{L}$ is a subanalytic function~\citep[p.43]{shiota1997geometry}. Furthermore, by the continuity, $\mathcal{L}$ is continuous in its domain. Therefore, $\mathcal{L}$ is a K\L{} function according to \citep[Theorem 3.1]{bolte2007lojasiewicz}.\footnote{Let $h: \mathbb{R}^{p} \rightarrow \mathbb{R} \cup\{+\infty\}$ be a subanalytic function with closed domain, and assume that $h$ is continuous on its domain, then $h$ is a K\L{} function.}
\end{proof}

Based on Lemma~\ref{lemma: sufficient decrease} and under the hypothesis that $\mathcal{L}$ is continuous on its domain and there exists a convergent subsequence, the continuity condition required in \citep{attouch2013convergence} holds naturally, \ie, there exists a subsequence $\{\mathcal{P}^{k_{j}}\}_{j \in \mathbb{N}}$ and $\mathcal{P}^{*}$ such that
\begin{equation}\label{eq: theorem proof 0}
    \mathcal{P}^{k_{j}} \rightarrow \mathcal{P}^{*} \quad \text{ and } \mathcal{L}(\mathcal{P}^{k_{j}}) \rightarrow \mathcal{L}(\mathcal{P}^{*}) \text{, as } j \rightarrow \infty
\end{equation}
Based on Lemmas~\ref{lemma: sufficient decrease},~\ref{lemma: subgradient low bound}, and~\eqref{eq: theorem proof 0}, we can justify the global convergence of $\mathcal{P}^{k}$ stated in Theorem~\ref{thm: global convergence}, following the proof idea in \citep{attouch2013convergence}. For the completeness of the proof, we still present the detailed proof as follows.

Before presenting the main proof, we establish a local convergence result of $\mathcal{P}^{k}$, \ie, the convergence of $\mathcal{P}^{k}$ when $\mathcal{P}^{0}$ is sufficiently close to some point $\mathcal{P}^{*}$. Specifically, let $(\varphi, \eta, U)$ be the associated parameters of the K\L{} property of $\mathcal{L}$ at $\mathcal{P}^{*}$, where $\varphi$ is a continuous concave function, $\eta$ is a positive constant, and $U$ is a neighborhood of $\mathcal{P}^{*}$. Let $\rho$ be some constant such that $\mathcal{N}(\mathcal{P}^{*}, \rho):=\{\mathcal{P}:\|\mathcal{P}-\mathcal{P}^{*}\|_{F} \leq \rho\} \subset U, \mathcal{B}:=\rho+\|\mathcal{P}^{*}\|_{F}$, and $L_{\mathcal{B}}$ be the uniform Lipschitz constant for $\sigma_{i}$, $i=1, \ldots, N-1$, within $\mathcal{N}(\mathcal{P}^{*}, \rho)$. Assume that $\mathcal{P}^{0}$ satisfies the following condition
\begin{equation}\label{eq: theorem proof 1}
    \frac{\bar{\delta}}{\lambda} \varphi(\mathcal{L}(\mathcal{P}^{0})-\mathcal{L}(\mathcal{P}^{*}))+3 \sqrt{\frac{\mathcal{L}(\mathcal{P}^{0})}{\lambda}}+\|\mathcal{P}^{0}-\mathcal{P}^{*}\|_{F}<\rho,
\end{equation}
where $\bar{\delta}=\delta \sqrt{4 N}$, $\lambda$ and $\delta$ are defined in Lemmas~\ref{lemma: sufficient decrease} and~\ref{lemma: subgradient low bound}, respectively.
\begin{lemma}[Local convergence]\label{lemma: local convergence}
     Under the conditions of Theorem 5, suppose that $\mathcal{P}^{0}$ satisfies the condition~\eqref{eq: theorem proof 1}, and $\mathcal{L}(\mathcal{P}^{k})>\mathcal{L}(\mathcal{P}^{*})$ for $k \in \mathbb{N}$, then
\begin{subequations}\label{eq: theorem proof 2}
    \begin{align}
        \sum_{i=1}^{k}\|\mathcal{P}^{i}-\mathcal{P}^{i-1}\|_{F} & \leq 2 \sqrt{\frac{\mathcal{L}(\mathcal{P}^{0})}{\lambda}}+\frac{\bar{\delta}}{\lambda} \varphi(\mathcal{L}(\mathcal{P}^{0})-\mathcal{L}(\mathcal{P}^{*})), \forall k \geq 1 \label{eq: theorem proof 2-1}\\
\mathcal{P}^{k} & \in \mathcal{N}(\mathcal{P}^{*}, \rho), \quad \forall k \in \mathbb{N}. \label{eq: theorem proof 2-2}
    \end{align}
\end{subequations}
As $k$ goes to infinity, \eqref{eq: theorem proof 2-1} yields
\begin{equation*}
    \sum_{i=1}^{\infty}\|\mathcal{P}^{i}-\mathcal{P}^{i-1}\|_{F}<\infty,
\end{equation*}
which implies the convergence of $\{\mathcal{P}^{k}\}_{k \in \mathbb{N}}$.
\end{lemma}

\begin{proof}[Proof of Lemma~\ref{lemma: local convergence}]
    We will prove $\mathcal{P}^{k} \in \mathcal{N}(\mathcal{P}^{*}, \rho)$ by induction on $k$. It is obvious that $\mathcal{P}^{0} \in \mathcal{N}(\mathcal{P}^{*}, \rho)$. Thus, \eqref{eq: theorem proof 2-2} holds for $k=0$. For $k=1$, we have from~\eqref{eq: sufficient decrease} and the nonnegativeness of $\{\mathcal{L}(\mathcal{P}^{k})\}_{k \in \mathbb{N}}$ that
    \begin{equation*}
        \mathcal{L}(\mathcal{P}^{0}) \geq \mathcal{L}(\mathcal{P}^{0})-\mathcal{L}(\mathcal{P}^{1}) \geq \lambda\|\mathcal{P}^{0}-\mathcal{P}^{1}\|_{F}^{2},
    \end{equation*}
which implies $\|\mathcal{P}^{0}-\mathcal{P}^{1}\|_{F} \leq \sqrt{\frac{\mathcal{L}(\mathcal{P}^{0})}{\lambda}}$. Therefore,
\begin{equation*}
    \|\mathcal{P}^{1}-\mathcal{P}^{*}\|_{F} \leq\|\mathcal{P}^{0}-\mathcal{P}^{1}\|_{F}+\|\mathcal{P}^{0}-\mathcal{P}^{*}\|_{F} \leq \sqrt{\frac{\mathcal{L}(\mathcal{P}^{0})}{\lambda}}+\|\mathcal{P}^{0}-\mathcal{P}^{*}\|_{F},
\end{equation*}
which indicates $\mathcal{P}^{1} \in \mathcal{N}(\mathcal{P}^{*}, \rho)$.

Suppose that $\mathcal{P}^{k} \in \mathcal{N}(\mathcal{P}^{*}, \rho)$ for $0 \leq k \leq K$. We proceed to show that $\mathcal{P}^{K+1} \in \mathcal{N}(\mathcal{P}^{*}, \rho)$. Since $\mathcal{P}^{k} \in \mathcal{N}(\mathcal{P}^{*}, \rho)$ for $0 \leq k \leq K$, it implies that $\|\mathcal{P}^{k}\|_{F} \leq \mathcal{B}:=\rho+\mathcal{P}^{*}$ for $0 \leq k \leq K$. Thus, by Lemma~\ref{lemma: subgradient low bound}, for $1 \leq k \leq K$,
\begin{equation*}
    \operatorname{dist}(\mathbf{0}, \partial \mathcal{L}(\mathcal{P}^{k})) \leq \bar{\delta}\|\mathcal{P}^{k}-\mathcal{P}^{k-1}\|_{F},
\end{equation*}
which together with the K\L{} inequality~\eqref{eq: KL inequality} yields
\begin{equation} \label{eq: theorem proof 3}
    \frac{1}{\varphi^{\prime}(\mathcal{L}(\mathcal{P}^{k})-\mathcal{L}(\mathcal{P}^{*}))} \leq \bar{\delta}\|\mathcal{P}^{k}-\mathcal{P}^{k-1}\|_{F}
\end{equation}
By inequality~\eqref{eq: sufficient decrease}, the above inequality and the concavity of $\varphi$, for $k \geq 2$, the following holds
\begin{equation*}
  \begin{aligned}
\lambda\|\mathcal{P}^{k}-\mathcal{P}^{k-1}\|_{F}^{2} & \leq \mathcal{L}(\mathcal{P}^{k-1})-\mathcal{L}(\mathcal{P}^{k})=(\mathcal{L}(\mathcal{P}^{k-1})-\mathcal{L}(\mathcal{P}^{*}))-(\mathcal{L}(\mathcal{P}^{k})-\mathcal{L}(\mathcal{P}^{*})) \\
& \leq \frac{\varphi(\mathcal{L}(\mathcal{P}^{k-1})-\mathcal{L}(\mathcal{P}^{*}))-\varphi(\mathcal{L}(\mathcal{P}^{k})-\mathcal{L}(\mathcal{P}^{*}))}{\varphi^{\prime}(\mathcal{L}(\mathcal{P}^{k-1})-\mathcal{L}(\mathcal{P}^{*}))} \\
& \leq \bar{\delta}\|\mathcal{P}^{k-1}-\mathcal{P}^{k-2}\|_{F} \cdot[\varphi(\mathcal{L}(\mathcal{P}^{k-1})-\mathcal{L}(\mathcal{P}^{*}))-\varphi(\mathcal{L}(\mathcal{P}^{k})-\mathcal{L}(\mathcal{P}^{*}))],
\end{aligned}  
\end{equation*}
which implies
\begin{equation*}
   \|\mathcal{P}^{k}-\mathcal{P}^{k-1}\|_{F}^{2} \leq\|\mathcal{P}^{k-1}-\mathcal{P}^{k-2}\|_{F} \cdot \frac{\bar{\delta}}{\lambda}[\varphi(\mathcal{L}(\mathcal{P}^{k-1})-\mathcal{L}(\mathcal{P}^{*}))-\varphi(\mathcal{L}(\mathcal{P}^{k})-\mathcal{L}(\mathcal{P}^{*}))].
\end{equation*}
Taking the square root on both sides and using the inequality $2 \sqrt{\alpha \beta} \leq \alpha+\beta$, the above inequality implies
\begin{equation*}
    2\|\mathcal{P}^{k}-\mathcal{P}^{k-1}\|_{F} \leq\|\mathcal{P}^{k-1}-\mathcal{P}^{k-2}\|_{F}+\frac{\bar{\delta}}{\lambda}[\varphi(\mathcal{L}(\mathcal{P}^{k-1})-\mathcal{L}(\mathcal{P}^{*}))-\varphi(\mathcal{L}(\mathcal{P}^{k})-\mathcal{L}(\mathcal{P}^{*}))].
\end{equation*}
Summing the above inequality over $k$ from 2 to $K$ and adding $\|\mathcal{P}^{1}-\mathcal{P}^{0}\|_{F}$ to both sides, it yields
\begin{equation*}
    \|\mathcal{P}^{K}-\mathcal{P}^{K-1}\|_{F}+\sum_{k=1}^{K}\|\mathcal{P}^{k}-\mathcal{P}^{k-1}\|_{F} \leq 2\|\mathcal{P}^{1}-\mathcal{P}^{0}\|_{F}+\frac{\bar{\delta}}{\lambda}[\varphi(\mathcal{L}(\mathcal{P}^{0})-\mathcal{L}(\mathcal{P}^{*}))-\varphi(\mathcal{L}(\mathcal{P}^{K})-\mathcal{L}(\mathcal{P}^{*}))]
\end{equation*}
which implies
\begin{equation}\label{eq: theorem proof 4}
    \sum_{k=1}^{K}\|\mathcal{P}^{k}-\mathcal{P}^{k-1}\|_{F} \leq 2 \sqrt{\frac{\mathcal{L}(\mathcal{P}^{0})}{\lambda}}+\frac{\bar{\delta}}{\lambda} \varphi(\mathcal{L}(\mathcal{P}^{0})-\mathcal{L}(\mathcal{P}^{*})),
\end{equation}
and further,
\begin{equation*}
    \begin{aligned}
& \|\mathcal{P}^{K+1}-\mathcal{P}^{*}\|_{F} \leq\|\mathcal{P}^{K+1}-\mathcal{P}^{K}\|_{F}+\sum_{k=1}^{K}\|\mathcal{P}^{k}-\mathcal{P}^{k-1}\|_{F}+\|\mathcal{P}^{0}-\mathcal{P}^{*}\|_{F} \\
& \leq \sqrt{\frac{\mathcal{L}(\mathcal{P}^{K})-\mathcal{L}(\mathcal{P}^{K+1})}{\lambda}}+2 \sqrt{\frac{\mathcal{L}(\mathcal{P}^{0})}{\lambda}}+\frac{\bar{\delta}}{\lambda} \varphi(\mathcal{L}(\mathcal{P}^{0})-\mathcal{L}(\mathcal{P}^{*}))+\|\mathcal{P}^{0}-\mathcal{P}^{*}\|_{F} \\
& \leq 3 \sqrt{\frac{\mathcal{L}(\mathcal{P}^{0})}{\lambda}}+\frac{\bar{\delta}}{\lambda} \varphi(\mathcal{L}(\mathcal{P}^{0})-\mathcal{L}(\mathcal{P}^{*}))+\|\mathcal{P}^{0}-\mathcal{P}^{*}\|_{F}<\rho,
\end{aligned}
\end{equation*}
where the second inequality holds for~\eqref{eq: sufficient decrease} and~\eqref{eq: theorem proof 4}, the third inequality holds for $\mathcal{L}(\mathcal{P}^{K})-\mathcal{L}(\mathcal{P}^{K+1}) \leq \mathcal{L}(\mathcal{P}^{K}) \leq$ $\mathcal{L}(\mathcal{P}^{0})$. Thus, $\mathcal{P}^{K+1} \in \mathcal{N}(\mathcal{P}^{*}, \rho)$. Therefore, we prove this lemma.
\end{proof}
\begin{proof}[Proof of Theorem~\ref{thm: global convergence}]
    We prove the whole sequence convergence stated in Theorem~\ref{thm: global convergence} according to the following two cases.

\textbf{Case 1:} $\mathcal{L}(\mathcal{P}^{k_{0}})=\mathcal{L}(\mathcal{P}^{*})$ at some $k_{0}$. In this case, by Lemma~\ref{lemma: sufficient decrease}, $\mathcal{P}^{k}=\mathcal{P}^{k_{0}}=\mathcal{P}^{*}$ holds for all $k \geq k_{0}$, which implies the convergence of $\mathcal{P}^{k}$ to a limit point $\mathcal{P}^{*}$.

\textbf{Case 2:} $\mathcal{L}(\mathcal{P}^{k})>\mathcal{L}(\mathcal{P}^{*})$ for all $k \in \mathbb{N}$. In this case, since $\mathcal{P}^{*}$ is a limit point and $\mathcal{L}(\mathcal{P}^{k}) \rightarrow \mathcal{L}(\mathcal{P}^{*})$, by Theorem 4 , there must exist an integer $k_{0}$ such that $\mathcal{P}^{k_{0}}$ is sufficiently close to $\mathcal{P}^{*}$ as required in Lemma~\ref{lemma: local convergence} (see the inequality~\eqref{eq: theorem proof 1}). Therefore, the whole sequence $\{\mathcal{P}^{k}\}_{k \in \mathbb{N}}$ converges according to Lemma~\ref{lemma: local convergence}. Since $\mathcal{P}^{*}$ is a limit point of $\{\mathcal{P}^{k}\}_{k \in \mathbb{N}}$, we have $\mathcal{P}^{k} \rightarrow \mathcal{P}^{*}$.

Next, we show $\mathcal{P}^{*}$ is a critical point of $\mathcal{L}$. By  $\lim _{k \rightarrow \infty}\|\mathcal{P}^{k}-\mathcal{P}^{k-1}\|_{F}=0$. Furthermore, by Lemma~\ref{lemma: subgradient low bound}, we have
\begin{equation*}
    \lim _{k \rightarrow \infty} \operatorname{dist}(\mathbf{0}, \partial \mathcal{L}(\mathcal{P}^{k}))=0,
\end{equation*}
which implies that any limit point is a critical point. Therefore, we prove the global convergence of the sequence generated by Algorithm~\ref{alg: NN-BCD}.

The convergence to a global minimum is a straightforward variant of Lemma~\ref{lemma: local convergence}.

The $\mathcal{O}(1 / k)$ rate of convergence is a direct claim according to the proof of Lemma~\ref{lemma: subgradient low bound} and $\lim _{k \rightarrow \infty}\|\mathcal{P}^{k}-\mathcal{P}^{k-1}\|_{F}=0$.
\end{proof}

\begin{remark}
As one important future direction, the stochastic version of the NN-BCD algorithm should be studied by following the works in stochastic programming \cite{xu2024ensemble,zhang2024sampling, zhang2024stochastic}.   
\end{remark}

\section{Additional Experiments} \label{appendix: additional experiments}
In Section~\ref{appendix: case study HAR 5 layers}, a deeper NN structure is considered to demonstrate the performance of our method. Section~\ref{subsec: different gradient-based methods} compares our method with different gradient-based optimizers for uncompressed MLP (CR = 1).
\subsection{Additional Experiments on Tensorized MLP} \label{appendix: case study HAR 5 layers}
We consider the NN structure with five hidden layers. Specifically, the number of neurons in each layer is 561, 1024, 1024, 1024, 512, 512, and 6 (including the input and output layers). The same setup in Section~\ref{subsec: case study HAR} is applied in this experiment.

\begin{table}[!htbp]
  \centering
 \caption{Results of NN-BCD algorithm with different compression ratios (MLP-5 HAR).}
 \label{tab: HAR5Layer CR ratio}
  \begin{tabular}{lrrrrrr}
    \toprule
    & \multicolumn{2}{c}{Training Loss~\eqref{eq: final formulation}}                     & \multicolumn{2}{c}{Training Accuracy}                       & \multicolumn{2}{c}{Test Accuracy}                           \\ \midrule
 CR    & Mean   & Std   & Mean  & Std   & Mean  & Std  \\ 
 \midrule
0.0536	&0.6864	&0.0184	&0.9263	&0.0098	&0.9134	&0.0088\\
0.1112	&0.0312	&0.0008	&0.9595	&0.0042	&0.9384	&0.0070\\
0.1767	&0.0135	&0.0002	&0.9820	&0.0013	&0.9538	&0.0047\\
0.4212	&0.0119	&0.0002	&0.9895	&0.0010	&0.9609	&0.0023\\
0.6742	&0.0071	&0.0001	&0.9939	&0.0009	&0.9642	&0.0034\\
1.0000	&0.0048	&0.0001	&0.9943	&0.0008	&0.9632	&0.0026\\
    \bottomrule
  \end{tabular}
\end{table}
\begin{figure}[!htbp]
\centering
	\subfloat[Training loss]{\includegraphics[width=0.32\textwidth]{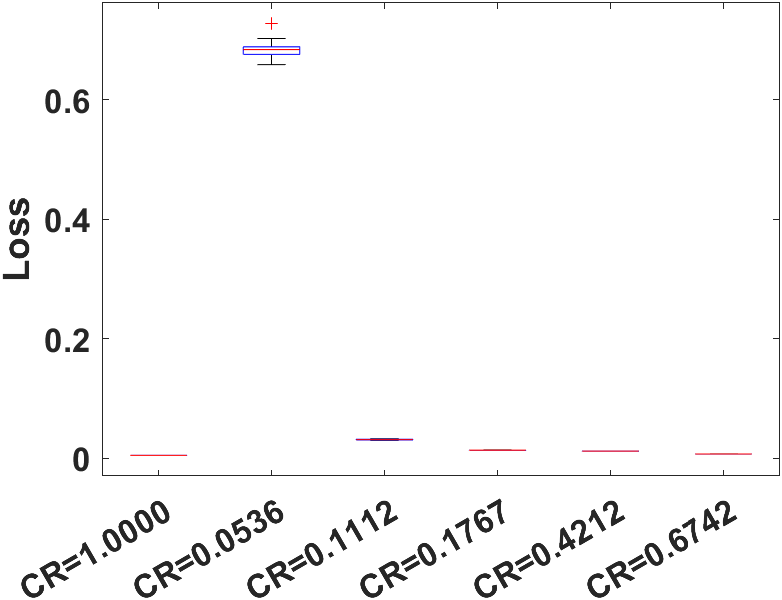} \label{subfig:HAR5Layerbox loss}}
    \subfloat[Training Accuracy]{\includegraphics[width=0.33\textwidth]{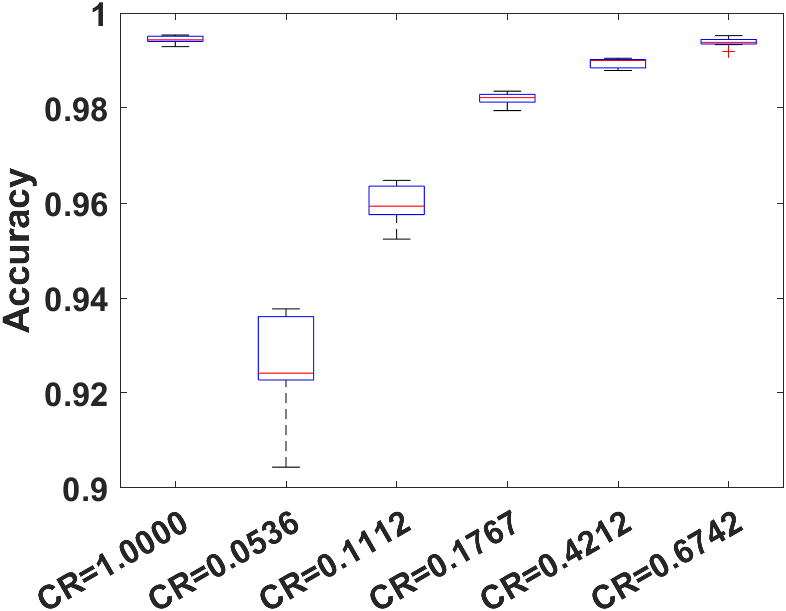} \label{subfig:HAR5Layerbox train acc}}
    \subfloat[Test Accuracy]{\includegraphics[width=0.33\textwidth]{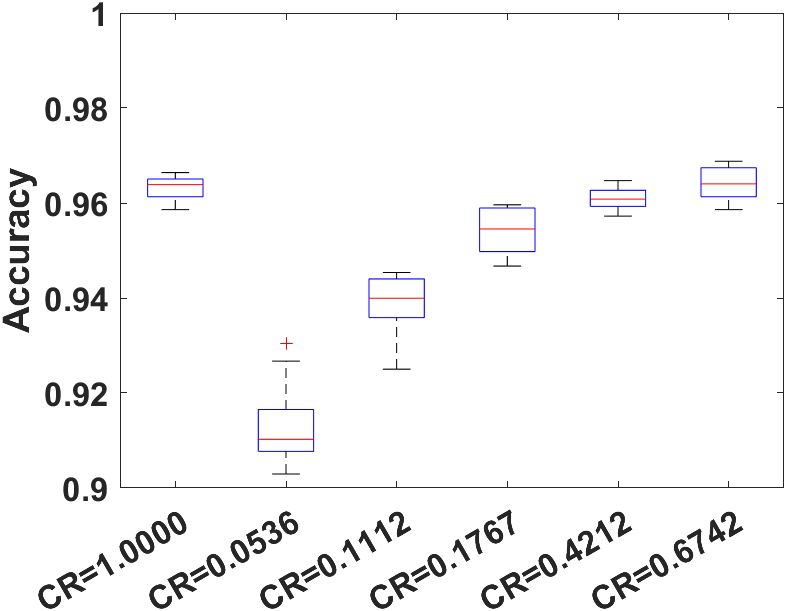} \label{subfig:HAR5Layerbox test acc}}
\caption{The boxplots among ten repetitions with different compression ratios (MLP-5 HAR): (a) training loss; (b) training accuracy; (c) test accuracy.} 
\label{fig: HAR5LayerBox}
\end{figure}
Each experiment was conducted repeatedly with different CRs. The total loss, training accuracy, and test accuracy are presented in Table~\ref{tab: HAR5Layer CR ratio}. Our method demonstrates that the test accuracy reaches 0.9538 with CR = 0.1767, and further increases to 0.9642 when the CR = 0.6742. The result indicates that better training and test accuracy can be obtained when a larger compression ratio (less than 1) is used. This trend is clearly illustrated by the escalating mean of the test accuracy as the CR increases. Figure~\ref{fig: HAR5LayerBox} illustrates boxplots for ten repetitions using different compression ratios. Except for CR = 0.0536, the standard deviation for total loss, training accuracy, and test accuracy across the ten repetitions are generally small.

\begin{figure}[!htbp]
\centering
	\subfloat[Training loss]{\includegraphics[width=0.32\textwidth]{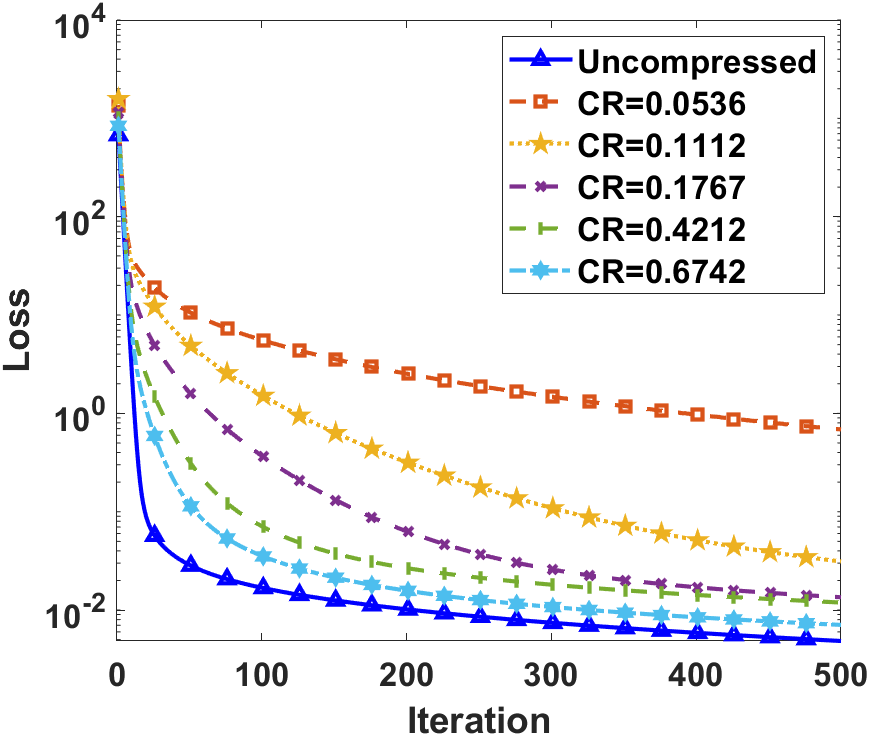} \label{subfig: HAR5Layer training loss}}
     \subfloat[Training Accuracy]{\includegraphics[width=0.32\textwidth]{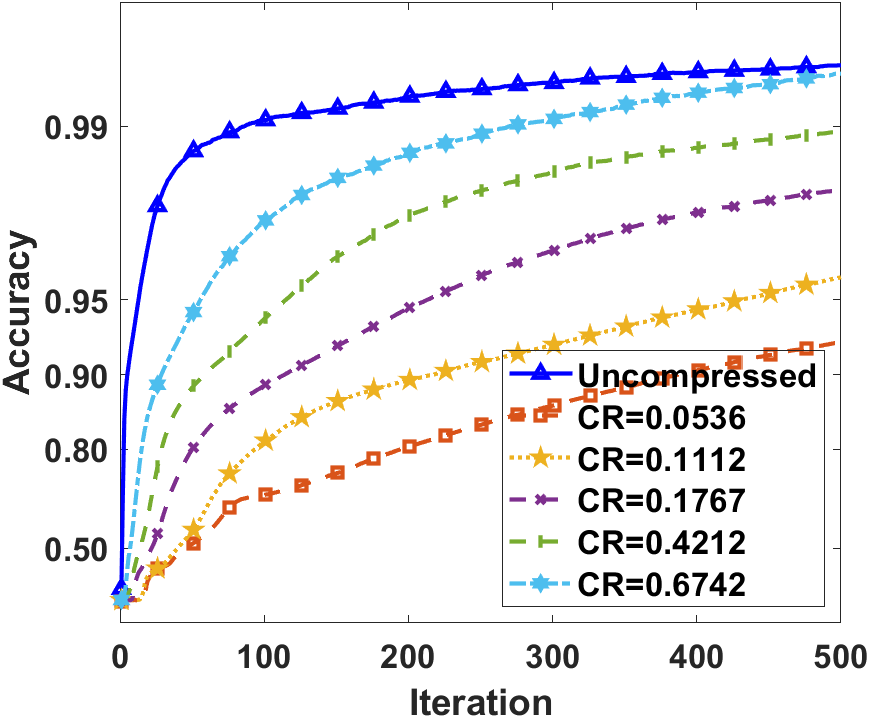} \label{subfig: HAR5Layer training acc}}
    \subfloat[Test Accuracy]{\includegraphics[width=0.32\textwidth]{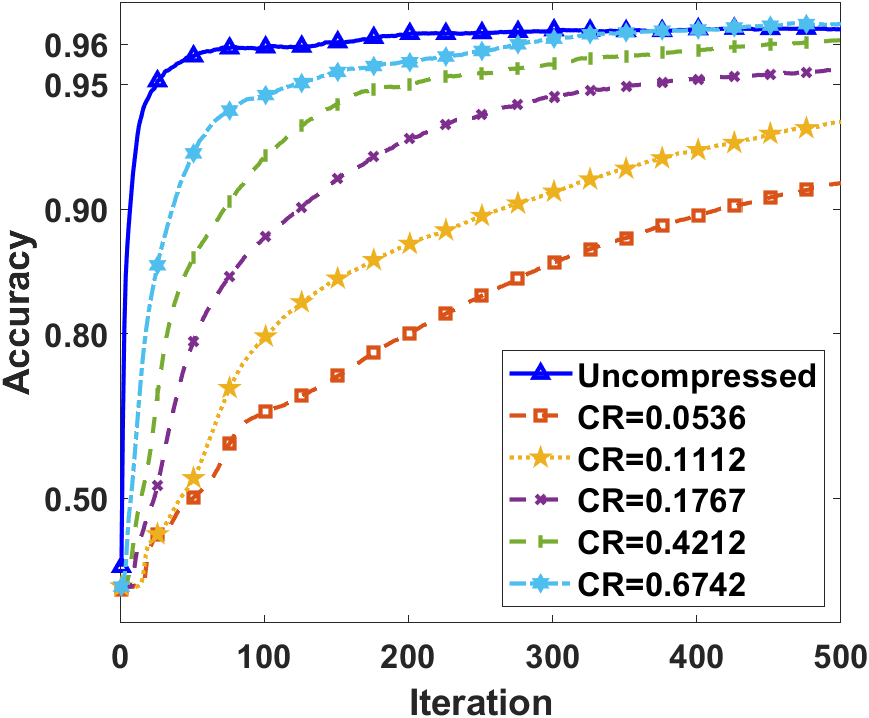} \label{subfig: HAR5Layer test acc}}
\caption{The convergence analysis of NN-BCD algorithm with different compression ratios (MLP-5 HAR): (a) training loss; (b) training accuracy; (c) test accuracy.} 
\label{fig: HAR5Layer convergence}
\end{figure}
The monotone decreasing trend in our loss function~\eqref{eq: final formulation} is clearly demonstrated in Figure~\ref{subfig: HAR5Layer training loss}, which presents the training loss associated with different CRs. The training and test accuracy for various CRs are shown in Figure~\ref{subfig: HAR5Layer training acc} and Figure~\ref{subfig: HAR5Layer test acc}, respectively. For all CRs, the training and test accuracy show the monotone increasing trend, which is an interesting observation even though it is not theoretically guaranteed. 
\begin{figure}[!htbp]
\centering
 \subfloat[Effect of Different Hyperparameters]{\includegraphics[width=0.5\textwidth]{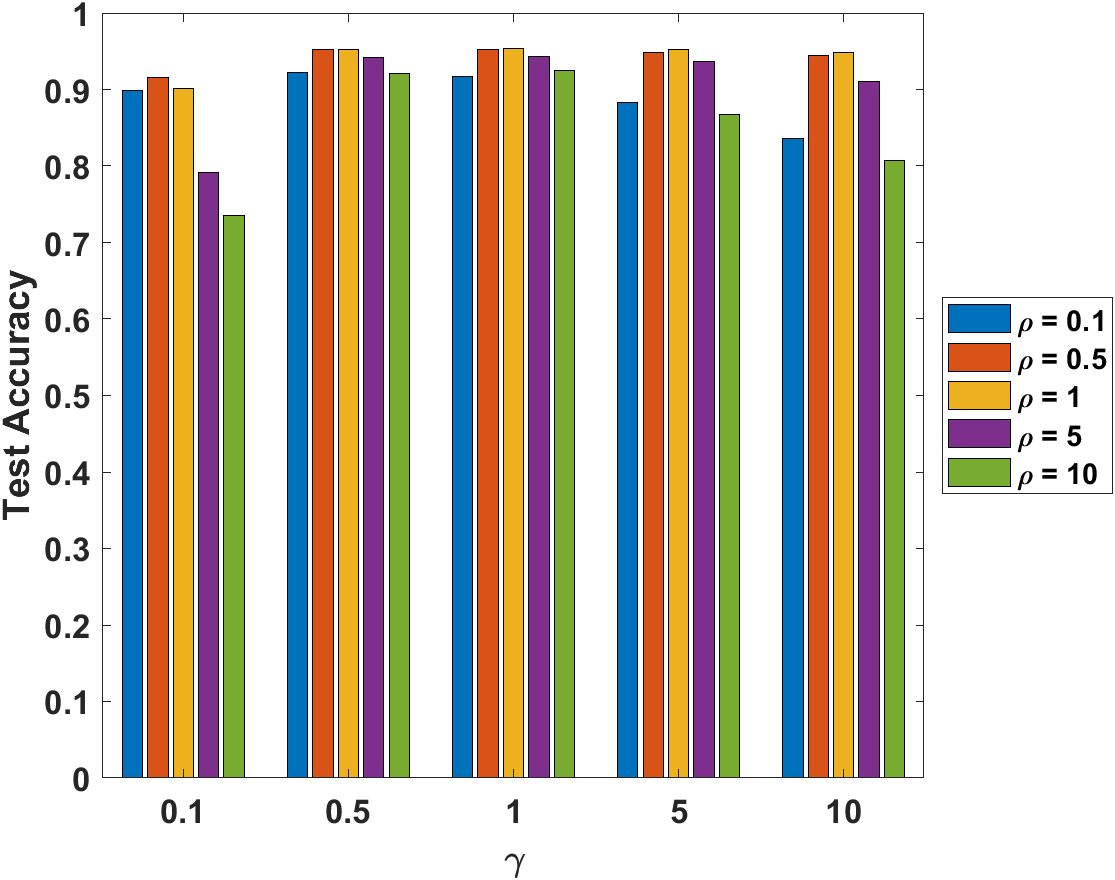} \label{subfig: HAR5LayerDiffGammaRho}}
 \subfloat[Stability of Initialization]{\includegraphics[width=0.5\textwidth]{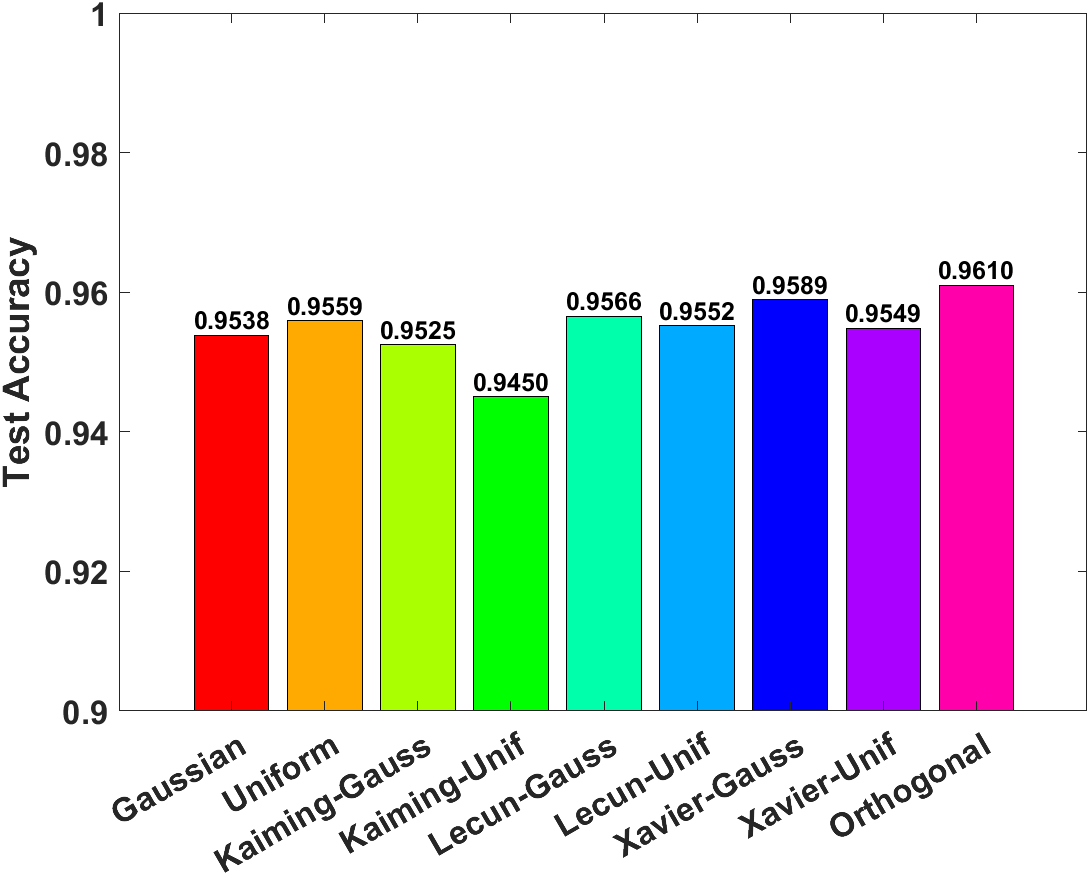} \label{subfig: HAR5LayerDiffInit}}
\caption{Effect of different hyperparameters and weight initialization of NN-BCD algorithm (MLP-5 HAR) when CR = 0.1767: (a) Effect of hyperparameters of NN-BCD; (b) Stability of different weight initialization methods.} 
\label{fig: HAR5Layer stability}
\end{figure} 

Figure~\ref{subfig: HAR5LayerDiffGammaRho} shows a relatively steady performance of our method with different scales of hyperparameters. However, when $\gamma$ and $\rho$ are set to 0.1 and 10, it performs a bit worse than the previous setup. Figure~\ref{subfig: HAR5LayerDiffInit} shows the test accuracy of different weight initialization methods, where the highest test accuracy = 0.9610 by using orthogonal weight initialization and the lowest test accuracy = 0.9450 by using Kaiming-uniform. Overall, our method performs pretty well on different weight initialization methods.

\subsection{Results of Different Gradient-Based Methods} \label{subsec: different gradient-based methods}
In this subsection, different gradient-based optimizers are compared with our method with CR = 1. The same setup in Section~\ref{subsec: case study CNN} is utilized. Specifically, the training accuracy and test accuracy of Adadelta~\citep{Zeiler2012Adadelta}, Adagrad~\citep{duchi2011adaptive}, Adam~\citep{kingma2014adam}, Adamax~\citep{kingma2014adam}, Nadam~\citep{Timothy2016Incorporating}, and RMSprop~\citep{tieleman2012lecture} are shown in Figure~\ref{fig: MnistCNN different optimizers}. The results demonstrate that most of these gradient-based methods have a similar training and test accuracy trend except Adadelta, Adagrad, and SGD. The performance of our method is very close to Adam, Adamax, Nadam, and RMSprop. However, these gradient-based optimizers have a lot of fluctuations for both training and test accuracy while our proposed method is very stable.
\begin{figure}[!htbp]
\centering
    \subfloat[Training Accucacy]{\includegraphics[width=0.5\textwidth]{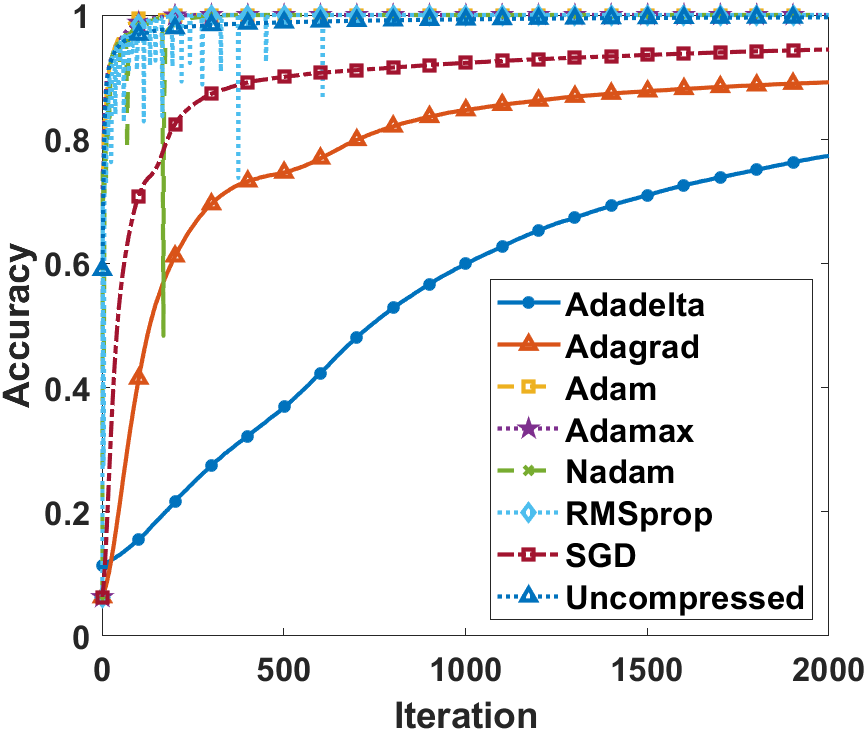} \label{subfig:MnistCNN_diff_optimizers_train_acc}}
    \subfloat[Test Accuracy]{\includegraphics[width=0.5\textwidth]{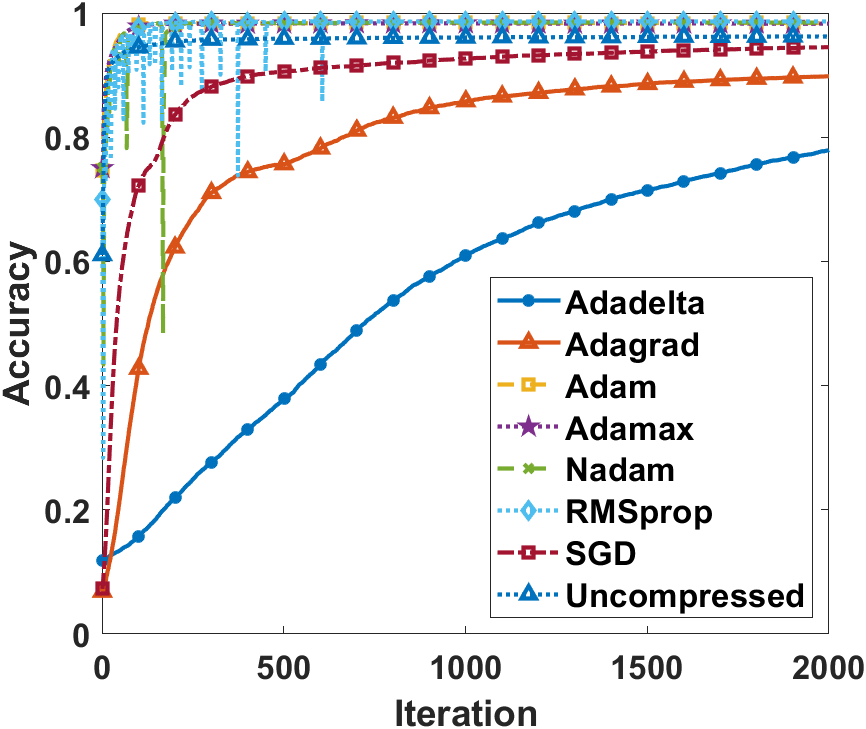} \label{subfig:MnistCNN_diff_optimizers_test_acc}}
\caption{The convergence analysis of different gradient-based methods (CNN MNIST): (a) training accuracy; (b) test accuracy.} 
\label{fig: MnistCNN different optimizers}
\end{figure}

\bibliography{sample}

\end{document}